\documentclass[a4paper,12pt]{amsart}

\usepackage{graphicx,natbib,amssymb,datetime,float,MnSymbol,currfile}
\usepackage{tikz}
\usetikzlibrary{arrows}
\usetikzlibrary{decorations.markings}

\newtheorem{theorem}{Theorem}
\newtheorem{lemma}{Lemma}
\newtheorem{corollary}{Corollary}

\theoremstyle{remark}

\theoremstyle{example}
\newtheorem{example}{Example}

\def\ci{\!\perp\!}
\def\nci{\!\not\perp\!}
\def\ra{\rightarrow}
\def\la{\leftarrow}
\def\aa{\leftrightarrow}
\def\bb{\leftfootline\!\!\!\!\!\rightfootline}

\def\ao{\leftarrow\!\!\!\!\!\multimap}

\def\oa{\mathrel{\reflectbox{\ensuremath{\ao}}}}

\def\oo{\mathrel{\reflectbox{\ensuremath{\multimap}}}\!\!\!\!\!\multimap}

\newcommand{\comments}[1]{}

\tikzset{tt/.style={decoration={
  markings,
  mark=at position .485 with {\arrow{>}},
  mark=at position .515 with {\arrow{<}}},postaction={decorate}}}

\begin{document}

\title[]{Marginal AMP Chain Graphs}

\author[]{Jose M. Pe\~{n}a\\
ADIT, IDA, Link\"oping University, SE-58183 Link\"{o}ping, Sweden\\
jose.m.pena@liu.se}

\date{\currfilename, \currenttime, \ddmmyydate{\today}}

\begin{abstract}
We present a new family of models that is based on graphs that may have undirected, directed and bidirected edges. We name these new models marginal AMP (MAMP) chain graphs because each of them is Markov equivalent to some AMP chain graph under marginalization of some of its nodes. However, MAMP chain graphs do not only subsume AMP chain graphs but also multivariate regression chain graphs. We describe global and pairwise Markov properties for MAMP chain graphs and prove their equivalence for compositional graphoids. We also characterize when two MAMP chain graphs are Markov equivalent.

For Gaussian probability distributions, we also show that every MAMP chain graph is Markov equivalent to some directed and acyclic graph with deterministic nodes under marginalization and conditioning on some of its nodes. This is important because it implies that the independence model represented by a MAMP chain graph can be accounted for by some data generating process that is partially observed and has selection bias. Finally, we modify MAMP chain graphs so that they are closed under marginalization for Gaussian probability distributions. This is a desirable feature because it guarantees parsimonious models under marginalization.
\end{abstract}

\maketitle

\section{Introduction}

Chain graphs (CGs) are graphs with possibly directed and undirected edges, and no semidirected cycle. They have been extensively studied as a formalism to represent independence models, because they can model symmetric and asymmetric relationships between the random variables of interest. However, there are four different interpretations of CGs as independence models \citep{CoxandWermuth1993,CoxandWermuth1996,Drton2009,SonntagandPenna2013}. In this paper, we are interested in the AMP interpretation \citep{Anderssonetal.2001,Levitzetal.2001} and in the multivariate regression (MVR) interpretation \citep{CoxandWermuth1993,CoxandWermuth1996}. Although MVR CGs were originally represented using dashed directed and undirected edges, we prefer to represent them using solid directed and bidirected edges. 

In this paper, we unify and generalize the AMP and MVR interpretations of CGs. We do so by introducing a new family of models that is based on graphs that may have undirected, directed and bidirected edges. We call this new family marginal AMP (MAMP) CGs. 

The rest of the paper is organized as follows. We start with some preliminaries and notation in Section \ref{sec:preliminaries}. We continue by proving in Section \ref{sec:eampcgs} that, for Gaussian probability distributions, every AMP CG is Markov equivalent to some directed and acyclic graph with deterministic nodes under marginalization and conditioning on some of its nodes. We extend this result to MAMP CGs in Section \ref{sec:mampcgs}, which implies that the independence model represented by a MAMP chain graph can be accounted for by some data generating process that is partially observed and has selection bias. Therefore, the independence models represented by MAMP CGs are not arbitrary and, thus, MAMP CGs are worth studying. We also describe in Section \ref{sec:mampcgs} global and pairwise Markov properties for MAMP CGs and prove their equivalence for compositional graphoids. Moreover, we also characterize in that section when two MAMP CGs are Markov equivalent. We show in Section \ref{sec:emampcgs} that MAMP CGs are not closed under marginalization and modify them so that they become closed under marginalization for Gaussian probability distributions. This is important because it guarantees parsimonious models under marginalization. Finally, we discuss in Section \ref{sec:discussion} how MAMP CGs relate to other existing models based on graphs such as regression CGs, maximal ancestral graphs, summary graphs and MC graphs.

\section{Preliminaries}\label{sec:preliminaries}

In this section, we introduce some concepts of models based on graphs, i.e. graphical models. Most of these concepts have a unique definition in the literature. However, a few concepts have more than one definition in the literature and, thus, we opt for the most suitable in this work. All the graphs and probability distributions in this paper are defined over a finite set $V$. All the graphs in this paper are simple, i.e. they contain at most one edge between any pair of nodes. The elements of $V$ are not distinguished from singletons. The operators set union and set difference are given equal precedence in the expressions. The term maximal is always wrt set inclusion.

If a graph $G$ contains an undirected, directed or bidirected edge between two nodes $V_{1}$ and $V_{2}$, then we write that $V_{1} - V_{2}$, $V_{1} \ra V_{2}$ or $V_{1} \aa V_{2}$ is in $G$. We represent with a circle, such as in $\oa$ or $\oo$, that the end of an edge is unspecified, i.e. it may be an arrow tip or nothing. The parents of a set of nodes $X$ of $G$ is the set $pa_G(X) = \{V_1 | V_1 \ra V_2$ is in $G$, $V_1 \notin X$ and $V_2 \in X \}$. The children of $X$ is the set $ch_G(X) = \{V_1 | V_1 \la V_2$ is in $G$, $V_1 \notin X$ and $V_2 \in X \}$. The neighbors of $X$ is the set $ne_G(X) = \{V_1 | V_1 - V_2$ is in $G$, $V_1 \notin X$ and $V_2 \in X \}$. The spouses of $X$ is the set $sp_G(X) = \{V_1 | V_1 \aa V_2$ is in $G$, $V_1 \notin X$ and $V_2 \in X \}$. The adjacents of $X$ is the set $ad_G(X) = ne_G(X) \cup pa_G(X) \cup ch_G(X) \cup sp_G(X)$. A route between a node $V_{1}$ and a node $V_{n}$ in $G$ is a sequence of (not necessarily distinct) nodes $V_{1}, \ldots, V_{n}$ st $V_i \in ad_G(V_{i+1})$ for all $1 \leq i < n$. If the nodes in the route are all distinct, then the route is called a path. The length of a route is the number of (not necessarily distinct) edges in the route, e.g. the length of the route $V_{1}, \ldots, V_{n}$ is $n-1$. A route is called undirected if $V_i - V_{i+1}$ is in $G$ for all $1 \leq i < n$. A route is called descending if $V_i \ra V_{i+1}$, $V_i - V_{i+1}$ or $V_i \aa V_{i+1}$ is in $G$ for all $1 \leq i < n$. A route is called strictly descending if $V_i \ra V_{i+1}$ is in $G$ for all $1 \leq i < n$. The descendants of a set of nodes $X$ of $G$ is the set $de_G(X) = \{V_n |$ there is a descending route from $V_1$ to $V_n$ in $G$, $V_1 \in X$ and $V_n \notin X \}$. The non-descendants of $X$ is the set $nde_G(X)=V \setminus X \setminus de_G(X)$. The strict ascendants of $X$ is the set $san_G(X) = \{V_1 |$ there is a strictly descending route from $V_1$ to $V_n$ in $G$, $V_1 \notin X$ and $V_n \in X \}$. A route $V_{1}, \ldots, V_{n}$ in $G$ is called a cycle if $V_n=V_1$. Moreover, it is called a semidirected cycle if $V_n=V_1$, $V_1 \ra V_2$ is in $G$ and $V_i \ra V_{i+1}$, $V_i \aa V_{i+1}$ or $V_i - V_{i+1}$ is in $G$ for all $1 < i < n$. An AMP chain graph (AMP CG) is a graph whose every edge is directed or undirected st it has no semidirected cycles. A MVR chain graph (MVR CG) is a graph whose every edge is directed or bidirected st it has no semidirected cycles. A set of nodes of a graph is connected if there exists a path in the graph between every pair of nodes in the set st all the edges in the path are undirected or bidirected. A connectivity component of a graph is a maximal connected set. The subgraph of $G$ induced by a set of its nodes $X$, denoted as $G_X$, is the graph over $X$ that has all and only the edges in $G$ whose both ends are in $X$.

Let $X$, $Y$, $Z$ and $W$ denote four disjoint subsets of $V$. An independence model $M$ is a set of statements $X \ci_M Y | Z$. Moreover, $M$ is called graphoid if it satisfies the following properties: Symmetry $X \ci_M Y | Z \Rightarrow Y \ci_M X | Z$, decomposition $X \ci_M Y \cup W | Z \Rightarrow X \ci_M Y | Z$, weak union $X \ci_M Y \cup W | Z \Rightarrow X \ci_M Y | Z \cup W$, contraction $X \ci_M Y | Z \cup W \land X \ci_M W | Z \Rightarrow X \ci_M Y \cup W | Z$, and intersection $X \ci_M Y | Z \cup W \land X \ci_M W | Z \cup Y \Rightarrow X \ci_M Y \cup W | Z$. Moreover, $M$ is called compositional graphoid if it is a graphoid that also satisfies the composition property $X \ci_M Y | Z \land X \ci_M W | Z \Rightarrow X \ci_M Y \cup W | Z$. Another property that $M$ may satisfy is weak transitivity $X \ci_M Y | Z \land X \ci_M Y | Z \cup K \Rightarrow X \ci_M K | Z \lor K \ci_M Y | Z$ with $K \in V \setminus X \setminus Y \setminus Z$.

We now recall the semantics of AMP, MVR and LWF CGs. A node $B$ in a path $\rho$ in an AMP CG $G$ is called a triplex node in $\rho$ if $A \ra B \la C$, $A \ra B - C$, or $A - B \la C$ is a subpath of $\rho$. Moreover, $\rho$ is said to be $Z$-open with $Z \subseteq V$ when

\begin{itemize}
\item every triplex node in $\rho$ is in $Z \cup san_G(Z)$, and

\item every non-triplex node $B$ in $\rho$ is outside $Z$, unless $A - B - C$ is a subpath of $\rho$ and $pa_G(B) \setminus Z \neq \emptyset$.
\end{itemize}

A node $B$ in a path $\rho$ in a MVR CG $G$ is called a triplex node in $\rho$ if $A \oa B \ao C$ is a subpath of $\rho$. Moreover, $\rho$ is said to be $Z$-open with $Z \subseteq V$ when

\begin{itemize}
\item every triplex node in $\rho$ is in $Z \cup san_G(Z)$, and

\item every non-triplex node $B$ in $\rho$ is outside $Z$.
\end{itemize}

A section of a route $\rho$ in a CG is a maximal undirected subroute of $\rho$. A section $V_{2} - \ldots - V_{n-1}$ of $\rho$ is a collider section of $\rho$ if $V_{1} \rightarrow V_{2} - \ldots - V_{n-1} \leftarrow V_{n}$ is a subroute of $\rho$. A route $\rho$ in a CG is said to be $Z$-open when

\begin{itemize}
\item every collider section of $\rho$ has a node in $Z$, and

\item no non-collider section of $\rho$ has a node in $Z$.
\end{itemize}

Let $X$, $Y$ and $Z$ denote three disjoint subsets of $V$. When there is no $Z$-open path/path/route in an AMP/MVR/LWF CG $G$ between a node in $X$ and a node in $Y$, we say that $X$ is separated from $Y$ given $Z$ in $G$ and denote it as $X \ci_G Y | Z$. The independence model represented by $G$ is the set of separations $X \ci_G Y | Z$. We denote it as $I_{AMP}(G)$, $I_{MVR}(G)$ or $I_{LWF}(G)$. In general, these three independence models are different. However, if $G$ is a directed and acyclic graph (DAG), then they are the same. Given an AMP, MVR or LWF CG $G$ and two disjoint subsets $L$ and $S$ of $V$, we denote by $[I(G)]_L^S$ the independence model represented by $G$ under marginalization of the nodes in $L$ and conditioning on the nodes in $S$. Specifically, $X \ci_G Y | Z$ is in $[I(G)]_L^S$ iff $X \ci_G Y | Z \cup S$ is in $I(G)$ and $X, Y, Z \subseteq V \setminus L \setminus S$.

Finally, we denote by $X \ci_p Y | Z$ that $X$ is independent of $Y$ given $Z$ in a probability distribution $p$. We say that $p$ is Markovian wrt an AMP, MVR or LWF CG $G$ when $X \ci_p Y | Z$ if $X \ci_G Y | Z$ for all $X$, $Y$ and $Z$ disjoint subsets of $V$. We say that $p$ is faithful to $G$ when $X \ci_p Y | Z$ iff $X \ci_G Y | Z$ for all $X$, $Y$ and $Z$ disjoint subsets of $V$.

\section{Error AMP CGs}\label{sec:eampcgs}

Any regular Gaussian probability distribution that can be represented by an AMP CG can be expressed as a system of linear equations with correlated errors whose structure depends on the CG \citep[Section 5]{Anderssonetal.2001}. However, the CG represents the errors implicitly, as no nodes in the CG correspond to the errors. We propose in this section to add some deterministic nodes to the CG in order to represent the errors explicitly. We call the result an EAMP CG. We will show that, as desired, every AMP CG is Markov equivalent to its corresponding EAMP CG under marginalization of the error nodes, i.e. the independence model represented by the former coincides with the independence model represented by the latter. We will also show that every EAMP CG under marginalization of the error nodes is Markov equivalent to some LWF CG under marginalization of the error nodes, and that the latter is Markov equivalent to some DAG under marginalization of the error nodes and conditioning on some selection nodes. The relevance of this result can be best explained by extending to AMP CGs what \citet[p. 838]{Koster2002} stated for summary graphs and \citet[p. 981]{RichardsonandSpirtes2002} stated for ancestral graphs: The fact that an AMP CG has a DAG as departure point implies that the independence model associated with the former can be accounted for by some data generating process that is partially observed (corresponding to marginalization) and has selection bias (corresponding to conditioning). We extend this result to MAMP CGs in the next section.

It is worth mentioning that \citet[Theorem 6]{Anderssonetal.2001} have identified the conditions under which an AMP CG is Markov equivalent to some LWF CG.\footnote{To be exact, \citet[Theorem 6]{Anderssonetal.2001} have identified the conditions under which all and only the probability distributions that can be represented by an AMP CG can also be represented by some LWF CG. However, for any AMP or LWF CG $G$, there are Gaussian probability distributions that have all and only the independencies in the independence model represented by $G$, as shown by \citet[Theorem 6.1]{Levitzetal.2001} and \citet[Theorems 1 and 2]{Penna2011}. Then, our formulation is equivalent to the original formulation of the result by \citet[Theorem 6]{Anderssonetal.2001}.} It is clear from these conditions that there are AMP CGs that are not Markov equivalent to any LWF CG. The results in this section differ from those by \citet[Theorem 6]{Anderssonetal.2001}, because we show that every AMP CG is Markov equivalent to some LWF CG with error nodes under marginalization of the error nodes.

It is also worth mentioning that \citet[p. 1025]{RichardsonandSpirtes2002} show that there are AMP CGs that are not Markov equivalent to any DAG under marginalization and conditioning. However, the results in this section show that every AMP CG is Markov equivalent to some DAG with error and selection nodes under marginalization of the error nodes and conditioning of the selection nodes. Therefore, the independence model represented by any AMP CG has indeed some DAG as departure point and, thus, it can be accounted for by some data generating process. The results in this section do not contradict those by \citet[p. 1025]{RichardsonandSpirtes2002}, because they did not consider deterministic nodes while we do (recall that the error nodes are deterministic).

Finally, it is also worth mentioning that EAMP CGs are not the first graphical models to have DAGs as departure point. Specifically, summary graphs \citep{CoxandWermuth1996}, MC graphs \citep{Koster2002}, ancestral graphs \citep{RichardsonandSpirtes2002}, and ribonless graphs \citep{Sadeghi2013} predate EAMP CGs and have the mentioned property. However, none of these other classes of graphical models subsumes AMP CGs, i.e. there are independence models that can be represented by an AMP CG but not by any member of the other class \citep[Section 4]{SadeghiandLauritzen2012}. Therefore, none of these other classes of graphical models subsumes EAMP CGs under marginalization of the error nodes.

\subsection{AMP and LWF CGs with Deterministic Nodes}\label{sec:deterministic}

We say that a node $A$ of an AMP or LWF CG is determined by some $Z \subseteq V$ when $A \in Z$ or $A$ is a function of $Z$. In that case, we also say that $A$ is a deterministic node. We use $D(Z)$ to denote all the nodes that are determined by $Z$. From the point of view of the separations in an AMP or LWF CG, that a node is determined by but is not in the conditioning set of a separation has the same effect as if the node were actually in the conditioning set. We extend the definitions of separation for AMP and LWF CGs to the case where deterministic nodes may exist.

Given an AMP CG $G$, a path $\rho$ in $G$ is said to be $Z$-open when

\begin{itemize}
\item every triplex node in $\rho$ is in $D(Z) \cup san_G(D(Z))$, and

\item no non-triplex node $B$ in $\rho$ is in $D(Z)$, unless $A - B - C$ is a subpath of $\rho$ and $pa_G(B) \setminus D(Z) \neq \emptyset$.
\end{itemize}

Given an LWF CG $G$, a route $\rho$ in $G$ is said to be $Z$-open when

\begin{itemize}
\item every collider section of $\rho$ has a node in $D(Z)$, and

\item no non-collider section of $\rho$ has a node in $D(Z)$.
\end{itemize}

It should be noted that we are not the first to consider models based on graphs with deterministic nodes. For instance, \citet[Section 4]{Geigeretal.1990} consider DAGs with deterministic nodes. However, our definition of deterministic node is more general than theirs.

\subsection{From AMP CGs to DAGs Via EAMP CGs}\label{sec:lwf}

\citet[Section 5]{Anderssonetal.2001} show that any regular Gaussian probability distribution $p$ that is Markovian wrt an AMP CG $G$ can be expressed as a system of linear equations with correlated errors whose structure depends on $G$. Specifically, assume without loss of generality that $p$ has mean 0. Let $K_i$ denote any connectivity component of $G$. Let $\Omega^i_{K_i,K_i}$ and $\Omega^i_{K_i ,pa_G(K_i)}$ denote submatrices of the precision matrix $\Omega^i$ of $p(K_i, pa_G(K_i))$. Then, as shown by \citet[Section 2.3.1]{Bishop2006},
\[
K_i | pa_G(K_i) \sim \mathcal{N}(\beta^i pa_G(K_i), \Lambda^i)
\]
where
\[
\beta^i= -(\Omega^i_{K_i,K_i})^{-1} \Omega^i_{K_i ,pa_G(K_i)}
\]
and
\[
(\Lambda^i)^{-1}= \Omega^i_{K_i,K_i}.
\]

Then, $p$ can be expressed as a system of linear equations with normally distributed errors whose structure depends on $G$ as follows:
\[
K_i = \beta^i \: pa_G(K_i) + \epsilon^i
\]
where 
\[
\epsilon^i \sim \mathcal{N}(0, \Lambda^i).
\]

Note that for all $A, B \in K_i$ st $A- B$ is not in $G$, $A \ci_G B | pa_G(K_i) \cup K_i \setminus A \setminus B$ and thus $(\Lambda^i)^{-1}_{A,B} = 0$ \citep[Proposition 5.2]{Lauritzen1996}. Note also that for all $A \in K_i$ and $B \in pa_G(K_i)$ st $A \la B$ is not in $G$, $A \ci_G B | pa_G(A)$ and thus $(\beta^i)_{A,B}=0$. Let $\beta_A$ contain the nonzero elements of the vector $(\beta^i)_{A, \bullet}$. Then, $p$ can be expressed as a system of linear equations with correlated errors whose structure depends on $G$ as follows. For any $A \in K_i$,
\[
A = \beta_A \: pa_G(A) + \epsilon^A
\]
and for any other $B \in K_i$,
\[
covariance(\epsilon^A, \epsilon^B) = \Lambda^i_{A,B}.
\]

It is worth mentioning that the mapping above between probability distributions and systems of linear equations is bijective \citep[Section 5]{Anderssonetal.2001}. Note that no nodes in $G$ correspond to the errors $\epsilon^A$. Therefore, $G$ represent the errors implicitly. We propose to represent them explicitly. This can easily be done by transforming $G$ into what we call an EAMP CG $G'$ as follows:

\begin{table}[H]
\centering
\scalebox{1.0}{
\begin{tabular}{ll}
1 & Let $G'=G$\\
2 & For each node $A$ in $G$\\
3 & \hspace{0.3cm} Add the node $\epsilon^A$ to $G'$\\
4 & \hspace{0.3cm} Add the edge $\epsilon^A \ra A$ to $G'$\\
5 & For each edge $A - B$ in $G$\\
6 & \hspace{0.3cm} Add the edge $\epsilon^A - \epsilon^B$ to $G'$\\
7 & \hspace{0.3cm} Remove the edge $A - B$ from $G'$\\
\end{tabular}}
\end{table}

The transformation above basically consists in adding the error nodes $\epsilon^A$ to $G$ and connect them appropriately. Figure \ref{fig:example} shows an example. Note that every node $A \in V$ is determined by $pa_{G'}(A)$ and, what will be more important, that $\epsilon^A$ is determined by $pa_{G'}(A) \setminus \epsilon^A \cup A$. Thus, the existence of deterministic nodes imposes independencies which do not correspond to separations in $G$. Note also that, given $Z \subseteq V$, a node $A \in V$ is determined by $Z$ iff $A \in Z$. The if part is trivial. To see the only if part, note that $\epsilon^A \notin Z$ and thus $A$ cannot be determined by $Z$ unless $A \in Z$. Therefore, a node $\epsilon^A$ in $G'$ is determined by $Z$ iff $pa_{G'}(A) \setminus \epsilon^A \cup A \subseteq Z$ because, as shown, there is no other way for $Z$ to determine $pa_{G'}(A) \setminus \epsilon^A \cup A$ which, in turn, determine $\epsilon^A$. Let $\epsilon$ denote all the error nodes in $G'$. Note that we have not yet given a formal definition of EAMP CGs. We define them as all the graphs resulting from applying the pseudocode above to an AMP CG. It is easy to see that every EAMP CG is an AMP CG over $V \cup \epsilon$ and, thus, its semantics are defined. The following theorem confirms that these semantics are as desired. The formal proofs of our results appear in the appendix at the end of the paper.

\begin{figure}
\centering
\scalebox{1.0}{
\begin{tabular}{ccc}\\
\hline
\\
$G$&$G'$&$G''$\\
\\
\begin{tikzpicture}[inner sep=1mm]
\node at (0,0) (A) {$A$};
\node at (1,0) (B) {$B$};
\node at (0,-1) (C) {$C$};
\node at (1,-1) (D) {$D$};
\node at (0,-3) (E) {$E$};
\node at (1,-3) (F) {$F$};
\path[->] (A) edge (B);
\path[->] (A) edge (C);
\path[->] (A) edge (D);
\path[->] (B) edge (D);
\path[-] (C) edge (D);
\path[-] (C) edge (E);
\path[-] (D) edge (F);
\path[-] (E) edge (F);
\end{tikzpicture}
&
\begin{tikzpicture}[inner sep=1mm]
\node at (0,0) (A) {$A$};
\node at (1,0) (B) {$B$};
\node at (0,-1) (C) {$C$};
\node at (1,-1) (D) {$D$};
\node at (0,-3) (E) {$E$};
\node at (1,-3) (F) {$F$};
\node at (-1,0) (EA) {$\epsilon^A$};
\node at (2,0) (EB) {$\epsilon^B$};
\node at (-1,-1) (EC) {$\epsilon^C$};
\node at (2,-1) (ED) {$\epsilon^D$};
\node at (-1,-3) (EE) {$\epsilon^E$};
\node at (2,-3) (EF) {$\epsilon^F$};
\path[->] (EA) edge (A);
\path[->] (EB) edge (B);
\path[->] (EC) edge (C);
\path[->] (ED) edge (D);
\path[->] (EE) edge (E);
\path[->] (EF) edge (F);
\path[->] (A) edge (B);
\path[->] (A) edge (C);
\path[->] (A) edge (D);
\path[->] (B) edge (D);
\path[-] (EC) edge [bend right] (ED);
\path[-] (EC) edge (EE);
\path[-] (ED) edge (EF);
\path[-] (EE) edge [bend left] (EF);
\end{tikzpicture}
&
\begin{tikzpicture}[inner sep=1mm]
\node at (0,0) (A) {$A$};
\node at (1,0) (B) {$B$};
\node at (0,-1) (C) {$C$};
\node at (1,-1) (D) {$D$};
\node at (0,-3) (E) {$E$};
\node at (1,-3) (F) {$F$};
\node at (-1,0) (EA) {$\epsilon^A$};
\node at (2,0) (EB) {$\epsilon^B$};
\node at (-1,-1) (EC) {$\epsilon^C$};
\node at (2,-1) (ED) {$\epsilon^D$};
\node at (-1,-3) (EE) {$\epsilon^E$};
\node at (2,-3) (EF) {$\epsilon^F$};
\node at (0.5,-1.5) (SCD) {$S_{\epsilon^C\epsilon^D}$};
\node at (-1,-2) (SCE) {$S_{\epsilon^C\epsilon^E}$};
\node at (2,-2) (SDF) {$S_{\epsilon^D\epsilon^F}$};
\node at (0.5,-2.5) (SEF) {$S_{\epsilon^E\epsilon^F}$};
\path[->] (EA) edge (A);
\path[->] (EB) edge (B);
\path[->] (EC) edge (C);
\path[->] (ED) edge (D);
\path[->] (EE) edge (E);
\path[->] (EF) edge (F);
\path[->] (A) edge (B);
\path[->] (A) edge (C);
\path[->] (A) edge (D);
\path[->] (B) edge (D);
\path[->] (EC) edge (SCE);
\path[->] (EE) edge (SCE);
\path[->] (ED) edge (SDF);
\path[->] (EF) edge (SDF);
\path[->] (EC) edge (SCD);
\path[->] (ED) edge (SCD);
\path[->] (EE) edge (SEF);
\path[->] (EF) edge (SEF);
\end{tikzpicture}\\
\hline
\end{tabular}}\caption{Example of the different transformations for AMP CGs.}\label{fig:example}
\end{figure}
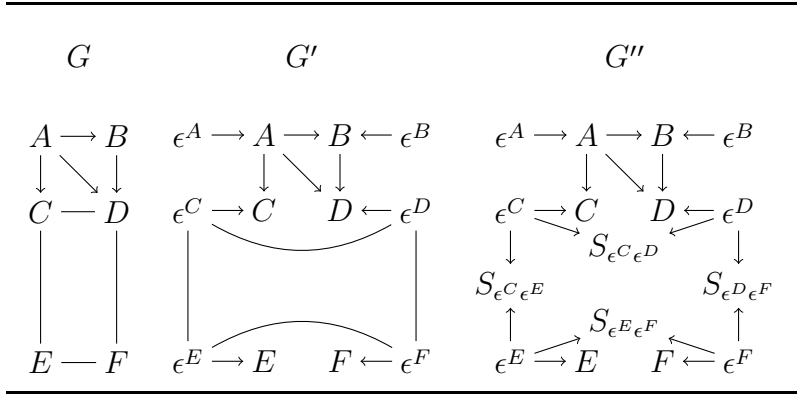

\begin{theorem}\label{the:GG'}
$I_{AMP}(G)=[I_{AMP}(G')]_\epsilon^\emptyset$.
\end{theorem}

\begin{theorem}\label{the:G'G'}
Assume that $G'$ has the same deterministic relationships no matter whether it is interpreted as an AMP or LWF CG. Then, $I_{AMP}(G')=I_{LWF}(G')$.
\end{theorem}

The following corollary links the two most popular interpretations of CGs. Specifically, it shows that every AMP CG is Markov equivalent to some LWF CG with deterministic nodes under marginalization. The corollary follows from Theorems \ref{the:GG'} and \ref{the:G'G'}.

\begin{corollary}\label{cor:GG'}
$I_{AMP}(G)=[I_{LWF}(G')]_\epsilon^\emptyset$.
\end{corollary}

Now, let $G''$ denote the DAG obtained from $G'$ by replacing every edge $\epsilon^A - \epsilon^B$ in $G'$ with $\epsilon^A \ra S_{\epsilon^A\epsilon^B} \la \epsilon^B$. Figure \ref{fig:example} shows an example. The nodes $S_{\epsilon^A\epsilon^B}$ are called selection nodes. Let $S$ denote all the selection nodes in $G''$. The following theorem relates the semantics of $G'$ and $G''$.

\begin{theorem}\label{the:G'G''}
Assume that $G'$ and $G''$ have the same deterministic relationships. Then, $I_{LWF}(G')=[I(G'')]_\emptyset^S$.
\end{theorem}

The main result of this section is the following corollary, which shows that every AMP CG is Markov equivalent to some DAG with deterministic nodes under marginalization and conditioning. The corollary follows from Corollary \ref{cor:GG'} and Theorem \ref{the:G'G''}.

\begin{corollary}\label{cor:GG''}
$I_{AMP}(G)=[I(G'')]_\epsilon^S$.
\end{corollary}

\section{Marginal AMP CGs}\label{sec:mampcgs}

In this section, we present the main contribution of this paper, namely a new family of graphical models that unify and generalize AMP and MVR CGs. Specifically, a graph $G$ containing possibly directed, bidirected and undirected edges is a marginal AMP (MAMP) CG if

\begin{itemize}
\item[{\bf C1}.] $G$ has no semidirected cycle,

\item[{\bf C2}.] $G$ has no cycle $V_1, \ldots, V_n=V_1$ st $V_1 \aa V_2$ is in $G$ and $V_i - V_{i+1}$ is in $G$ for all $1 < i < n$, and

\item[{\bf C3}.] if $V_1 - V_2 - V_3$ is in $G$ and $sp_G(V_2) \neq \emptyset$, then $V_1 - V_3$ is in $G$ too.
\end{itemize}

A set of nodes of a MAMP CG $G$ is undirectly connected if there exists a path in $G$ between every pair of nodes in the set st all the edges in the path are undirected. An undirected connectivity component of $G$ is a maximal undirectly connected set. We denote by $uc_G(A)$ the undirected connectivity component a node $A$ of $G$ belongs to.

The semantics of MAMP CGs is as follows. A node $B$ in a path $\rho$ in a MAMP CG $G$ is called a triplex node in $\rho$ if $A \oa B \ao C$, $A \oa B - C$, or $A - B \ao C$ is a subpath of $\rho$. Moreover, $\rho$ is said to be $Z$-open with $Z \subseteq V$ when

\begin{itemize}
\item every triplex node in $\rho$ is in $Z \cup san_G(Z)$, and

\item every non-triplex node $B$ in $\rho$ is outside $Z$, unless $A - B - C$ is a subpath of $\rho$ and $sp_G(B) \neq \emptyset$ or $pa_G(B) \setminus Z \neq \emptyset$.
\end{itemize}

Let $X$, $Y$ and $Z$ denote three disjoint subsets of $V$. When there is no $Z$-open path in $G$ between a node in $X$ and a node in $Y$, we say that $X$ is separated from $Y$ given $Z$ in $G$ and denote it as $X \ci_G Y | Z$. We denote by $X \nci_G Y | Z$ that $X \ci_G Y | Z$ does not hold. Likewise, we denote by $X \ci_p Y | Z$ (respectively $X \nci_p$ $Y | Z$) that $X$ is independent (respectively dependent) of $Y$ given $Z$ in a probability distribution $p$. The independence model represented by $G$, denoted as $I(G)$, is the set of separation statements $X \ci_G$ $Y | Z$. We say that $p$ is Markovian wrt $G$ when $X \ci_p Y | Z$ if $X \ci_G Y | Z$ for all $X$, $Y$ and $Z$ disjoint subsets of $V$. Moreover, we say that $p$ is faithful to $G$ when $X \ci_p Y | Z$ iff $X \ci_G Y | Z$ for all $X$, $Y$ and $Z$ disjoint subsets of $V$.

Note that if a MAMP CG $G$ has a path $V_1 - V_2 - \ldots - V_n$ st $sp_G(V_i) \neq \emptyset$ for all $1 < i < n$, then $V_1 - V_n$ must be in $G$. Therefore, the independence model represented by a MAMP CG is the same whether we use the definition of $Z$-open path above or the following simpler one. A path $\rho$ in a MAMP CG $G$ is said to be $Z$-open when

\begin{itemize}
\item every triplex node in $\rho$ is in $Z \cup san_G(Z)$, and

\item every non-triplex node $B$ in $\rho$ is outside $Z$, unless $A - B - C$ is a subpath of $\rho$ and $pa_G(B) \setminus Z \neq \emptyset$.
\end{itemize}

The motivation behind the three constraints in the definition of MAMP CGs is as follows. The constraint C1 follows from the semidirected acyclicity constraint of AMP and MVR CGs. For the constraints C2 and C3, note that typically every missing edge in the graph of a graphical model corresponds to a separation. However, this may not be true for graphs that do not satisfy the constraints C2 and C3. For instance, the graph $G$ below does not contain any edge between $B$ and $D$ but $B \nci_G D | Z$ for all $Z \subseteq V \setminus \{B, D\}$. Likewise, $G$ does not contain any edge between $A$ and $E$ but $A \nci_G E | Z$ for all $Z \subseteq V \setminus \{A, E\}$.

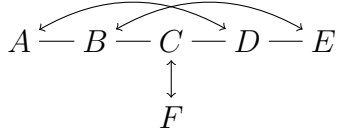
\begin{table}[H]
\centering
\scalebox{1.0}{
\begin{tabular}{c}
\begin{tikzpicture}[inner sep=1mm]
\node at (0,0) (A) {$A$};
\node at (1,0) (B) {$B$};
\node at (2,0) (C) {$C$};
\node at (3,0) (D) {$D$};
\node at (4,0) (E) {$E$};
\node at (2,-1) (F) {$F$};
\path[-] (A) edge (B);
\path[-] (B) edge (C);
\path[-] (C) edge (D);
\path[-] (D) edge (E);
\path[<->] (A) edge [bend left] (D);
\path[<->] (B) edge [bend left] (E);
\path[<->] (C) edge (F);
\end{tikzpicture}
\end{tabular}}
\end{table}

Since the situation above is counterintuitive, we enforce the constraints C2 and C3. Theorem \ref{the:pairwise1} below shows that every missing edge in a MAMP CG corresponds to a separation.

Note that AMP and MVR CGs are special cases of MAMP CGs. However, MAMP CGs are a proper generalization of AMP and MVR CGs, as there are independence models that can be represented by the former but not by the two latter. An example follows (we postpone the proof that it cannot be represented by any AMP or MVR CG until after Theorem \ref{the:triplex}).

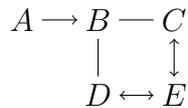
\begin{table}[H]
\centering
\scalebox{1.0}{
\begin{tabular}{c}
\begin{tikzpicture}[inner sep=1mm]
\node at (-1,0) (A) {$A$};
\node at (0,0) (B) {$B$};
\node at (1,0) (C) {$C$};
\node at (0,-1) (D) {$D$};
\node at (1,-1) (E) {$E$};
\path[->] (A) edge (B);
\path[-] (B) edge (C);
\path[-] (B) edge (D);
\path[<->] (C) edge (E);
\path[<->] (D) edge (E);
\end{tikzpicture} 
\end{tabular}}
\end{table}

Given a MAMP CG $G$, let $\widehat{G}$ denote the AMP CG obtained by replacing every bidirected edge $A \aa B$ in $G$ with $A \la L_{AB} \ra B$. Note that $G$ and $\widehat{G}$ represent the same separations over $V$. Therefore, every MAMP CG can be seen as the result of marginalizing out some nodes in an AMP CG, hence the name. Furthermore, Corollary \ref{cor:GG''} shows that every AMP CG can be seen as the result of marginalizing out and conditioning on some nodes in a DAG. Consequently, every MAMP CG can also be seen as the result of marginalizing out and conditioning on some nodes in a DAG. Therefore, the independence model represented by a MAMP CG can be accounted for by some data generating process that is partially observed and has selection bias. This implies that the independence models represented by MAMP CGs are not arbitrary and, thus, MAMP CGs are worth studying. The theorem below provides another way to see that the independence models represented by MAMP CGs are not arbitrary. Specifically, it shows that each of them coincides with the independence model of some probability distribution.

\begin{theorem}\label{the:faithfulness}
For any MAMP CG $G$, there exists a regular Gaussian probability distribution $p$ that is faithful to $G$.
\end{theorem}

\begin{corollary}\label{cor:wtc}
Any independence model represented by a MAMP CG is a compositional graphoid that satisfies weak transitivity.
\end{corollary}

Finally, we show below that the independence model represented by a MAMP CG coincides with certain closure of certain separations. This is interesting because it implies that a few separations and rules to combine them characterize all the separations represented by a MAMP CG. Moreover, it also implies that we have a simple graphical criterion to decide whether a given separation is or is not in the closure without having to find a derivation of it, which is usually a tedious task. Specifically, we define the pairwise separation base of a MAMP CG $G$ as the separations

\begin{itemize}

\item $A \ci B | pa_G(A)$ for all $A, B \in V$ st $A \notin ad_G(B)$ and $B \notin de_G(A)$,

\item $A \ci B | ne_G(A) \cup pa_G(A \cup ne_G(A))$ for all $A, B \in V$ st $A \notin ad_G(B)$, $A \in de_G(B)$, $B \in de_G(A)$ and $uc_G(A)=uc_G(B)$, and

\item $A \ci B | pa_G(A)$ for all $A, B \in V$ st $A \notin ad_G(B)$, $A \in de_G(B)$, $B \in de_G(A)$ and $uc_G(A) \neq uc_G(B)$.

\end{itemize}

We define the compositional graphoid closure of the pairwise separation base of $G$, denoted as $cl(G)$, as the set of separations that are in the base plus those that can be derived from it by applying the compositional graphoid properties. We denote the separations in $cl(G)$ as $X \ci_{cl(G)} Y | Z$.

\begin{theorem}\label{the:pairwise1}
For any MAMP CG $G$, if $X \ci_{cl(G)} Y | Z$ then $X \ci_G Y | Z$.
\end{theorem}

\begin{theorem}\label{the:pairwise2}
For any MAMP CG $G$, if $X \ci_G Y | Z$ then $X \ci_{cl(G)} Y | Z$.
\end{theorem}

\subsection{Markov Equivalence}\label{sec:equivalence}

We say that two MAMP CGs are Markov equivalent if they represent the same independence model. In a MAMP CG, a triplex $(\{A,C\},B)$ is an induced subgraph of the form $A \oa B \ao$$ C$, $A \oa B - C$, or $A - B \ao C$. We say that two MAMP CGs are triplex equivalent if they have the same adjacencies and the same triplexes.

\begin{theorem}\label{the:triplex}
Two MAMP CGs are Markov equivalent iff they are triplex equivalent.
\end{theorem}

We mentioned in the previous section that MAMP CGs are a proper generalization of AMP and MVR CGs, as there are independence models that can be represented by the former but not by the two latter. Moreover, we gave the an example and postponed the proof. With the help of Theorem \ref{the:triplex}, we can now give the proof.

\begin{example}
The independence model represented by the MAMP CG $G$ below cannot be represented by any AMP or MVR CG.

\begin{table}[H]
\centering
\scalebox{1.0}{
\begin{tabular}{c}
\begin{tikzpicture}[inner sep=1mm]
\node at (-1,0) (A) {$A$};
\node at (0,0) (B) {$B$};
\node at (1,0) (C) {$C$};
\node at (0,-1) (D) {$D$};
\node at (1,-1) (E) {$E$};
\path[->] (A) edge (B);
\path[-] (B) edge (C);
\path[-] (B) edge (D);
\path[<->] (C) edge (E);
\path[<->] (D) edge (E);
\end{tikzpicture} 
\end{tabular}}
\end{table}

To see it, assume to the contrary that it can be represented by an AMP CG $H$. Note that $H$ is a MAMP CG too. Then, $G$ and $H$ must have the same triplexes by Theorem \ref{the:triplex}. Then, $H$ must have triplexes $(\{A,D\},B)$ and $(\{A,C\},B)$ but no triplex $(\{C,D\},B)$. So, $C - B - D$ must be in $H$. Moreover, $H$ must have a triplex $(\{B,E\},C)$. So, $C \la E$ must be in $H$. However, this implies that $H$ does not have a triplex $(\{C,D\},E)$, which is a contradiction because $G$ has such a triplex. To see that no MVR CG can represent the independence model represented by $G$, simply note that no MVR CG can have triplexes $(\{A,D\},B)$ and $(\{A,C\},B)$ but no triplex $(\{C,D\},B)$.
\end{example}

We end this section with two lemmas that identify some interesting distinguished members of a triplex equivalence class of MAMP CGs. We say that two nodes form a directed node pair if there is a directed edge between them.

\begin{lemma}\label{lem:directed}
For every triplex equivalence class of MAMP CGs, there is a unique maximal set of directed node pairs st some CG in the class has exactly those directed node pairs.
\end{lemma}

A MAMP CG is a maximally directed CG (MDCG) if it has exactly the maximal set of directed node pairs corresponding to its triplex equivalence class. Note that there may be several MDCGs in the class. For instance, the triplex equivalence class that contains the MAMP CG $A \ra B$ has two MDCGs (i.e. $A \ra B$ and $A \la B$).

\begin{lemma}\label{lem:bidirected}
For every triplex equivalence class of MDCGs, there is a unique maximal set of bidirected edges st some MDCG in the class has exactly those bidirected edges.
\end{lemma}

A MDCG is a maximally bidirected MDCG (MBMDCG) if it has exactly the maximal set of bidirected edges corresponding to its triplex equivalence class. Note that there may be several MBMDCGs in the class. For instance, the triplex equivalence class that contains the MAMP CG $A \ra B$ has two MBMDCGs (i.e. $A \ra B$ and $A \la B$). Note however that all the MBMDCGs in a triplex equivalence class have the same triplex edges, i.e. the edges in a triplex.

\section{Error MAMP CGs}\label{sec:emampcgs}

Unfortunately, MAMP CGs are not closed under marginalization, meaning that the independence model resulting from marginalizing out some nodes in a MAMP CG may not be representable by any MAMP CG. An example follows.

\begin{example}
The independence model resulting from marginalizing out $E$ and $I$ in the MAMP CG $G$ below cannot be represented by any MAMP CG.

\begin{table}[H]
\centering
\scalebox{1.0}{
\begin{tabular}{c}
\begin{tikzpicture}[inner sep=1mm]
\node at (0,0) (A) {$A$};
\node at (4,0) (B) {$B$};
\node at (-1,-1) (C) {$C$};
\node at (0,-1) (D) {$D$};
\node at (1,-1) (E) {$E$};
\node at (2,-1) (F) {$F$};
\node at (3,-1) (I) {$I$};
\node at (4,-1) (J) {$J$};
\node at (5,-1) (K) {$K$};
\path[->] (A) edge (D);
\path[->] (B) edge (J);
\path[->] (E) edge (F);
\path[->] (I) edge (F);
\path[-] (C) edge (D);
\path[-] (E) edge (D);
\path[-] (J) edge (K);
\path[-] (I) edge (J);
\end{tikzpicture}
\end{tabular}}
\end{table}

To see it, assume to the contrary that it can be represented by a MAMP CG $H$. Note that $C$ and $D$ must be adjacent in $H$, because $C \nci_G D | Z$ for all $Z \subseteq \{A, B, F, J, K\}$. Similarly, $D$ and $F$ must be adjacent in $H$. However, $H$ cannot have a triplex $(\{C,F\},D)$ because $C \ci_G F | A \cup D$. Moreover, $C \la D$ cannot be in $H$ because $A \ci_G C$, and $D \ra F$ cannot be in $H$ because $A \ci_G F$. Then, $C - D- F$ must be in $H$. Following an analogous reasoning, we can conclude that $F - J - K$ must be in $H$. However, this contradicts that $D \ci_G J$.
\end{example}

A solution to the problem above is to represent the marginal model by a MAMP CG with extra edges so as to avoid representing false independencies. This, of course, has two undesirable consequences: Some true independencies may not be represented, and the complexity of the CG increases. See \cite[p. 965]{RichardsonandSpirtes2002} for a discussion on the importance of the class of models considered being closed under marginalization. In this section, we propose an alternative solution to this problem: Much like we did in Section \ref{sec:eampcgs} with AMP CGs, we modify MAMP CGs into what we call EMAMP CGs, and show that the latter are closed under marginalization.\footnote{The reader may think that parts of this section are repetition of Section \ref{sec:eampcgs} and, thus, that both sections should be unified. However, we think that this would harm readability.}

\subsection{MAMP CGs with Deterministic Nodes}

We say that a node $A$ of a MAMP CG is determined by some $Z \subseteq V$ when $A \in Z$ or $A$ is a function of $Z$. In that case, we also say that $A$ is a deterministic node. We use $D(Z)$ to denote all the nodes that are determined by $Z$. From the point of view of the separations in a MAMP CG, that a node is determined by but is not in the conditioning set of a separation has the same effect as if the node were actually in the conditioning set. We extend the definition of separation for MAMP CGs to the case where deterministic nodes may exist.

Given a MAMP CG $G$, a path $\rho$ in $G$ is said to be $Z$-open when

\begin{itemize}
\item every triplex node in $\rho$ is in $D(Z) \cup san_G(D(Z))$, and

\item no non-triplex node $B$ in $\rho$ is in $D(Z)$, unless $A - B - C$ is a subpath of $\rho$ and $pa_G(B) \setminus D(Z) \neq \emptyset$.
\end{itemize}

\subsection{From MAMP CGs to EMAMP CGs}

\citet[Section 5]{Anderssonetal.2001} and \citet[Section 2]{KangandTian2009} show that any regular Gaussian probability distribution that is Markovian wrt an AMP or MVR CG $G$ can be expressed as a system of linear equations with correlated errors whose structure depends on $G$. As we show below, these two works can easily be combined to obtain a similar result for MAMP CGs.

Let $p$ denote any regular Gaussian distributions that is Markovian wrt a MAMP CG $G$. Assume without loss of generality that $p$ has mean 0. Let $K_i$ denote any connectivity component of $G$. Let $\Omega^i_{K_i,K_i}$ and $\Omega^i_{K_i ,pa_G(K_i)}$ denote submatrices of the precision matrix $\Omega^i$ of $p(K_i, pa_G(K_i))$. Then, as shown by \citet[Section 2.3.1]{Bishop2006},
\[
K_i | pa_G(K_i) \sim \mathcal{N}(\beta^i pa_G(K_i), \Lambda^i)
\]
where
\[
\beta^i= -(\Omega^i_{K_i,K_i})^{-1} \Omega^i_{K_i ,pa_G(K_i)}
\]
and
\[
(\Lambda^i)^{-1}= \Omega^i_{K_i,K_i}.
\]

Then, $p$ can be expressed as a system of linear equations with normally distributed errors whose structure depends on $G$ as follows:
\[
K_i = \beta^i \: pa_G(K_i) + \epsilon^i
\]
where 
\[
\epsilon^i \sim \mathcal{N}(0, \Lambda^i).
\]

Note that for all $A, B \in K_i$ st $uc_G(A)=uc_G(B)$ and $A- B$ is not in $G$, $A \ci_G B | pa_G(K_i) \cup uc_G(A) \setminus A \setminus B$ and thus $(\Lambda^i_{uc_G(A),uc_G(A)})^{-1}_{A,B} = 0$ \citep[Proposition 5.2]{Lauritzen1996}. Note also that for all $A, B \in K_i$ st $uc_G(A) \neq uc_G(B)$ and $A \aa B$ is not in $G$, $A \ci_G B | pa_G(K_i)$ and thus $\Lambda^i_{A,B} = 0$. Finally, note also that for all $A \in K_i$ and $B \in pa_G(K_i)$ st $A \la B$ is not in $G$, $A \ci_G B | pa_G(A)$ and thus $(\beta^i)_{A,B}=0$. Let $\beta_A$ contain the nonzero elements of the vector $(\beta^i)_{A, \bullet}$. Then, $p$ can be expressed as a system of linear equations with correlated errors whose structure depends on $G$ as follows. For any $A \in K_i$,
\[
A = \beta_A \: pa_G(A) + \epsilon^A
\]
and for any other $B \in K_i$,
\[
covariance(\epsilon^A, \epsilon^B) = \Lambda^i_{A,B}.
\]

It is worth mentioning that the mapping above between probability distributions and systems of linear equations is bijective. We omit the proof of this fact because it is unimportant in this work, but it can be proven much in the same way as Lemma 1 in \citet{Penna2011}. Note that each equation in the system of linear equations above is a univariate recursive regression, i.e. a random variable can be a regressor in an equation only if it has been the regressand in a previous equation. This has two main advantages, as \citet[p. 207]{CoxandWermuth1993} explain: "First, and most importantly, it describes a stepwise process by which the observations could have been generated and in this sense may prove the basis for developing potential causal explanations. Second, each parameter in the system [of linear equations] has a well-understood meaning since it is a regression coefficient: That is, it gives for unstandardized variables the amount by which the response is expected to change if the explanatory variable is increased by one unit and all other variables in the equation are kept constant." Therefore, a MAMP CG can be seen as a data generating process and, thus, it gives us insight into the system under study.

Note that no nodes in $G$ correspond to the errors $\epsilon^A$. Therefore, $G$ represent the errors implicitly. We propose to represent them explicitly. This can easily be done by transforming $G$ into what we call an EMAMP CG $G'$ as follows, where $A \bb B$ means $A \aa B$ or $A - B$:

\begin{table}[H]
\centering
\scalebox{1.0}{
\begin{tabular}{ll}
1 & Let $G'=G$\\
2 & For each node $A$ in $G$\\
3 & \hspace{0.3cm} Add the node $\epsilon^A$ to $G'$\\
4 & \hspace{0.3cm} Add the edge $\epsilon^A \ra A$ to $G'$\\
5 & For each edge $A \bb B$ in $G$\\
6 & \hspace{0.3cm} Add the edge $\epsilon^A \bb \epsilon^B$ to $G'$\\
7 & \hspace{0.3cm} Remove the edge $A \bb B$ from $G'$\\
\end{tabular}}
\end{table}

The transformation above basically consists in adding the error nodes $\epsilon^A$ to $G$ and connect them appropriately. Figure \ref{fig:example2} shows an example. Note that every node $A \in V$ is determined by $pa_{G'}(A)$ and, what will be more important, that $\epsilon^A$ is determined by $pa_{G'}(A) \setminus \epsilon^A \cup A$. Thus, the existence of deterministic nodes imposes independencies which do not correspond to separations in $G$. Note also that, given $Z \subseteq V$, a node $A \in V$ is determined by $Z$ iff $A \in Z$. The if part is trivial. To see the only if part, note that $\epsilon^A \notin Z$ and thus $A$ cannot be determined by $Z$ unless $A \in Z$. Therefore, a node $\epsilon^A$ in $G'$ is determined by $Z$ iff $pa_{G'}(A) \setminus \epsilon^A \cup A \subseteq Z$ because, as shown, there is no other way for $Z$ to determine $pa_{G'}(A) \setminus \epsilon^A \cup A$ which, in turn, determine $\epsilon^A$. Let $\epsilon$ denote all the error nodes in $G'$. It is easy to see that $G'$ is a MAMP CG over $V \cup \epsilon$ and, thus, its semantics are defined. The following theorem confirms that these semantics are as desired.

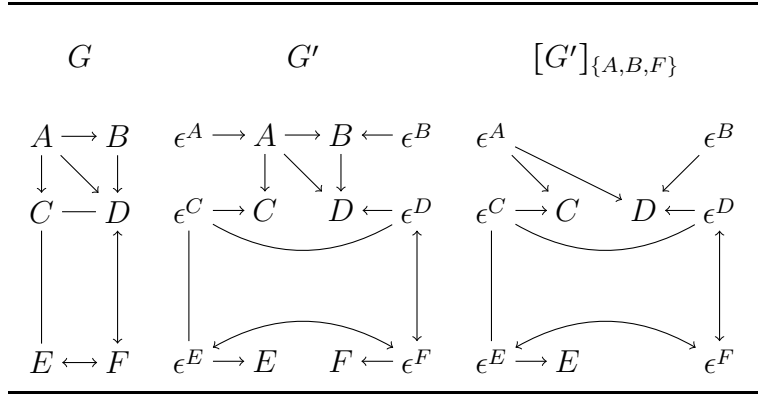
\begin{figure}
\centering
\scalebox{1.0}{
\begin{tabular}{ccc}\\
\hline
\\
$G$&$G'$&$[G']_{\{A,B,F\}}$\\
\\
\begin{tikzpicture}[inner sep=1mm]
\node at (0,0) (A) {$A$};
\node at (1,0) (B) {$B$};
\node at (0,-1) (C) {$C$};
\node at (1,-1) (D) {$D$};
\node at (0,-3) (E) {$E$};
\node at (1,-3) (F) {$F$};
\path[->] (A) edge (B);
\path[->] (A) edge (C);
\path[->] (A) edge (D);
\path[->] (B) edge (D);
\path[-] (C) edge (D);
\path[-] (C) edge (E);
\path[<->] (D) edge (F);
\path[<->] (E) edge (F);
\end{tikzpicture}
&
\begin{tikzpicture}[inner sep=1mm]
\node at (0,0) (A) {$A$};
\node at (1,0) (B) {$B$};
\node at (0,-1) (C) {$C$};
\node at (1,-1) (D) {$D$};
\node at (0,-3) (E) {$E$};
\node at (1,-3) (F) {$F$};
\node at (-1,0) (EA) {$\epsilon^A$};
\node at (2,0) (EB) {$\epsilon^B$};
\node at (-1,-1) (EC) {$\epsilon^C$};
\node at (2,-1) (ED) {$\epsilon^D$};
\node at (-1,-3) (EE) {$\epsilon^E$};
\node at (2,-3) (EF) {$\epsilon^F$};
\path[->] (EA) edge (A);
\path[->] (EB) edge (B);
\path[->] (EC) edge (C);
\path[->] (ED) edge (D);
\path[->] (EE) edge (E);
\path[->] (EF) edge (F);
\path[->] (A) edge (B);
\path[->] (A) edge (C);
\path[->] (A) edge (D);
\path[->] (B) edge (D);
\path[-] (EC) edge [bend right] (ED);
\path[-] (EC) edge (EE);
\path[<->] (ED) edge (EF);
\path[<->] (EE) edge [bend left] (EF);
\end{tikzpicture}
&
\begin{tikzpicture}[inner sep=1mm]
\node at (0,-1) (C) {$C$};
\node at (1,-1) (D) {$D$};
\node at (0,-3) (E) {$E$};
\node at (-1,0) (EA) {$\epsilon^A$};
\node at (2,0) (EB) {$\epsilon^B$};
\node at (-1,-1) (EC) {$\epsilon^C$};
\node at (2,-1) (ED) {$\epsilon^D$};
\node at (-1,-3) (EE) {$\epsilon^E$};
\node at (2,-3) (EF) {$\epsilon^F$};
\path[->] (EA) edge (C);
\path[->] (EA) edge (D);
\path[->] (EB) edge (D);
\path[->] (EC) edge (C);
\path[->] (ED) edge (D);
\path[->] (EE) edge (E);
\path[-] (EC) edge [bend right] (ED);
\path[-] (EC) edge (EE);
\path[<->] (ED) edge (EF);
\path[<->] (EE) edge [bend left] (EF);
\end{tikzpicture}\\
\hline
\end{tabular}}\caption{Example of the different transformations for MAMP CGs.}\label{fig:example2}
\end{figure}

\begin{theorem}\label{the:GG'2}
$I(G)=[I(G')]_\epsilon^\emptyset$.
\end{theorem}

\subsection{EMAMP CGs Are Closed under Marginalization}

Finally, we show that EMAMP CGs are closed under marginalization, meaning that for any EMAMP CG $G'$ and $L \subseteq V$ there is an EMAMP CG $[G']_L$ st $[I(G')]_{L \cup \epsilon}=[I([G']_L)]_\epsilon$. We actually show how to transform $G'$ into $[G']_L$. Note that our definition of closed under marginalization is an adaptation of the standard one to the fact that we only care about independence models under marginalization of the error nodes.

To gain some intuition into the problem and our solution to it, assume that $L$ contains a single node $B$. Then, marginalizing out $B$ from the system of linear equations associated with $G$ implies the following: For every $C$ st $B \in pa_{G}(C)$, modify the equation $C = \beta_C \: pa_{G}(C) + \epsilon^C$ by replacing $B$ with the right-hand side of its corresponding equation, i.e. $\beta_B \: pa_{G}(B) + \epsilon^B$ and, then, remove the equation $B = \beta_B \: pa_{G}(B) + \epsilon^B$ from the system. In graphical terms, this corresponds to $C$ inheriting the parents of $B$ in $G'$ and, then, removing $B$ from $G'$. The following pseudocode formalizes this idea for any $L \subseteq V$.

\begin{table}[H]
\centering
\scalebox{1.0}{
\begin{tabular}{ll}
1 & Let $[G']_L=G'$\\
2 & Repeat until all the nodes in $L$ have been considered\\
3 & \hspace{0.3cm} Let $B$ denote any node in $L$ that has not been considered before\\
4 & \hspace{0.3cm} For each pair of edges $A \ra B$ and $B \ra C$ in $[G']_L$ with $A, C \in V \cup \epsilon$\\
5 & \hspace{0.8cm} Add the edge $A \ra C$ to $[G']_L$\\
6 & \hspace{0.3cm} Remove $B$ and all the edges it participates in from $[G']_L$\\
\end{tabular}}
\end{table}

Note that the result of the pseudocode above is the same no matter the ordering in which the nodes in $L$ are selected in line 3. Note also that we have not yet given a formal definition of EMAMP CGs. We define them recursively as all the graphs resulting from applying the first pseudocode in this section to a MAMP CG, plus all the graphs resulting from applying the second pseudocode in this section to an EMAMP CG. It is easy to see that every EMAMP CG is a MAMP CG over $W \cup \epsilon$ with $W \subseteq V$ and, thus, its semantics are defined. Theorem \ref{the:GG'2} together with the following theorem confirm that these semantics are as desired.

\begin{theorem}\label{the:closed}
$[I(G')]_{L \cup \epsilon}=[I([G']_L)]_\epsilon$.
\end{theorem}

\begin{figure}
\centering
\scalebox{1.0}{
\begin{tabular}{c}\\
\hline
\\
\begin{tikzpicture}[inner sep=1mm]
\node at (0,0) (A) {MAMP CGs};
\node at (2,-1) (B) {RCGs};
\node at (-2,-1) (C) {AMP CGs};
\node at (4,-2) (D) {MVR CGs};
\node at (0,-2) (E) {Markov networks};
\node at (4,-3) (F) {Covariance graphs};
\node at (-2,-3) (G) {Bayesian networks};
\path[->] (A) edge (B);
\path[->] (A) edge (C);
\path[->] (B) edge (D);
\path[->] (B) edge (E);
\path[->] (C) edge (E);
\path[->] (D) edge (F);
\path[->] (C) edge (G);
\path[->] (D) edge (G);
\end{tikzpicture}\\
\hline
\end{tabular}}\caption{Subfamilies of MAMP CGs.}\label{fig:subfamilies}
\end{figure}
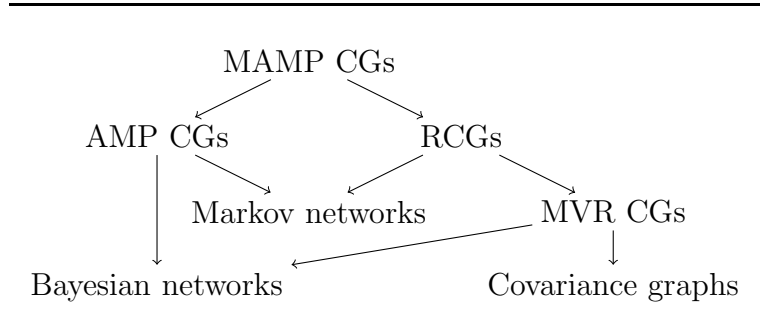

\section{Discussion}\label{sec:discussion}

In this paper we have introduced MAMP CGs, a new family of graphical models that unify and generalize AMP and MVR CGs. We have described global and pairwise Markov properties for them and proved their equivalence for compositional graphoids. We have shown that every MAMP CG is Markov equivalent to some DAG with deterministic nodes under marginalization and conditioning on some of its nodes. Therefore, the independence model represented by a MAMP CG can be accounted for by some data generating process that is partially observed and has selection bias. We have also characterized when two MAMP CGs are Markov equivalent. We conjecture that every Markov equivalence class of MAMP CGs has a distinguished member. We are currently working on this question. It is worth mentioning that such a result has been proven for AMP CGs \citep{RoveratoandStudeny2006}. Finally, we have modified MAMP CGs so that they are closed under marginalization. This is a desirable feature because it guarantees parsimonious models under marginalization. We are currently studying how to modify MAMP CGs so that they are closed under conditioning too. We are also working on a constraint based algorithm for learning a MAMP CG a given probability distribution is faithful to. The idea is to combine the learning algorithms that we have recently proposed for AMP CGs \citep{Penna2012} and MVR CGs \citep{SonntagandPenna2012}.

We believe that the most natural way to generalize AMP and MVR CGs is by allowing undirected, directed and bidirected edges. However, we are not the first to introduce a family of models that is based on graphs that may contain these three types of edges. In the rest of this section, we review some works that have done it before us, and explain how our work differs from them. \cite{CoxandWermuth1993,CoxandWermuth1996} introduced regression CGs (RCGs) to generalize MVR CGs by allowing them to have also undirected edges. The separation criterion for RCGs is identical to that of MVR CGs. Then, there are independence models that can be represented by MAMP CGs but that cannot be represented by RCGs, because RCGs generalize MVR CGs but not AMP CGs. An example follows.

\begin{example}
The independence model represented by the AMP CG $G$ below cannot be represented by any RCG.

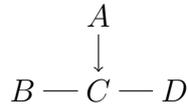
\begin{table}[H]
\centering
\scalebox{1.0}{
\begin{tabular}{c}
\begin{tikzpicture}[inner sep=1mm]
\node at (1,1) (A) {$A$};
\node at (0,0) (B) {$B$};
\node at (1,0) (C) {$C$};
\node at (2,0) (D) {$D$};
\path[->] (A) edge (C);
\path[-] (B) edge (C);
\path[-] (C) edge (D);
\end{tikzpicture}
\end{tabular}}
\end{table}

To see it, assume to the contrary that it can be represented by a RCG $H$. Note that $H$ is a MAMP CG too. Then, $G$ and $H$ must have the same triplexes by Theorem \ref{the:triplex}. Then, $H$ must have triplexes $(\{A,B\},C)$ and $(\{A,D\},C)$ but no triplex $(\{B,D\},C)$. So, $B \oo C \ra D$, $B \oo C - D$, $B \la C \oo D$ or $B - C \oo D$ must be in $H$. However, this implies that $H$ does not have the triplex $(\{A,B\},C)$ or $(\{A,D\},C)$, which is a contradiction.
\end{example}

It is worth mentioning that, although RCGs can have undirected edges, they cannot have a subgraph of the form $A \oa B - C$. Therefore, RCGs are a subfamily of MAMP CGs. Figure \ref{fig:subfamilies} depicts this and other subfamilies of MAMP CGs.

Another family of models that is based on graphs that may contain undirected, directed and bidirected edges is maximal ancestral graphs (MAGs) \citep{RichardsonandSpirtes2002}. Although MAGs can have undirected edges, they must comply with certain topological constraints. The separation criterion for MAGs is identical to that of MVR CGs. Therefore, the example above also serves to illustrate that MAGs generalize MVR CGs but not AMP CGs, as MAMP CGs do. See also \citep[p. 1025]{RichardsonandSpirtes2002}. Therefore, MAMP CGs are not a subfamily of MAGs. The following example shows that MAGs are not a subfamily of MAMP CGs either.

\begin{example}\label{exa:reviewer}
The independence model represented by the MAG $G$ below cannot be represented by any MAMP CG.

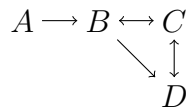
\begin{table}[H]
\centering
\scalebox{1.0}{
\begin{tabular}{c}
\begin{tikzpicture}[inner sep=1mm]
\node at (0,0) (A) {$A$};
\node at (1,0) (B) {$B$};
\node at (2,0) (C) {$C$};
\node at (2,-1) (D) {$D$};
\path[->] (A) edge (B);
\path[->] (B) edge (D);
\path[<->] (B) edge (C);
\path[<->] (C) edge (D);
\end{tikzpicture}
\end{tabular}}
\end{table}

To see it, assume to the contrary that it can be represented by a MAMP CG $H$. Obviously, $G$ and $H$ must have the same adjacencies. Then, $H$ must have a triplex $(\{A,C\},B)$ because $A \ci_G C$, but it cannot have a triplex $(\{A,D\},B)$ because $A \ci_G D | B$. This is possible only if the edge $A \la B$ is not in $H$. Then, $H$ must have one of the following induced subgraphs:

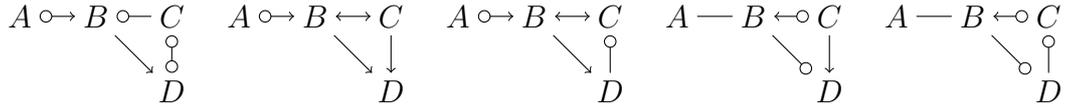
\begin{table}[H]
\centering
\scalebox{1.0}{
\begin{tabular}{ccccc}
\begin{tikzpicture}[inner sep=1mm]
\node at (0,0) (A) {$A$};
\node at (1,0) (B) {$B$};
\node at (2,0) (C) {$C$};
\node at (2,-1) (D) {$D$};
\path[o->] (A) edge (B);
\path[->] (B) edge (D);
\path[o-] (B) edge (C);
\path[o-o] (C) edge (D);
\end{tikzpicture}
&
\begin{tikzpicture}[inner sep=1mm]
\node at (0,0) (A) {$A$};
\node at (1,0) (B) {$B$};
\node at (2,0) (C) {$C$};
\node at (2,-1) (D) {$D$};
\path[o->] (A) edge (B);
\path[->] (B) edge (D);
\path[<->] (B) edge (C);
\path[->] (C) edge (D);
\end{tikzpicture}
&
\begin{tikzpicture}[inner sep=1mm]
\node at (0,0) (A) {$A$};
\node at (1,0) (B) {$B$};
\node at (2,0) (C) {$C$};
\node at (2,-1) (D) {$D$};
\path[o->] (A) edge (B);
\path[->] (B) edge (D);
\path[<->] (B) edge (C);
\path[o-] (C) edge (D);
\end{tikzpicture}
&
\begin{tikzpicture}[inner sep=1mm]
\node at (0,0) (A) {$A$};
\node at (1,0) (B) {$B$};
\node at (2,0) (C) {$C$};
\node at (2,-1) (D) {$D$};
\path[-] (A) edge (B);
\path[-o] (B) edge (D);
\path[<-o] (B) edge (C);
\path[->] (C) edge (D);
\end{tikzpicture}
&
\begin{tikzpicture}[inner sep=1mm]
\node at (0,0) (A) {$A$};
\node at (1,0) (B) {$B$};
\node at (2,0) (C) {$C$};
\node at (2,-1) (D) {$D$};
\path[-] (A) edge (B);
\path[-o] (B) edge (D);
\path[<-o] (B) edge (C);
\path[o-] (C) edge (D);
\end{tikzpicture}
\end{tabular}}
\end{table}

However, the first and second cases are impossible because $A \ci_H D | B \cup C$ whereas $A \nci_G D | B \cup C$. The third case is impossible because it does not satisfy the constraint C1. In the fourth case, note that $C \aa B - D$ cannot be in $H$ because, otherwise, it does not satisfy the constraint C1. Then, the fourth case is impossible because $A \ci_H D | B \cup C$ whereas $A \nci_G D | B \cup C$. Finally, the fifth case is also impossible because it does not satisfy the constraint C1 or C2.
\end{example}

It is worth mentioning that the models represented by AMP and MVR CGs are smooth, i.e. they are curved exponential families, for Gaussian probability distributions. However, only the models represented by MVR CGs are smooth for discrete probability distributions. The models represented by MAGs are smooth in the Gaussian and discrete cases. See \cite{Drton2009} and \cite{EvansandRichardson2013}.

Finally, three other families of models that are based on graphs that may contain undirected, directed and bidirected edges are summary graphs after replacing the dashed undirected edges with bidirected edges \citep{CoxandWermuth1996}, MC graphs \citep{Koster2002}, and loopless mixed graphs \citep{SadeghiandLauritzen2012}. As shown in \citep[Sections 4.2 and 4.3]{SadeghiandLauritzen2012}, every independence model that can be represented by summary graphs and MC graphs can also be represented by loopless mixed graphs. The separation criterion for loopless mixed graphs is identical to that of MVR CGs. Therefore, the example above also serves to illustrate that loopless mixed graphs generalize MVR CGs but not AMP CGs, as MAMP CGs do. See also \citep[Section 4.1]{SadeghiandLauritzen2012}. Moreover, summary graphs and MC graphs have a rather counterintuitive and undesirable feature: Not every missing edge corresponds to a separation \citep[p. 1023]{RichardsonandSpirtes2002}. MAMP CGs, on the other hand, do not have this disadvantage (recall Theorem \ref{the:pairwise1}).

In summary, MAMP CGs are the only graphical models we are aware of that generalize both AMP and MVR CGs.

\section*{Acknowledgments}

We would like to thank the anonymous Reviewers and specially Reviewer 3 for suggesting Example \ref{exa:reviewer}. This work is funded by the Center for Industrial Information Technology (CENIIT) and a so-called career contract at Link\"oping University, by the Swedish Research Council (ref. 2010-4808), and by FEDER funds and the Spanish Government (MICINN) through the project TIN2010-20900-C04-03.

\section*{Appendix: Proofs}

\begin{proof}[{\bf Proof of Theorem \ref{the:GG'}}]

It suffices to show that every $Z$-open path between $\alpha$ and $\beta$ in $G$ can be transformed into a $Z$-open path between $\alpha$ and $\beta$ in $G'$ and vice versa, with $\alpha, \beta \in V$ and $Z \subseteq V \setminus \alpha \setminus \beta$.

Let $\rho$ denote a $Z$-open path between $\alpha$ and $\beta$ in $G$. We can easily transform $\rho$ into a path $\rho'$ between $\alpha$ and $\beta$ in $G'$: Simply, replace every maximal subpath of $\rho$ of the form $V_1 - V_2 - \ldots - V_{n-1} - V_n$ ($n \geq 2$) with $V_1 \la \epsilon^{V_1} - \epsilon^{V_2} - \ldots - \epsilon^{V_{n-1}} - \epsilon^{V_n} \ra V_n$. We now show that $\rho'$ is $Z$-open.

First, if $B \in V$ is a triplex node in $\rho'$, then $\rho'$ must have one of the following subpaths:

\begin{table}[H]
\centering
\scalebox{1.0}{
\begin{tabular}{c}
\begin{tikzpicture}[inner sep=1mm]
\node at (0,0) (A) {$A$};
\node at (1,0) (B) {$B$};
\node at (2,0) (C) {$C$};
\path[->] (A) edge (B);
\path[<-] (B) edge (C);
\end{tikzpicture}
\begin{tikzpicture}[inner sep=1mm]
\node at (0,0) (A) {$A$};
\node at (1,0) (B) {$B$};
\node at (2,0) (C) {$\epsilon^B$};
\node at (3,0) (D) {$\epsilon^C$};
\path[->] (A) edge (B);
\path[<-] (B) edge (C);
\path[-] (D) edge (C);
\end{tikzpicture}
\begin{tikzpicture}[inner sep=1mm]
\node at (0,0) (A) {$\epsilon^B$};
\node at (1,0) (B) {$B$};
\node at (2,0) (C) {$C$};
\node at (-1,0) (D) {$\epsilon^A$};
\path[->] (A) edge (B);
\path[<-] (B) edge (C);
\path[-] (D) edge (A);
\end{tikzpicture}
\end{tabular}}
\end{table}

with $A, C \in V$. Therefore, $\rho$ must have one of the following subpaths (specifically, if $\rho'$ has the $i$-th subpath above, then $\rho$ has the $i$-th subpath below):

\begin{table}[H]
\centering
\scalebox{1.0}{
\begin{tabular}{c}
\begin{tikzpicture}[inner sep=1mm]
\node at (0,0) (A) {$A$};
\node at (1,0) (B) {$B$};
\node at (2,0) (C) {$C$};
\path[->] (A) edge (B);
\path[<-] (B) edge (C);
\end{tikzpicture}
\begin{tikzpicture}[inner sep=1mm]
\node at (0,0) (A) {$A$};
\node at (1,0) (B) {$B$};
\node at (2,0) (C) {$C$};
\path[->] (A) edge (B);
\path[-] (B) edge (C);
\end{tikzpicture}
\begin{tikzpicture}[inner sep=1mm]
\node at (0,0) (A) {$A$};
\node at (1,0) (B) {$B$};
\node at (2,0) (C) {$C$};
\path[-] (A) edge (B);
\path[<-] (B) edge (C);
\end{tikzpicture}
\end{tabular}}
\end{table}

In either case, $B$ is a triplex node in $\rho$ and, thus, $B \in Z \cup san_G(Z)$ for $\rho$ to be $Z$-open. Then, $B \in Z \cup san_{G'}(Z)$ by construction of $G'$ and, thus, $B \in D(Z) \cup san_{G'}(D(Z))$.

Second, if $B \in V$ is a non-triplex node in $\rho'$, then $\rho'$ must have one of the following subpaths:

\begin{table}[H]
\centering
\scalebox{1.0}{
\begin{tabular}{c}
\begin{tikzpicture}[inner sep=1mm]
\node at (0,0) (A) {$A$};
\node at (1,0) (B) {$B$};
\node at (2,0) (C) {$C$};
\path[->] (A) edge (B);
\path[->] (B) edge (C);
\end{tikzpicture}
\begin{tikzpicture}[inner sep=1mm]
\node at (0,0) (A) {$A$};
\node at (1,0) (B) {$B$};
\node at (2,0) (C) {$C$};
\path[<-] (A) edge (B);
\path[->] (B) edge (C);
\end{tikzpicture}
\begin{tikzpicture}[inner sep=1mm]
\node at (0,0) (A) {$A$};
\node at (1,0) (B) {$B$};
\node at (2,0) (C) {$C$};
\path[<-] (A) edge (B);
\path[<-] (B) edge (C);
\end{tikzpicture}
\begin{tikzpicture}[inner sep=1mm]
\node at (0,0) (A) {$A$};
\node at (1,0) (B) {$B$};
\node at (2,0) (C) {$\epsilon^B$};
\node at (3,0) (D) {$\epsilon^C$};
\path[<-] (A) edge (B);
\path[<-] (B) edge (C);
\path[-] (D) edge (C);
\end{tikzpicture}
\begin{tikzpicture}[inner sep=1mm]
\node at (0,0) (A) {$\epsilon^B$};
\node at (1,0) (B) {$B$};
\node at (2,0) (C) {$C$};
\node at (-1,0) (D) {$\epsilon^A$};
\path[->] (A) edge (B);
\path[->] (B) edge (C);
\path[-] (D) edge (A);
\end{tikzpicture}
\end{tabular}}
\end{table}

with $A, C \in V$. Therefore, $\rho$ must have one of the following subpaths (specifically, if $\rho'$ has the $i$-th subpath above, then $\rho$ has the $i$-th subpath below):

\begin{table}[H]
\centering
\scalebox{1.0}{
\begin{tabular}{c}
\begin{tikzpicture}[inner sep=1mm]
\node at (0,0) (A) {$A$};
\node at (1,0) (B) {$B$};
\node at (2,0) (C) {$C$};
\path[->] (A) edge (B);
\path[->] (B) edge (C);
\end{tikzpicture}
\begin{tikzpicture}[inner sep=1mm]
\node at (0,0) (A) {$A$};
\node at (1,0) (B) {$B$};
\node at (2,0) (C) {$C$};
\path[<-] (A) edge (B);
\path[->] (B) edge (C);
\end{tikzpicture}
\begin{tikzpicture}[inner sep=1mm]
\node at (0,0) (A) {$A$};
\node at (1,0) (B) {$B$};
\node at (2,0) (C) {$C$};
\path[<-] (A) edge (B);
\path[<-] (B) edge (C);
\end{tikzpicture}
\begin{tikzpicture}[inner sep=1mm]
\node at (0,0) (A) {$A$};
\node at (1,0) (B) {$B$};
\node at (2,0) (C) {$C$};
\path[<-] (A) edge (B);
\path[-] (B) edge (C);
\end{tikzpicture}
\begin{tikzpicture}[inner sep=1mm]
\node at (0,0) (A) {$A$};
\node at (1,0) (B) {$B$};
\node at (2,0) (C) {$C$};
\path[-] (A) edge (B);
\path[->] (B) edge (C);
\end{tikzpicture}
\end{tabular}}
\end{table}

In either case, $B$ is a non-triplex node in $\rho$ and, thus, $B \notin Z$ for $\rho$ to be $Z$-open. Since $Z$ contains no error node, $Z$ cannot determine any node in $V$ that is not already in $Z$. Then, $B \notin D(Z)$.

Third, if $\epsilon^B$ is a non-triplex node in $\rho'$ (note that $\epsilon^B$ cannot be a triplex node in $\rho'$), then $\rho'$ must have one of the following subpaths:

\begin{table}[H]
\centering
\scalebox{1.0}{
\begin{tabular}{c}
\begin{tikzpicture}[inner sep=1mm]
\node at (0,0) (A) {$A$};
\node at (1,0) (B) {$B$};
\node at (2,0) (C) {$\epsilon^B$};
\node at (3,0) (D) {$\epsilon^C$};
\path[->] (A) edge (B);
\path[<-] (B) edge (C);
\path[-] (D) edge (C);
\end{tikzpicture}
\begin{tikzpicture}[inner sep=1mm]
\node at (0,0) (A) {$\epsilon^B$};
\node at (1,0) (B) {$B$};
\node at (2,0) (C) {$C$};
\node at (-1,0) (D) {$\epsilon^A$};
\path[->] (A) edge (B);
\path[<-] (B) edge (C);
\path[-] (D) edge (A);
\end{tikzpicture}
\begin{tikzpicture}[inner sep=1mm]
\node at (0.65,0) (B) {$\alpha=B$};
\node at (2,0) (C) {$\epsilon^B$};
\node at (3,0) (D) {$\epsilon^C$};
\path[<-] (B) edge (C);
\path[-] (D) edge (C);
\end{tikzpicture}
\begin{tikzpicture}[inner sep=1mm]
\node at (0,0) (A) {$\epsilon^B$};
\node at (1.35,0) (B) {$B=\beta$};
\node at (-1,0) (D) {$\epsilon^A$};
\path[->] (A) edge (B);
\path[-] (D) edge (A);
\end{tikzpicture}\\
\begin{tikzpicture}[inner sep=1mm]
\node at (0,0) (A) {$A$};
\node at (1,0) (B) {$B$};
\node at (2,0) (C) {$\epsilon^B$};
\node at (3,0) (D) {$\epsilon^C$};
\path[<-] (A) edge (B);
\path[<-] (B) edge (C);
\path[-] (D) edge (C);
\end{tikzpicture}
\begin{tikzpicture}[inner sep=1mm]
\node at (0,0) (A) {$\epsilon^B$};
\node at (1,0) (B) {$B$};
\node at (2,0) (C) {$C$};
\node at (-1,0) (D) {$\epsilon^A$};
\path[->] (A) edge (B);
\path[->] (B) edge (C);
\path[-] (D) edge (A);
\end{tikzpicture}
\begin{tikzpicture}[inner sep=1mm]
\node at (0,0) (A) {$\epsilon^A$};
\node at (1,0) (B) {$\epsilon^B$};
\node at (2,0) (C) {$\epsilon^C$};
\path[-] (A) edge (B);
\path[-] (B) edge (C);
\end{tikzpicture}
\end{tabular}}
\end{table}

with $A, C \in V$. Recall that $\epsilon^B \notin Z$ because $Z \subseteq V \setminus \alpha \setminus \beta$. In the first case, if $\alpha=A$ then $A \notin Z$, else $A \notin Z$ for $\rho$ to be $Z$-open. Then, $\epsilon^B \notin D(Z)$. In the second case, if $\beta=C$ then $C \notin Z$, else $C \notin Z$ for $\rho$ to be $Z$-open. Then, $\epsilon^B \notin D(Z)$. In the third and fourth cases, $B \notin Z$ because $\alpha=B$ or $\beta=B$. Then, $\epsilon^B \notin D(Z)$. In the fifth and sixth cases, $B \notin Z$ for $\rho$ to be $Z$-open. Then, $\epsilon^B \notin D(Z)$. The last case implies that $\rho$ has the following subpath:

\begin{table}[H]
\centering
\scalebox{1.0}{
\begin{tabular}{c}
\begin{tikzpicture}[inner sep=1mm]
\node at (0,0) (A) {$A$};
\node at (1,0) (B) {$B$};
\node at (2,0) (C) {$C$};
\path[-] (A) edge (B);
\path[-] (B) edge (C);
\end{tikzpicture}
\end{tabular}}
\end{table}

Thus, $B$ is a non-triplex node in $\rho$, which implies that $B \notin Z$ or $pa_G(B) \setminus Z \neq \emptyset$ for $\rho$ to be $Z$-open. In either case, $\epsilon^B \notin D(Z)$ (recall that $pa_{G'}(B)=pa_G(B) \cup \epsilon^B$ by construction of $G'$).

Finally, let $\rho'$ denote a $Z$-open path between $\alpha$ and $\beta$ in $G'$. We can easily transform $\rho'$ into a path $\rho$ between $\alpha$ and $\beta$ in $G$: Simply, replace every maximal subpath of $\rho'$ of the form $V_1 \la \epsilon^{V_1} - \epsilon^{V_2} - \ldots - \epsilon^{V_{n-1}} - \epsilon^{V_n} \ra V_n$ ($n \geq 2$) with $V_1 - V_2 - \ldots - V_{n-1} - V_n$. We now show that $\rho$ is $Z$-open.

First, note that all the nodes in $\rho$ are in $V$. Moreover, if $B$ is a triplex node in $\rho$, then $\rho$ must have one of the following subpaths:

\begin{table}[H]
\centering
\scalebox{1.0}{
\begin{tabular}{c}
\begin{tikzpicture}[inner sep=1mm]
\node at (0,0) (A) {$A$};
\node at (1,0) (B) {$B$};
\node at (2,0) (C) {$C$};
\path[->] (A) edge (B);
\path[<-] (B) edge (C);
\end{tikzpicture}
\begin{tikzpicture}[inner sep=1mm]
\node at (0,0) (A) {$A$};
\node at (1,0) (B) {$B$};
\node at (2,0) (C) {$C$};
\path[->] (A) edge (B);
\path[-] (B) edge (C);
\end{tikzpicture}
\begin{tikzpicture}[inner sep=1mm]
\node at (0,0) (A) {$A$};
\node at (1,0) (B) {$B$};
\node at (2,0) (C) {$C$};
\path[-] (A) edge (B);
\path[<-] (B) edge (C);
\end{tikzpicture}
\end{tabular}}
\end{table}

with $A, C \in V$. Therefore, $\rho'$ must have one of the following subpaths (specifically, if $\rho$ has the $i$-th subpath above, then $\rho'$ has the $i$-th subpath below):

\begin{table}[H]
\centering
\scalebox{1.0}{
\begin{tabular}{c}
\begin{tikzpicture}[inner sep=1mm]
\node at (0,0) (A) {$A$};
\node at (1,0) (B) {$B$};
\node at (2,0) (C) {$C$};
\path[->] (A) edge (B);
\path[<-] (B) edge (C);
\end{tikzpicture}
\begin{tikzpicture}[inner sep=1mm]
\node at (0,0) (A) {$A$};
\node at (1,0) (B) {$B$};
\node at (2,0) (C) {$\epsilon^B$};
\node at (3,0) (D) {$\epsilon^C$};
\path[->] (A) edge (B);
\path[<-] (B) edge (C);
\path[-] (D) edge (C);
\end{tikzpicture}
\begin{tikzpicture}[inner sep=1mm]
\node at (0,0) (A) {$\epsilon^B$};
\node at (1,0) (B) {$B$};
\node at (2,0) (C) {$C$};
\node at (-1,0) (D) {$\epsilon^A$};
\path[->] (A) edge (B);
\path[<-] (B) edge (C);
\path[-] (D) edge (A);
\end{tikzpicture}
\end{tabular}}
\end{table}

In either case, $B$ is a triplex node in $\rho'$ and, thus, $B \in D(Z) \cup san_{G'}(D(Z))$ for $\rho'$ to be $Z$-open. Since $Z$ contains no error node, $Z$ cannot determine any node in $V$ that is not already in $Z$. Then, $B \in D(Z)$ iff $B \in Z$. Since there is no strictly descending route from $B$ to any error node, then any strictly descending route from $B$ to a node $D \in D(Z)$ implies that $D \in V$ which, as seen, implies that $D \in Z$. Then, $B \in san_{G'}(D(Z))$ iff $B \in san_{G'}(Z)$. Moreover, $B \in san_{G'}(Z)$ iff $B \in san_{G}(Z)$ by construction of $G'$. These results together imply that $B \in Z \cup san_{G}(Z)$.

Second, if $B$ is a non-triplex node in $\rho$, then $\rho$ must have one of the following subpaths:

\begin{table}[H]
\centering
\scalebox{1.0}{
\begin{tabular}{c}
\begin{tikzpicture}[inner sep=1mm]
\node at (0,0) (A) {$A$};
\node at (1,0) (B) {$B$};
\node at (2,0) (C) {$C$};
\path[->] (A) edge (B);
\path[->] (B) edge (C);
\end{tikzpicture}
\begin{tikzpicture}[inner sep=1mm]
\node at (0,0) (A) {$A$};
\node at (1,0) (B) {$B$};
\node at (2,0) (C) {$C$};
\path[<-] (A) edge (B);
\path[->] (B) edge (C);
\end{tikzpicture}
\begin{tikzpicture}[inner sep=1mm]
\node at (0,0) (A) {$A$};
\node at (1,0) (B) {$B$};
\node at (2,0) (C) {$C$};
\path[<-] (A) edge (B);
\path[<-] (B) edge (C);
\end{tikzpicture}
\begin{tikzpicture}[inner sep=1mm]
\node at (0,0) (A) {$A$};
\node at (1,0) (B) {$B$};
\node at (2,0) (C) {$C$};
\path[<-] (A) edge (B);
\path[-] (B) edge (C);
\end{tikzpicture}
\begin{tikzpicture}[inner sep=1mm]
\node at (0,0) (A) {$A$};
\node at (1,0) (B) {$B$};
\node at (2,0) (C) {$C$};
\path[-] (A) edge (B);
\path[->] (B) edge (C);
\end{tikzpicture}
\begin{tikzpicture}[inner sep=1mm]
\node at (0,0) (A) {$A$};
\node at (1,0) (B) {$B$};
\node at (2,0) (C) {$C$};
\path[-] (A) edge (B);
\path[-] (B) edge (C);
\end{tikzpicture}
\end{tabular}}
\end{table}

with $A, C \in V$. Therefore, $\rho'$ must have one of the following subpaths (specifically, if $\rho$ has the $i$-th subpath above, then $\rho'$ has the $i$-th subpath below):

\begin{table}[H]
\centering
\scalebox{1.0}{
\begin{tabular}{c}
\begin{tikzpicture}[inner sep=1mm]
\node at (0,0) (A) {$A$};
\node at (1,0) (B) {$B$};
\node at (2,0) (C) {$C$};
\path[->] (A) edge (B);
\path[->] (B) edge (C);
\end{tikzpicture}
\begin{tikzpicture}[inner sep=1mm]
\node at (0,0) (A) {$A$};
\node at (1,0) (B) {$B$};
\node at (2,0) (C) {$C$};
\path[<-] (A) edge (B);
\path[->] (B) edge (C);
\end{tikzpicture}
\begin{tikzpicture}[inner sep=1mm]
\node at (0,0) (A) {$A$};
\node at (1,0) (B) {$B$};
\node at (2,0) (C) {$C$};
\path[<-] (A) edge (B);
\path[<-] (B) edge (C);
\end{tikzpicture}
\begin{tikzpicture}[inner sep=1mm]
\node at (0,0) (A) {$A$};
\node at (1,0) (B) {$B$};
\node at (2,0) (C) {$\epsilon^B$};
\node at (3,0) (D) {$\epsilon^C$};
\path[<-] (A) edge (B);
\path[<-] (B) edge (C);
\path[-] (D) edge (C);
\end{tikzpicture}
\begin{tikzpicture}[inner sep=1mm]
\node at (0,0) (A) {$\epsilon^B$};
\node at (1,0) (B) {$B$};
\node at (2,0) (C) {$C$};
\node at (-1,0) (D) {$\epsilon^A$};
\path[->] (A) edge (B);
\path[->] (B) edge (C);
\path[-] (D) edge (A);
\end{tikzpicture}\\
\begin{tikzpicture}[inner sep=1mm]
\node at (0,0) (A) {$\epsilon^A$};
\node at (1,0) (B) {$\epsilon^B$};
\node at (2,0) (C) {$\epsilon^C$};
\path[-] (A) edge (B);
\path[-] (B) edge (C);
\end{tikzpicture}
\end{tabular}}
\end{table}

In the first five cases, $B$ is a non-triplex node in $\rho'$ and, thus, $B \notin D(Z)$ for $\rho'$ to be $Z$-open. Since $Z$ contains no error node, $Z$ cannot determine any node in $V$ that is not already in $Z$. Then, $B \notin Z$. In the last case, $\epsilon^B$ is a non-triplex node in $\rho'$ and, thus, $\epsilon^B \notin D(Z)$ for $\rho'$ to be $Z$-open. Then, $B \notin Z$ or $pa_{G'}(B) \setminus \epsilon^B \setminus Z \ \neq \emptyset$. Then, $B \notin Z$ or $pa_{G}(B) \setminus Z \ \neq \emptyset$ (recall that $pa_{G'}(B)=pa_G(B) \cup \epsilon^B$ by construction of $G'$).

\end{proof}

\begin{proof}[{\bf Proof of Theorem \ref{the:G'G'}}]

Assume for a moment that $G'$ has no deterministic node. Note that $G'$ has no induced subgraph of the form $A \ra B - C$ with $A, B, C \in V \cup \epsilon$. Such an induced subgraph is called a flag by \citet[pp. 40-41]{Anderssonetal.2001}. They also introduce the term biflag, whose definition is irrelevant here. What is relevant here is the observation that a CG cannot have a biflag unless it has some flag. Therefore, $G'$ has no biflags. Consequently, every probability distribution that is Markovian wrt $G'$ when interpreted as an AMP CG is also Markovian wrt $G'$ when interpreted as a LWF CG and vice versa \citep[Corollary 1]{Anderssonetal.2001}. Now, note that there are Gaussian probability distributions that are faithful to $G'$ when interpreted as an AMP CG \citep[Theorem 6.1]{Levitzetal.2001} as well as when interpreted as a LWF CG \citep[Theorems 1 and 2]{Penna2011}. Therefore, $I_{AMP}(G')=I_{LWF}(G')$. We denote this independence model by $I_{NDN}(G')$.

Now, forget the momentary assumption made above that $G'$ has no deterministic node. Recall that we assumed that $D(Z)$ is the same under the AMP and the LWF interpretations of $G'$ for all $Z \subseteq V \cup \epsilon$. Recall also that, from the point of view of the separations in an AMP or LWF CG, that a node is determined by the conditioning set has the same effect as if the node were in the conditioning set. Then, $X \ci_{G'} Y | Z$ is in $I_{AMP}(G')$ iff $X \ci_{G'} Y | D(Z)$ is in $I_{NDN}(G')$ iff $X \ci_{G'} Y | Z$ is in $I_{LWF}(G')$. Then, $I_{AMP}(G')=I_{LWF}(G')$.

\end{proof}

\begin{proof}[{\bf Proof of Theorem \ref{the:G'G''}}]

Assume for a moment that $G'$ has no deterministic node. Then, $G''$ has no deterministic node either. We show below that every $Z$-open route between $\alpha$ and $\beta$ in $G'$ can be transformed into a $(Z \cup S)$-open route between $\alpha$ and $\beta$ in $G''$ and vice versa, with $\alpha, \beta \in V \cup \epsilon$. This implies that $I_{LWF}(G')=[I(G'')]_\emptyset^S$. We denote this independence model by $I_{NDN}(G')$.

First, let $\rho'$ denote a $Z$-open route between $\alpha$ and $\beta$ in $G'$. Then, we can easily transform $\rho'$ into a $(Z \cup S)$-open route $\rho''$ between $\alpha$ and $\beta$ in $G''$: Simply, replace every edge $\epsilon^A - \epsilon^B$ in $\rho'$ with $\epsilon^A \ra S_{\epsilon^A\epsilon^B} \la \epsilon^B$. To see that $\rho''$ is actually $(Z \cup S)$-open, note that every collider section in $\rho'$ is due to a subroute of the form $A \ra B \la C$ with $A, B \in V$ and $C \in V \cup \epsilon$. Then, any node that is in a collider (respectively non-collider) section of $\rho'$ is also in a collider (respectively non-collider) section of $\rho''$.

Second, let $\rho''$ denote a $(Z \cup S)$-open route between $\alpha$ and $\beta$ in $G''$. Then, we can easily transform $\rho''$ into a $Z$-open route $\rho'$ between $\alpha$ and $\beta$ in $G'$: First, replace every subroute $\epsilon^A \ra S_{\epsilon^A\epsilon^B} \la \epsilon^A$ of $\rho''$ with $\epsilon^A$ and, then, replace every subroute $\epsilon^A \ra S_{\epsilon^A\epsilon^B} \la \epsilon^B$ of $\rho''$ with $\epsilon^A - \epsilon^B$. To see that $\rho'$ is actually $Z$-open, note that every undirected edge in $\rho'$ is between two noise nodes and recall that no noise node has incoming directed edges in $G'$. Then, again every collider section in $\rho'$ is due to a subroute of the form $A \ra B \la C$ with $A, B \in V$ and $C \in V \cup \epsilon$. Then, again any node that is in a collider (respectively non-collider) section of $\rho'$ is also in a collider (respectively non-collider) section of $\rho''$.

Now, forget the momentary assumption made above that $G'$ has no deterministic node. Recall that we assumed that $D(Z)$ is the same no matter whether we are considering $G'$ or $G''$ for all $Z \subseteq V \cup \epsilon$. Recall also that, from the point of view of the separations in a LWF CG, that a node is determined by the conditioning set has the same effect as if the node were in the conditioning set. Then, $X \ci_{G''} Y | Z$ is in $[I(G'')]_\emptyset^S$ iff $X \ci_{G'} Y | D(Z)$ is in $I_{NDN}(G')$ iff $X \ci_{G'} Y | Z$ is in $I_{LWF}(G')$. Then, $I_{LWF}(G')=[I(G'')]_\emptyset^S$.

\end{proof}

\begin{proof}[{\bf Proof of Theorem \ref{the:faithfulness}}]

It suffices to replace every bidirected edge $A \aa B$ in $G$ with $A \la L_{AB} \ra B$ to create an AMP CG $\widehat{G}$, apply Theorem 6.1 by \cite{Levitzetal.2001} to conclude that there exists a regular Gaussian probability distribution $q$ that is faithful to $\widehat{G}$, and then let $p$ be the marginal probability distribution of $q$ over $V$.

\end{proof}

\begin{proof}[{\bf Proof of Corollary \ref{cor:wtc}}]

It follows from Theorem \ref{the:faithfulness} by just noting that the set of independencies in any regular Gaussian probability distribution satisfies the compositional graphoid properties \cite[Sections 2.2.2, 2.3.5 and 2.3.6]{Studeny2005}.

\end{proof}

\begin{proof}[{\bf Proof of Theorem \ref{the:pairwise1}}]

Since the independence model represented by $G$ satisfies the compositional graphoid properties by Corollary \ref{cor:wtc}, it suffices to prove that the pairwise separation base of $G$ is a subset of the independence model represented by $G$. We prove this next. Let $A, B \in V$ st $A \notin ad_G(B)$. Consider the following cases.

\begin{description}
\item[Case 1] Assume that $B \notin de_G(A)$. Then, every path between $A$ and $B$ in $G$ falls within one of the following cases.

\begin{description}
\item[Case 1.1] $A=V_1 \la V_2 \ldots V_n=B$. Then, this path is not $pa_G(A)$-open.

\item[Case 1.2] $A=V_1 \oa V_2 \ldots V_n=B$. Note that $V_2 \neq V_n$ because $A \notin ad_G(B)$. Note also that $V_2 \notin pa_G(A)$ due to the constraint C1. Then, $V_2 \ra V_3$ must be in $G$ for the path to be $pa_G(A)$-open. By repeating this reasoning, we can conclude that $A=V_1 \oa V_2 \ra V_3 \ra \ldots \ra V_n=B$ is in $G$. However, this contradicts that $B \notin de_G(A)$.

\item[Case 1.3] $A=V_1 - V_2 - \ldots - V_m \ao V_{m+1} \ldots V_n=B$. Note that $V_m \notin pa_G(A)$ due to the constraint C1. Then, this path is not $pa_G(A)$-open.

\item[Case 1.4] $A=V_1 - V_2 - \ldots - V_m \ra V_{m+1} \ldots V_n=B$. Note that $V_{m+1} \neq V_n$ because $B \notin de_G(A)$. Note also that $V_{m+1} \notin pa_G(A)$ due to the constraint C1. Then, $V_{m+1} \ra V_{m+2}$ must be in $G$ for the path to be $pa_G(A)$-open. By repeating this reasoning, we can conclude that $A=V_1 - V_2 - \ldots - V_m \ra V_{m+1} \ra \ldots \ra V_n=B$ is in $G$. However, this contradicts that $B \notin de_G(A)$.

\item[Case 1.5] $A=V_1 - V_2 - \ldots - V_n=B$. This case contradicts the assumption that $B \notin de_G(A)$.
\end{description}

\item[Case 2] Assume that $A \in de_G(B)$, $B \in de_G(A)$ and $uc_G(A)=uc_G(B)$. Then, there is an undirected path $\rho$ between $A$ and $B$ in $G$. Then, every path between $A$ and $B$ in $G$ falls within one of the following cases.

\begin{description}
\item[Case 2.1] $A=V_1 \la V_2 \ldots V_n=B$. Then, this path is not $(ne_G(A) \cup pa_G(A \cup ne_G(A)))$-open.

\item[Case 2.2] $A=V_1 \oa V_2 \ldots V_n=B$. Note that $V_2 \neq V_n$ because $A \notin ad_G(B)$. Note also that $V_2 \notin ne_G(A) \cup pa_G(A \cup ne_G(A))$ due to the constraints C1 and C2. Then, $V_2 \ra V_3$ must be in $G$ for the path to be $(ne_G(A) \cup pa_G(A \cup ne_G(A)))$-open. By repeating this reasoning, we can conclude that $A=V_1 \oa V_2 \ra V_3 \ra \ldots \ra V_n=B$ is in $G$. However, this together with $\rho$ violate the constraint C1.

\item[Case 2.3] $A=V_1 - V_2 \la V_3 \ldots V_n=B$. Then, this path is not $(ne_G(A) \cup pa_G(A \cup ne_G(A)))$-open.

\item[Case 2.4] $A=V_1 - V_2 \oa V_3 \ldots V_n=B$. Note that $V_3 \neq V_n$ due to $\rho$ and the constraints C1 and C2. Note also that $V_3 \notin ne_G(A) \cup pa_G(A \cup ne_G(A))$ due to the constraints C1 and C2. Then, $V_3 \ra V_4$ must be in $G$ for the path to be $(ne_G(A) \cup pa_G(A \cup ne_G(A)))$-open. By repeating this reasoning, we can conclude that $A=V_1 - V_2 \oa V_3 \ra \ldots \ra V_n=B$ is in $G$. However, this together with $\rho$ violate the constraint C1.

\item[Case 2.5] $A=V_1 - V_2 - V_3 \ldots V_n=B$ st $sp_G(V_2) = \emptyset$. Then, this path is not $(ne_G(A) \cup pa_G(A \cup ne_G(A)))$-open.

\item[Case 2.6] $A=V_1 - V_2 - \ldots - V_n=B$ st $sp_G(V_i) \neq \emptyset$ for all $2 \leq i \leq n-1$. Note that $V_i \in ne_G(V_1)$ for all $3 \leq i \leq n$ by the constraint C3. However, this contradicts that $A \notin ad_G(B)$.

\item[Case 2.7] $A=V_1 - V_2 - \ldots - V_m - V_{m+1} - V_{m+2} \ldots V_n=B$ st $sp_G(V_i) \neq \emptyset$ for all $2 \leq i \leq m$ and $sp_G(V_{m+1}) = \emptyset$. Note that $V_i \in ne_G(V_1)$ for all $3 \leq i \leq m+1$ by the constraint C3. Then, this path is not $(ne_G(A) \cup pa_G(A \cup ne_G(A)))$-open.

\item[Case 2.8] $A=V_1 - V_2 - \ldots - V_m - V_{m+1} \la V_{m+2} \ldots V_n=B$ st $sp_G(V_i) \neq \emptyset$ for all $2 \leq i \leq m$. Note that $V_i \in ne_G(V_1)$ for all $3 \leq i \leq m+1$ by the constraint C3. Then, this path is not $(ne_G(A) \cup pa_G(A \cup ne_G(A)))$-open.

\item[Case 2.9] $A=V_1 - V_2 - \ldots - V_m - V_{m+1} \oa V_{m+2} \ldots V_n=B$ st $sp_G(V_i) \neq \emptyset$ for all $2 \leq i \leq m$. Note that $V_{m+2} \neq V_n$ due to $\rho$ and the constraints C1 and C2. Note also that $V_{m+2} \notin ne_G(A) \cup pa_G(A \cup ne_G(A))$ due to the constraints C1 and C2. Then, $V_{m+2} \ra V_{m+3}$ must be in $G$ for the path to be $(ne_G(A) \cup pa_G(A \cup ne_G(A)))$-open. By repeating this reasoning, we can conclude that $A=V_1 - V_2 - \ldots - V_m - V_{m+1} \oa V_{m+2} \ra \ldots \ra V_n=B$ is in $G$. However, this together with $\rho$ violate the constraint C1.
\end{description}

\item[Case 3] Assume that $A \in de_G(B)$, $B \in de_G(A)$ and $uc_G(A) \neq uc_G(B)$. Then, every path between $A$ and $B$ in $G$ falls within one of the following cases.

\begin{description}
\item[Case 3.1] $A=V_1 \la V_2 \ldots V_n=B$. Then, this path is not $pa_G(A)$-open.

\item[Case 3.2] $A=V_1 \oa V_2 \ldots V_n=B$. Note that $V_2 \neq V_n$ because $A \notin ad_G(B)$. Note also that $V_2 \notin pa_G(A)$ due to the constraint C1. Then, $V_2 \ra V_3$ must be in $G$ for the path to be $pa_G(A)$-open. By repeating this reasoning, we can conclude that $A=V_1 \oa V_2 \ra V_3 \ra \ldots \ra V_n=B$ is in $G$. However, this together with the assumption that $A \in de_G(B)$ contradict the constraint C1.

\item[Case 3.3] $A=V_1 - V_2 - \ldots - V_m \ao V_{m+1} \ldots V_n=B$. Note that $V_m \notin pa_G(A)$ due to the constraint C1. Then, this path is not $pa_G(A)$-open.

\item[Case 3.4] $A=V_1 - V_2 - \ldots - V_m \ra V_{m+1} \ldots V_n=B$. Note that $V_{m+1} \neq V_n$ because, otherwise, this together with the assumption that $A \in de_G(B)$ contradict the constraint C1. Note also that $V_{m+1} \notin pa_G(A)$ due to the constraint C1. Then, $V_{m+1} \ra V_{m+2}$ must be in $G$ for the path to be $pa_G(A)$-open. By repeating this reasoning, we can conclude that $A=V_1 - V_2 - \ldots - V_m \ra V_{m+1} \ra \ldots \ra V_n=B$ is in $G$. However, this together with the assumption that $A \in de_G(B)$ contradict the constraint C1.
\end{description}

\end{description}

\end{proof}

\begin{lemma}\label{lem:aux}
Let $X$ and $Y$ denote two nodes of a MAMP CG $G$ with only one connectivity component. If $X \ci_G Y | Z$ and there is a node $C \in Z$ st $sp_G(C) \neq \emptyset$, then $X \ci_G Y | Z \setminus C$.
\end{lemma}

\begin{proof}

Assume to the contrary that there is a $(Z \setminus C)$-open path $\rho$ between $X$ and $Y$ in $G$. Note that $C$ must occur in $\rho$ because, otherwise, $\rho$ is $Z$-open which contradicts that $X \ci_G Y | Z$. For the same reason, $C$ must be a non-triplex node in $\rho$. Then, $D - C - E$ must be a subpath of $\rho$ and, thus, the edge $D - E$ must be in $G$ by the constraint C3, because $sp_G(C) \neq \emptyset$. Then, the path obtained from $\rho$ by replacing the subpath $D - C - E$ with the edge $D - E$ is $Z$-open. However, this contradicts that $X \ci_G Y | Z$.

\end{proof}

\begin{lemma}\label{lem:aux2}
Let $X$ and $Y$ denote two nodes of a MAMP CG $G$ with only one connectivity component. If $X \ci_G Y | Z$ then $X \ci_{cl(G)} Y | Z$.
\end{lemma}

\begin{proof}

We prove the lemma by induction on $|Z|$. If $|Z|=0$, then $uc_G(X) \neq uc_G(Y)$. Consequently, $X \ci_{cl(G)} Y$ follows from the pairwise separation base of $G$ because $X \notin ad_G(Y)$. Assume as induction hypothesis that the lemma holds for $|Z|<l$. We now prove it for $|Z|=l$. Consider the following cases.

\begin{description}

\item[Case 1] Assume that $uc_G(X) = uc_G(Y)$. Consider the following cases.

\begin{description}

\item[Case 1.1] Assume that $Z \subseteq uc_G(X)$. Then, the pairwise separation base of $G$ implies that $C \ci_{cl(G)} uc_G(X) \setminus C \setminus ne_G(C) | ne_G(C)$ for all $C \in uc_G(X)$ by repeated composition, which implies $X \ci_{cl(G)} Y | Z$ by the graphoid properties \cite[Theorem 3.7]{Lauritzen1996}. 

\item[Case 1.2] Assume that there is some node $C \in Z \setminus uc_G(X)$ st $C \aa D$ is in $G$ with $D \in uc_G(X)$ and $X \nci_G C | Z \setminus C$. Then, $Y \ci_G C | Z \setminus C$. To see it, assume the contrary. Then, $X \nci_G Y | Z \setminus C$ by weak transitivity because $X \ci_G Y | Z$. However, this contradicts Lemma \ref{lem:aux}.

Now, note that $Y \ci_G C | Z \setminus C$ implies $Y \ci_{cl(G)} C | Z \setminus C$ by the induction hypothesis. Note also that $X \ci_G Y | Z \setminus C$ by Lemma \ref{lem:aux} and, thus, $X \ci_{cl(G)} Y | Z \setminus C$ by the induction hypothesis. Then, $X \ci_{cl(G)} Y | Z$ by symmetry, composition and weak union.

\item[Case 1.3] Assume that Cases 1.1 and 1.2 do not apply. Let $E \in Z \setminus uc_G(X)$. Such a node $E$ exists because, otherwise, Case 1.1 applies. Moreover, $X \ci_G E | Z \setminus E$ because, otherwise, there is some node $C$ that satisfies the conditions of Case 1.2. Note also that $X \ci_G Y | Z \setminus E$. To see it, assume the contrary. Then, there is a $(Z \setminus E)$-open path between $X$ and $Y$ in $G$. Note that $E$ must occur in the path because, otherwise, the path is $Z$-open, which contradicts that $X \ci_G Y | Z$. However, this implies that $X \nci_G E | Z \setminus E$, which is a contradiction.

Now, note that $X \ci_G E | Z \setminus E$ and $X \ci_G Y | Z \setminus E$ imply $X \ci_{cl(G)} E | Z \setminus E$ and $X \ci_{cl(G)} Y | Z \setminus E$ by the induction hypothesis. Then, $X \ci_{cl(G)} Y | Z$ by composition and weak union.

\end{description}

\item[Case 2] Assume that $uc_G(X) \neq uc_G(Y)$. Consider the following cases.

\begin{description}

\item[Case 2.1] Assume that there is some node $C \in Z$ st $C \aa X$ is in $G$. Then, $Y \ci_G C | Z \setminus C$ because, otherwise, $X \nci_G Y | Z$. Then, $Y \ci_{cl(G)} C | Z \setminus C$ by the induction hypothesis. Note that $X \ci_G Y | Z \setminus C$ by Lemma \ref{lem:aux} and, thus, $X \ci_{cl(G)} Y | Z \setminus C$ by the induction hypothesis. Then, $X \ci_{cl(G)} Y | Z$ by symmetry, composition and weak union.

\item[Case 2.2] Assume that there is some node $C \in Z \cap uc_G(X)$ st $sp_G(C) \neq \emptyset$, and $X \ci_G C | Z \setminus C$. Then, $X \ci_{cl(G)} C | Z \setminus C$ by the induction hypothesis. Note that $X \ci_G Y | Z \setminus C$ by Lemma \ref{lem:aux} and, thus, $X \ci_{cl(G)} Y | Z \setminus C$ by the induction hypothesis. Then, $X \ci_{cl(G)} Y | Z$ by composition and weak union.

\item[Case 2.3] Assume that there is some node $C \in Z \cap uc_G(X)$ st $sp_G(C) \neq \emptyset$, and $X \nci_G C | Z \setminus C$. Then, $Y \ci_G C | Z \setminus C$. To see it, assume the contrary. Then, $X \nci_G Y | Z \setminus C$ by weak transitivity because $X \ci_G Y | Z$. However, this contradicts Lemma \ref{lem:aux}.

Now, note that $Y \ci_G C | Z \setminus C$ implies $Y \ci_{cl(G)} C | Z \setminus C$ by the induction hypothesis. Note also that $X \ci_G Y | Z \setminus C$ by Lemma \ref{lem:aux} and, thus, $X \ci_{cl(G)} Y | Z \setminus C$ by the induction hypothesis. Then, $X \ci_{cl(G)} Y | Z$ by composition and weak union.

\item[Case 2.4] Assume that Cases 2.1-2.3 do not apply. Let $V_1, \ldots, V_m$ be the nodes in $Z \cap uc_G(X)$. Let $W_1, \ldots, W_n$ be the nodes in $Z \setminus uc_G(X)$. Then, 

\begin{enumerate}

\item $X \ci_{cl(G)} Y$ follows from the pairwise separation base of $G$ because $uc_G(X) \neq uc_G(Y)$ and $X \notin ad_G(Y)$. Moreover, for all $1 \leq i \leq m$

\item $V_i \ci_{cl(G)} Y$ follows from the pairwise separation base of $G$ because $V_i \notin uc_G(Y)$ and $V_i \notin ad_G(Y)$, since $sp_G(V_i) = \emptyset$ because, otherwise, Case 2.2 or 2.3 applies. Moreover, for all $1 \leq j \leq n$

\item $X \ci_{cl(G)} W_j$ follows from the pairwise separation base of $G$ because $W_j \notin uc_G(X)$ and $W_j \notin ad_G(X)$, since $W_j \aa X$ is not in $G$ because, otherwise, Case 2.1 applies. Moreover, for all $1 \leq i \leq m$ and $1 \leq j \leq n$

\item $V_i \ci_{cl(G)} W_j$ follows from the pairwise separation base of $G$ because $uc_G(V_i) \neq uc_G(W_j)$ and $V_i \notin ad_G(W_j)$, since $sp_G(V_i) = \emptyset$ because, otherwise, Case 2.2 or 2.3 applies. Then, 

\item $X \ci_{cl(G)} Y | Z$ by repeated symmetry, composition and weak union.

\end{enumerate}

\end{description}

\end{description}

\end{proof}

We sort the connectivity components of a MAMP CG $G$ as $K_1, \ldots, K_n$ st if $X \ra Y$ is in $G$, then $X \in K_i$ and $Y \in K_j$ with $i<j$. It is worth mentioning that, in the proofs below, we make use of the fact that the independence model represented by $G$ satisfies weak transitivity by Corollary \ref{cor:wtc}. Note, however, that this property is not used in the construction of $cl(G)$. In the expressions below, we give equal precedence to the operators set minus, set union and set intersection.

\begin{lemma}\label{lem:same}
Let $X$ and $Y$ denote two nodes of a MAMP CG $G$ st $X, Y \in K_m$, $X \ci_G Y | Z$ and $Z \cap (K_{m+1} \cup \ldots \cup K_n) = \emptyset$. Let $H$ denote the subgraph of $G$ induced by $K_m$. Let $W = Z \cap K_m$. Let $W_1$ denote a minimal (wrt set inclusion) subset of $W$ st $X \ci_H W \setminus W_1 | W_1$. Then, $X \ci_{cl(G)} Y | Z \cup pa_G(X \cup W_1)$.
\end{lemma}

\begin{proof}

We define the restricted separation base of $G$ as the following set of separations:

\begin{itemize}
\item[R1.] $A \ci B | ne_G(A)$ for all $A, B \in K_m$ st $A \notin ad_G(B)$ and $uc_G(A)=uc_G(B)$, and
\item[R2.] $A \ci B$ for all $A, B \in K_m$ st $A \notin ad_G(B)$ and $uc_G(A) \neq uc_G(B)$.
\end{itemize}

We define the extended separation base of $G$ as the following set of separations:

\begin{itemize}
\item[E1.] $A \ci B | ne_G(A) \cup pa_G(K_m)$ for all $A, B \in K_m$ st $A \notin ad_G(B)$ and $uc_G(A)=uc_G(B)$, and
\item[E2.] $A \ci B | pa_G(K_m)$ for all $A, B \in K_m$ st $A \notin ad_G(B)$ and $uc_G(A) \neq uc_G(B)$.
\end{itemize}

Note that the separations E1 (resp. E2) are in one-to-one correspondence with the separations R1 (resp. R2) st the latter can be obtained from the former by adding $pa_G(K_m)$ to the conditioning sets. Let $W_2=W \setminus W_1$. Then, $X \ci_H W_2 | W_1$ implies that $X \ci_{cl(H)} W_2 | W_1$ by Lemma \ref{lem:aux2}. Note also that the pairwise separation base of $H$ coincides with the restricted separation base of $G$. Then, $X \ci_{cl(H)} W_2 | W_1$ implies that $X \ci W_2 | W_1$ can be derived from the restricted separation base of $G$ by applying the compositional graphoid properties. We can now reuse this derivation to derive $X \ci W_2 | W_1 \cup pa_G(K_m)$ from the extended separation base of $G$ by applying the compositional graphoid properties: It suffices to apply the same sequence of properties but replacing any separation of the restricted separation base in the derivation with the corresponding separation of the extended separation base. In fact, $X \ci W_2 | W_1 \cup pa_G(K_m)$ is not only in the closure of the extended separation base of $G$ but also in the closure of the pairwise separation base of $G$, i.e. $X \ci_{cl(G)} W_2 | W_1 \cup pa_G(K_m)$. To show it, it suffices to show that the extended separation base is in the closure of the pairwise separation base. Specifically, consider any $A, B \in K_m$ st $A \notin ad_G(B)$ and $uc_G(A) \neq uc_G(B)$. Then,

\begin{enumerate}

\item $A \ci_{cl(G)} B | pa_G(A)$ follows from the pairwise separation base of $G$, and

\item $A \ci_{cl(G)} pa_G(K_m) \setminus pa_G(A) | pa_G(A)$ follows from the pairwise separation base of $G$ by repeated composition. Then,

\item $A \ci_{cl(G)} B | pa_G(K_m)$ by composition on (1) and (2), and weak union.

\end{enumerate}

Now, consider any $A, B \in K_m$ st $A \notin ad_G(B)$ and $uc_G(A) = uc_G(B)$. Then,

\begin{enumerate}\setcounter{enumi}{3}

\item $A \ci_{cl(G)} B | ne_G(A) \cup pa_G(A \cup ne_G(A))$ follows from the pairwise separation base of $G$. Moreover, for any $C \in A \cup ne_G(A)$

\item $C \ci_{cl(G)} pa_G(K_m) \setminus pa_G(C) | pa_G(C)$ follows from the pairwise separation base of $G$ by repeated composition. Then,

\item $C \ci_{cl(G)} pa_G(K_m) \setminus pa_G(A \cup ne_G(A)) | pa_G(A \cup ne_G(A))$ by weak union. Then,

\item $A \ci_{cl(G)} pa_G(K_m) \setminus pa_G(A \cup ne_G(A)) | ne_G(A) \cup pa_G(A \cup ne_G(A))$ by repeated symmetry, composition and weak union. Then,

\item $A \ci_{cl(G)} B | ne_G(A) \cup pa_G(K_m)$ by composition on (4) and (7), and weak union.

\end{enumerate}

Note that $X \ci_H Y | W_1$ because, otherwise, $X \nci_G Y | Z$ which is a contradiction. Then, we can repeat the reasoning above to show that $X \ci_{cl(G)} Y | W_1 \cup pa_G(K_m)$. Then, $X \ci_{cl(G)} Y \cup W_2 | W_1 \cup pa_G(K_m)$ by composition on $X \ci_{cl(G)} W_2 | W_1 \cup pa_G(K_m)$. Finally, we show that this implies that $X \ci_{cl(G)} Y | Z \cup pa_G(X \cup W_1)$. Specifically,

\begin{enumerate}\setcounter{enumi}{8}

\item $X \ci_{cl(G)} Y \cup W_2 | W_1 \cup pa_G(K_m)$ as shown above. Moreover, for any $C \in X \cup W_1$

\item $C \ci_{cl(G)} pa_G(K_m) \setminus pa_G(C) | pa_G(C)$ follows from the pairwise separation base of $G$ by repeated composition. Then,

\item $C \ci_{cl(G)} pa_G(K_m) \setminus pa_G(X \cup W_1) | pa_G(X \cup W_1)$ by weak union. Then,

\item $X \ci_{cl(G)} pa_G(K_m) \setminus pa_G(X \cup W_1) | W_1 \cup pa_G(X \cup W_1)$ by repeated symmetry, composition and weak union. Then,

\item $X \ci_{cl(G)} Y \cup W_2 | W_1 \cup pa_G(X \cup W_1)$ by contraction on (9) and (12), and decomposition. Moreover, for any $C \in X \cup W_1$

\item $C \ci_{cl(G)} Z \setminus W \cup pa_G(X \cup W_1) \setminus pa_G(C) | pa_G(C)$ follows from the pairwise separation base of $G$ by repeated composition. Then,

\item $C \ci_{cl(G)} Z \setminus W \setminus pa_G(X \cup W_1) | pa_G(X \cup W_1)$ by weak union. Then,

\item $X \ci_{cl(G)} Z \setminus W \setminus pa_G(X \cup W_1) | W_1 \cup pa_G(X \cup W_1)$ by repeated symmetry, composition and weak union. Then,

\item $X \ci_{cl(G)} Y | Z \cup pa_G(X \cup W_1)$ by composition on (13) and (16), and weak union.

\end{enumerate}

\end{proof}

\begin{lemma}\label{lem:aux3}
Let $X$ and $Y$ denote two nodes of a MAMP CG $G$ st $Y \in K_1 \cup \ldots \cup K_m$, $X \in K_m$ and $X \ci_G Y | Z$. Let $H$ denote the subgraph of $G$ induced by $K_m$. Let $W = Z \cap K_m$. Let $W_1$ denote a minimal (wrt set inclusion) subset of $W$ st $X \ci_H W \setminus W_1 | W_1$. Then, $X \nci_G C | Z$ for all $C \in pa_G(X \cup W_1) \setminus Z$.
\end{lemma}

\begin{proof}

Note that $X \nci_H D | W \setminus D$ for all $D \in W_1$. To see it, assume the contrary. Then, $X \ci_H D | W \setminus D$ and $X \ci_H W \setminus W_1 | W_1$ imply $X \ci_H W \setminus W_1 \cup D | W_1 \setminus D$ by intersection, which contradicts the definition of $W_1$. Finally, note that $X \nci_H D | W \setminus D$ implies that there is a $(W \setminus D)$-open path between $X$ and $D$ in $G$ whose all nodes are in $K_m$. Then, $X \nci_G C | Z$ for all $C \in pa_G(X \cup W_1) \setminus Z$.

\end{proof}

\begin{lemma}\label{lem:different}
Let $X$ and $Y$ denote two nodes of a MAMP CG $G$ st $Y \in K_1 \cup \ldots \cup K_{m-1}$, $X \in K_m$, $X \ci_G Y | Z$ and $Z \cap (K_{m+1} \cup \ldots \cup K_n) = \emptyset$. Let $H$ denote the subgraph of $G$ induced by $K_m$. Let $W = Z \cap K_m$. Let $W_1$ denote a minimal (wrt set inclusion) subset of $W$ st $X \ci_H W \setminus W_1 | W_1$. Then, $X \ci_{cl(G)} Y | Z \cup pa_G(X \cup W_1)$.
\end{lemma}

\begin{proof}

Let $W_2=W \setminus W_1$. Note that $X \nci_G C | Z$ for all $C \in pa_G(X \cup W_1) \setminus Z$ by Lemma \ref{lem:aux3}, because $Y \in K_1 \cup \ldots \cup K_{m-1}$, $X \in K_m$ and $X \ci_G Y | Z$. Then, $Y \notin pa_G(X \cup W_1)$ because, otherwise, $X \nci_G Y | Z$ which is a contradiction. Moreover, for any $C \in X \cup W_1$

\begin{enumerate}

\item $C \ci_{cl(G)} Y \cup pa_G(K_m) \setminus pa_G(C) | pa_G(C)$ follows from the pairwise separation base of $G$ by repeated composition. Then,

\item $C \ci_{cl(G)} Y | pa_G(K_m)$ by weak union. Then,

\item $X \ci_{cl(G)} Y | W_1 \cup pa_G(K_m)$ by repeated symmetry, composition and weak union. Moreover,

\item $X \ci_{cl(G)} W_2 | W_1 \cup pa_G(K_m)$ as shown in the third paragraph of the proof of Lemma \ref{lem:same}. Then,

\item $X \ci_{cl(G)} Y \cup W_2 | W_1 \cup pa_G(K_m)$ by composition on (3) and (4). Moreover, for any $C \in X \cup W_1$

\item $C \ci_{cl(G)} pa_G(K_m) \setminus pa_G(C) | pa_G(C)$ follows from the pairwise separation base of $G$ by repeated composition. Then,

\item $C \ci_{cl(G)} pa_G(K_m) \setminus pa_G(X \cup W_1) | pa_G(X \cup W_1)$ by weak union. Then,

\item $X \ci_{cl(G)} pa_G(K_m) \setminus pa_G(X \cup W_1) | W_1 \cup pa_G(X \cup W_1)$ by repeated symmetry, composition and weak union. Then,

\item $X \ci_{cl(G)} Y \cup W_2 | W_1 \cup pa_G(X \cup W_1)$ by contraction on (5) and (8), and decomposition. Moreover, for any $C \in X \cup W_1$

\item $C \ci_{cl(G)} Z \setminus W \cup pa_G(X \cup W_1) \setminus pa_G(C) | pa_G(C)$ follows from the pairwise separation base of $G$ by repeated composition. Then,

\item $C \ci_{cl(G)} Z \setminus W \setminus pa_G(X \cup W_1) | pa_G(X \cup W_1)$ by weak union. Then,

\item $X \ci_{cl(G)} Z \setminus W \setminus pa_G(X \cup W_1) | W_1 \cup pa_G(X \cup W_1)$ by repeated symmetry, composition and weak union. Then,

\item $X \ci_{cl(G)} Y | Z \cup pa_G(X \cup W_1)$ by composition on (9) and (12), and weak union.

\end{enumerate}

\end{proof}

\begin{proof}[{\bf Proof of Theorem \ref{the:pairwise2}}]

Since the independence model induced by $G$ satisfies the decomposition property and $cl(G)$ satisfies the composition property, it suffices to prove the theorem for $|X|=|Y|=1$. Moreover, assume without loss of generality that $Y \in K_1 \cup \ldots \cup K_m$ and $X \in K_m$. We prove the theorem by induction on $|Z|$. The theorem holds for $|Z|=0$ and $m=1$ by Lemma \ref{lem:same}, because $X, Y \in K_1$, $X \ci_G Y | Z$, $Z \cap (K_2 \cup \ldots \cup K_n) = \emptyset$ and $pa_G(X \cup W_1) \setminus Z=\emptyset$. Assume as induction hypothesis that the theorem holds for $|Z|=0$ and $m<l$. We now prove it for $|Z|=0$ and $m=l$. Consider the following cases.

\begin{description}

\item[Case 1] Assume that $Y \in K_1 \cup \ldots \cup K_{l-1}$. Then,

\begin{enumerate}

\item $X \ci_{cl(G)} Y | Z \cup pa_G(X \cup W_1)$ by Lemma \ref{lem:different}, because $Y \in K_1 \cup \ldots \cup K_{l-1}$, $X \in K_l$, $X \ci_G Y | Z$ and $Z \cap (K_{l+1} \cup \ldots \cup K_n) = \emptyset$. Moreover, for any $C \in pa_G(X \cup W_1) \setminus Z$

\item $X \nci_G C | Z$ by Lemma \ref{lem:aux3}, because $Y \in K_1 \cup \ldots \cup K_{l-1}$, $X \in K_l$ and $X \ci_G Y | Z$. Then,

\item $C \ci_G Y | Z$ because, otherwise, $X \nci_G Y | Z$ which is a contradiction. Then,

\item $C \ci_{cl(G)} Y | Z$ by the induction hypothesis, because $C, Y \in K_1 \cup \ldots \cup K_{l-1}$. Then,

\item $pa_G(X \cup W_1) \setminus Z \ci_{cl(G)} Y | Z$ by repeated symmetry and composition. Then, 

\item $X \ci_{cl(G)} Y | Z$ by symmetry, contraction on (1) and (5), and decomposition.

\end{enumerate}

\item[Case 2] Assume that $Y \in K_l$. Then,

\begin{enumerate}

\item $X \ci_{cl(G)} Y | Z \cup pa_G(X \cup W_1)$ by Lemma \ref{lem:same}, because $X, Y \in K_l$, $X \ci_G Y | Z$ and $Z \cap (K_{l+1} \cup \ldots \cup K_n) = \emptyset$. Moreover, for any $D \in pa_G(X \cup W_1) \setminus Z$

\item $X \nci_G D | Z$ by Lemma \ref{lem:aux3}, because $X, Y \in K_l$ and $X \ci_G Y | Z$. Then,

\item $Y \ci_G D | Z$ because, otherwise, $X \nci_G Y | Z$ which is a contradiction. Then,

\item $Y \ci_{cl(G)} D | Z$ by Case 1 replacing $X$ with $Y$ and $Y$ with $D$, because $D \in K_1 \cup \ldots \cup K_{l-1}$, $Y \in K_l$ and (3). Then,

\item $Y \ci_{cl(G)} pa_G(X \cup W_1) \setminus Z | Z$ by repeated composition. Then,

\item $X \ci_{cl(G)} Y | Z$ by symmetry, contraction on (1) and (5), and decomposition.

\end{enumerate}

\end{description}

This ends the proof for $|Z|=0$. Assume as induction hypothesis that the theorem holds for $|Z|<t$. We now prove it for $|Z|=t$ and $m=1$. Let $K_j$ be the connectivity component st $Z \cap K_j \neq \emptyset$ and $Z \cap (K_{j+1} \cup \ldots \cup K_n) = \emptyset$. Consider the following cases.

\begin{description}

\item[Case 3] Assume that $j=1$. Then, $X \ci_{cl(G)} Y | Z$ holds by Lemma \ref{lem:same}, because $X, Y \in K_1$, $X \ci_G Y | Z$, $Z \cap (K_2 \cup \ldots \cup K_n) = \emptyset$ and $pa_G(X \cup W_1) \setminus Z=\emptyset$.

\item[Case 4] Assume that $j>1$ and $pa_G(Z \cap K_j) \setminus Z = \emptyset$. Then, note that there is no $(Z \setminus C)$-open path between $X$ and any $C \in Z \cap K_j$. To see it, assume the contrary. Since $X \in K_1$ and $j>1$, the path must reach $K_j$ from one of its parents or children. However, the path cannot reach $K_j$ from one of its children because, otherwise, the path has a triplex node outside $Z$ since $X \in K_1$, $j>1$ and $Z \cap (K_{j+1} \cup \ldots \cup K_n) = \emptyset$. This contradicts that the path is $(Z \setminus C)$-open. Then, the path must reach $K_j$ from one of its parents. However, this contradicts that the path is $(Z \setminus C)$-open, because $pa_G(Z \cap K_j) \setminus Z=\emptyset$. Then,

\begin{enumerate}

\item $X \ci_G C | Z \setminus C$ as shown above. Then,

\item $X \ci_{cl(G)} C | Z \setminus C$ by the induction hypothesis. Moreover,

\item $X \ci_G Y | Z \setminus C$ by contraction on $X \ci_G Y | Z$ and (1), and decomposition. Then,

\item $X \ci_{cl(G)} Y | Z \setminus C$ by the induction hypothesis. Then, 

\item $X \ci_{cl(G)} Y | Z$ by composition on (2) and (4), and weak union.

\end{enumerate} 

\item[Case 5] Assume that $j>1$ and $pa_G(C) \setminus Z \neq \emptyset$ for some $C \in Z \cap K_j$. Then, note that there is no $(Z \setminus C)$-open path between $X$ and $Y$. To see it, assume the contrary. If $C$ is not in the path, then $C \in pa_G(D)$ st $-D-$ is in the path and $D \in Z$ because, otherwise, the path is $Z$-open which contradicts that $X \ci_G Y | Z$. However, this implies a contradiction because $C \in K_j$ and thus $D \in K_{j+1} \cup \ldots \cup K_n$, but $Z \cap (K_{j+1} \cup \ldots \cup K_n) = \emptyset$. Therefore, $C$ must be in the path. In fact, $C$ must be a non-triplex node in the path because, otherwise, the path is not $(Z \setminus C)$-open. Then, either (i) $-C-$, (ii) $\la C \oo$ or (iii) $\oo C \ra$ is in the path. Case (i) implies that the path is $Z$-open, because $pa_G(C) \setminus Z \neq \emptyset$. This contradicts that $X \ci_G Y | Z$. Cases (ii) and (iii) imply that the path has a directed subpath from $C$ to (iv) $X$, (v) $Y$ or (vi) a triplex node $E$ in the path. Cases (iv) and (v) are impossible because $X, Y \in K_1$ but $C \in K_j$ with $j > 1$. Case (vi) contradicts that the path is $(Z \setminus C)$-open, because $C \in K_j$ and thus $E \in K_{j+1} \cup \ldots \cup K_n$, but $Z \cap (K_{j+1} \cup \ldots \cup K_n) = \emptyset$. Then,

\begin{enumerate}

\item $X \ci_G Y | Z \setminus C$ as shown above. Then,

\item $X \ci_{cl(G)} Y | Z \setminus C$ by the induction hypothesis. Moreover,

\item $X \ci_G C | Z \setminus C$ or $C \ci_G Y | Z \setminus C$ by weak transitivity on $X \ci_G Y | Z$ and (1). Then,

\item $X \ci_{cl(G)} C | Z \setminus C$ or $C \ci_{cl(G)} Y | Z \setminus C$ by the induction hypothesis. Then, 

\item $X \ci_{cl(G)} Y | Z$ by symmetry, composition on (2) and (4), and weak union.

\end{enumerate} 

\end{description}

This ends the proof for $|Z|=t$ and $m=1$. Assume as induction hypothesis that the theorem holds for $|Z|=t$ and $m<l$. In order to prove it for $|Z|=t$ and $m=l$, it suffices to repeat Cases 1 and 2 if $Z \cap (K_{l+1} \cup \ldots \cup K_n) = \emptyset$, and Cases 4 and 5 replacing 1 with $l$ otherwise.

\end{proof}

\begin{proof}[{\bf Proof of Theorem \ref{the:triplex}}]

We first prove the ``only if" part. Let $G_1$ and $G_2$ be two Markov equivalent MAMP CGs. First, assume that two nodes $A$ and $C$ are adjacent in $G_2$ but not in $G_1$. If $A$ and $C$ are in the same undirected connectivity component of $G_1$, then $A \ci C | ne_{G_1}(A) \cup pa_{G_1}(A \cup ne_{G_1}(A))$ holds for $G_1$ by Theorem \ref{the:pairwise1} but it does not hold for $G_2$, which is a contradiction. On the other hand, if $A$ and $C$ are in different undirected connectivity components of $G_1$, then $A \ci C | pa_{G_1}(C)$ or $A \ci C | pa_{G_1}(A)$ holds for $G_1$ by Theorem \ref{the:pairwise1} but neither holds for $G_2$, which is a contradiction. Consequently, $G_1$ and $G_2$ must have the same adjacencies.

Finally, assume that $G_1$ and $G_2$ have the same adjacencies but $G_1$ has a triplex $(\{A,C\},B)$ that $G_2$ does not have. If $A$ and $C$ are in the same undirected connectivity component of $G_1$, then $A \ci C | ne_{G_1}(A) \cup pa_{G_1}(A \cup ne_{G_1}(A))$ holds for $G_1$ by Theorem \ref{the:pairwise1}. Note also that $B \notin ne_{G_1}(A) \cup pa_{G_1}(A \cup ne_{G_1}(A))$ because, otherwise, $G_1$ would not satisfy the constraint C1 or C2. Then, $A \ci C | ne_{G_1}(A) \cup pa_{G_1}(A \cup ne_{G_1}(A))$ does not hold for $G_2$, which is a contradiction. On the other hand, if $A$ and $C$ are in different undirected connectivity components of $G_1$, then $A \ci C | pa_{G_1}(C)$ or $A \ci C | pa_{G_1}(A)$ holds for $G_1$ by Theorem \ref{the:pairwise1}. Note also that $B \notin pa_{G_1}(A)$ and $B \notin pa_{G_1}(C)$ because, otherwise, $G_1$ would not have the triplex $(\{A,C\},B)$. Then, neither $A \ci C | pa_{G_1}(C)$ nor $A \ci C | pa_{G_1}(A)$ holds for $G_2$, which is a contradiction. Consequently, $G_1$ and $G_2$ must be triplex equivalent.

We now prove the ``if" part. Let $G_1$ and $G_2$ be two triplex equivalent MAMP CGs. We just prove that all the non-separations in $G_1$ are also in $G_2$. The opposite result can be proven in the same manner by just exchanging the roles of $G_1$ and $G_2$ in the proof. Specifically, assume that $\alpha \ci \beta | Z$ does not hold for $G_1$. We prove that $\alpha \ci \beta | Z$ does not hold for $G_2$ either. We divide the proof in three parts.

{\bf Part 1}

We say that a path has a triplex $(\{A,C\},B)$ if it has a subpath of the form $A \oa B \ao C$, $A \oa B - C$, or $A - B \ao C$. Let $\rho_1$ be any path between $\alpha$ and $\beta$ in $G_1$ that is $Z$-open st (i) no subpath of $\rho_1$ between $\alpha$ and $\beta$ in $G_1$ is $Z$-open, (ii) every triplex node in $\rho_1$ is in $Z$, and (iii) $\rho_1$ has no non-triplex node in $Z$. Let $\rho_2$ be the path in $G_2$ that consists of the same nodes as $\rho_1$. Then, $\rho_2$ is $Z$-open. To see it, assume the contrary. Then, one of the following cases must occur.

\begin{description}
\item[Case 1] $\rho_2$ does not have a triplex $(\{A,C\},B)$ and $B \in Z$. Then, $\rho_1$ must have a triplex $(\{A,C\},B)$ because it is $Z$-open. Then, $A$ and $C$ must be adjacent in $G_1$ and $G_2$ because these are triplex equivalent. Let $\varrho_1$ be the path obtained from $\rho_1$ by replacing the triplex $(\{A,C\},B)$ with the edge between $A$ and $C$ in $G_1$. Note that $\varrho_1$ cannot be $Z$-open because, otherwise, it would contradict the condition (i). Then, $\varrho_1$ is not $Z$-open because $A$ or $C$ do not meet the requirements. Assume without loss of generality that $C$ does not meet the requirements. Then, one of the following cases must occur.

\begin{description}
\item[Case 1.1] $\varrho_1$ does not have a triplex $(\{A,D\},C)$ and $C \in Z$. Then, one of the following subgraphs must occur in $G_1$.\footnote{If $\varrho_1$ does not have a triplex $(\{A,D\},C)$, then $A \la C$, $C \ra D$ or $A - C - D$ must be in $G_1$. Moreover, recall that $B$ is a triplex node in $\rho_1$. Then, $A \ra B \la C$, $A \ra B \aa C$, $A \ra B - C$, $A \aa B \la C$, $A \aa B \aa C$, $A \aa B - C$, $A - B \la C$ or $A - B \aa C$ must be in $G_1$. However, if $A \la C$ is in $G_1$ then the only legal options are those that contain the edge $B \la C$. On the other hand, if $A - C - D$ is in $G_1$ then the only legal options are $A \ra B \la C$ and $A \aa B \aa C$.}

\begin{table}[H]
\centering
\scalebox{1.0}{
\begin{tabular}{cccc}
\begin{tikzpicture}[inner sep=1mm]
\node at (0,0) (A) {$A$};
\node at (1,0) (B) {$B$};
\node at (2,0) (C) {$C$};
\node at (3,0) (D) {$D$};
\path[o-o] (A) edge (B);
\path[o-o] (B) edge (C);
\path[->] (C) edge (D);
\path[o-o] (A) edge [bend left] (C);
\end{tikzpicture}
&
\begin{tikzpicture}[inner sep=1mm]
\node at (0,0) (A) {$A$};
\node at (1,0) (B) {$B$};
\node at (2,0) (C) {$C$};
\node at (3,0) (D) {$D$};
\path[o-o] (A) edge (B);
\path[<-] (B) edge (C);
\path[o-o] (C) edge (D);
\path[<-] (A) edge [bend left] (C);
\end{tikzpicture}
&
\begin{tikzpicture}[inner sep=1mm]
\node at (0,0) (A) {$A$};
\node at (1,0) (B) {$B$};
\node at (2,0) (C) {$C$};
\node at (3,0) (D) {$D$};
\path[->] (A) edge (B);
\path[<-] (B) edge (C);
\path[-] (C) edge (D);
\path[-] (A) edge [bend left] (C);
\end{tikzpicture}
&
\begin{tikzpicture}[inner sep=1mm]
\node at (0,0) (A) {$A$};
\node at (1,0) (B) {$B$};
\node at (2,0) (C) {$C$};
\node at (3,0) (D) {$D$};
\path[<->] (A) edge (B);
\path[<->] (B) edge (C);
\path[-] (C) edge (D);
\path[-] (A) edge [bend left] (C);
\end{tikzpicture}
\end{tabular}}
\end{table}

However, the first three subgraphs imply that $\rho_1$ is not $Z$-open, which is a contradiction. The fourth subgraph implies that $\varrho_1$ is $Z$-open, which is a contradiction.

\item[Case 1.2] $\varrho_1$ has a triplex $(\{A,D\},C)$ and $C \notin Z \cup san_{G_1}(Z)$. Note that $C$ cannot be a triplex node in $\rho_1$ because, otherwise, $\rho_1$ would not be $Z$-open. Then, one of the following subgraphs must occur in $G_1$.

\begin{table}[H]
\centering
\scalebox{1.0}{
\begin{tabular}{cccc}
\begin{tikzpicture}[inner sep=1mm]
\node at (0,0) (A) {$A$};
\node at (1,0) (B) {$B$};
\node at (2,0) (C) {$C$};
\node at (3,0) (D) {$D$};
\path[o-o] (A) edge (B);
\path[<-] (B) edge (C);
\path[<-o] (C) edge (D);
\path[o-o] (A) edge [bend left] (C);
\end{tikzpicture}
&
\begin{tikzpicture}[inner sep=1mm]
\node at (0,0) (A) {$A$};
\node at (1,0) (B) {$B$};
\node at (2,0) (C) {$C$};
\node at (3,0) (D) {$D$};
\path[o-o] (A) edge (B);
\path[<-] (B) edge (C);
\path[-] (C) edge (D);
\path[o->] (A) edge [bend left] (C);
\end{tikzpicture}
&
\begin{tikzpicture}[inner sep=1mm]
\node at (0,0) (A) {$A$};
\node at (1,0) (B) {$B$};
\node at (2,0) (C) {$C$};
\node at (3,0) (D) {$D$};
\path[<->] (A) edge (B);
\path[-] (B) edge (C);
\path[-] (C) edge (D);
\path[<->] (A) edge [bend left] (C);
\end{tikzpicture}
&
\begin{tikzpicture}[inner sep=1mm]
\node at (0,0) (A) {$A$};
\node at (1,0) (B) {$B$};
\node at (2,0) (C) {$C$};
\node at (3,0) (D) {$D$};
\path[->] (A) edge (B);
\path[-] (B) edge (C);
\path[-] (C) edge (D);
\path[->] (A) edge [bend left] (C);
\end{tikzpicture}
\end{tabular}}
\end{table}

However, the first and second subgraphs imply that $C \in Z \cup san_{G_1}(Z)$ because $B \in Z$, which is a contradiction. The third subgraph implies that $B - D$ is in $G_1$ by the constraint C3 and, thus, that the path obtained from $\rho_1$ by replacing $B - C - D$ with $B - D$ is $Z$-open, which contradicts the condition (i). For the fourth subgraph, assume that $A$ and $D$ are adjacent in $G_1$. Then, one of the following subgraphs must occur in $G_1$.

\begin{table}[H]
\centering
\scalebox{1.0}{
\begin{tabular}{ccc}
\begin{tikzpicture}[inner sep=1mm]
\node at (0,0) (A) {$A$};
\node at (1,0) (B) {$B$};
\node at (2,0) (C) {$C$};
\node at (3,0) (D) {$D$};
\node at (4,0) (E) {$E$};
\path[->] (A) edge (B);
\path[-] (B) edge (C);
\path[-] (C) edge (D);
\path[->] (A) edge [bend left] (C);
\path[->] (D) edge (E);
\path[->] (A) edge [bend left] (D);
\end{tikzpicture}
&
\begin{tikzpicture}[inner sep=1mm]
\node at (0,0) (A) {$A$};
\node at (1,0) (B) {$B$};
\node at (2,0) (C) {$C$};
\node at (3,0) (D) {$D$};
\node at (4,0) (E) {$E$};
\path[->] (A) edge (B);
\path[-] (B) edge (C);
\path[-] (C) edge (D);
\path[->] (A) edge [bend left] (C);
\path[<-o] (D) edge (E);
\path[->] (A) edge [bend left] (D);
\end{tikzpicture}
&
\begin{tikzpicture}[inner sep=1mm]
\node at (0,0) (A) {$A$};
\node at (1,0) (B) {$B$};
\node at (2,0) (C) {$C$};
\node at (3,0) (D) {$D$};
\node at (4,0) (E) {$E$};
\path[->] (A) edge (B);
\path[-] (B) edge (C);
\path[-] (C) edge (D);
\path[->] (A) edge [bend left] (C);
\path[-] (D) edge (E);
\path[->] (A) edge [bend left] (D);
\end{tikzpicture}
\end{tabular}}
\end{table}

However, the first subgraph implies that the path obtained from $\rho_1$ by replacing $A \ra B - C - D$ with $A \ra D$ is $Z$-open, because $D \notin Z$ since $\rho_1$ is $Z$-open. This contradicts the condition (i). The second subgraph implies that the path obtained from $\rho_1$ by replacing $A \ra B - C - D$ with $A \ra D$ is $Z$-open, because $D \in Z \cup san_{G_1}(Z)$ since $\rho_1$ is $Z$-open. This contradicts the condition (i). Therefore, only the third subgraph is possible. Thus, by repeatedly applying the previous reasoning, we can conclude without loss of generality that the following subgraph must occur in $G_1$, with $n \geq 4$, $V_1=A$, $V_2=B$, $V_3=C$, $V_4=D$ and where $V_1$ and $V_n$ are not adjacent in $G_1$. Note that the subgraph below covers the case where $A$ and $D$ are not adjacent in the original subgraph by simply taking $n=4$. 

\begin{table}[H]
\centering
\scalebox{1.0}{
\begin{tabular}{c}
\begin{tikzpicture}[inner sep=1mm]
\node at (0,0) (A) {$V_1$};
\node at (1,0) (B) {$V_2$};
\node at (2,0) (C) {$V_3$};
\node at (3,0) (D) {$V_4$};
\node at (4,0) (D2) {$\ldots$};
\node at (5,0) (E) {$V_{n-1}$};
\node at (6,0) (F) {$V_n$};
\path[->] (A) edge (B);
\path[-] (B) edge (C);
\path[-] (C) edge (D);
\path[->] (A) edge [bend left] (C);
\path[-] (D) edge (D2);
\path[-] (D2) edge (E);
\path[->] (A) edge [bend left] (D);
\path[-] (E) edge (F);
\path[->] (A) edge [bend left] (E);
\end{tikzpicture}
\end{tabular}}
\end{table}

Since $V_1$ and $V_n$ are not adjacent in $G_1$, $G_1$ has a triplex $(\{V_1, V_n\}, V_{n-1})$ and, thus, so does $G_2$ because $G_1$ and $G_2$ are triplex equivalent. Then, one of the following subgraphs must occur in $G_2$.

\begin{table}[H]
\centering
\scalebox{1.0}{
\begin{tabular}{ccc}
\begin{tikzpicture}[inner sep=1mm]
\node at (0,0) (A) {$V_1$};
\node at (1,0) (B) {$\ldots$};
\node at (2,0) (C) {$V_{n-1}$};
\node at (3.2,0) (D) {$V_n$};
\path[<-o] (C) edge (D);
\path[-] (A) edge [bend left] (C);
\end{tikzpicture}
&
\begin{tikzpicture}[inner sep=1mm]
\node at (0,0) (A) {$V_1$};
\node at (1,0) (B) {$\ldots$};
\node at (2,0) (C) {$V_{n-1}$};
\node at (3.2,0) (D) {$V_n$};
\path[<-o] (C) edge (D);
\path[o->] (A) edge [bend left] (C);
\end{tikzpicture}
&
\begin{tikzpicture}[inner sep=1mm]
\node at (0,0) (A) {$V_1$};
\node at (1,0) (B) {$\ldots$};
\node at (2,0) (C) {$V_{n-1}$};
\node at (3,0) (D) {$V_n$};
\path[-] (C) edge (D);
\path[o->] (A) edge [bend left] (C);
\end{tikzpicture}
\end{tabular}}
\end{table}

Note that $V_1, \ldots, V_n$ must be a path in $G_2$, because $G_1$ and $G_2$ are triplex equivalent. Note also that this path cannot have any triplex in $G_2$. To see it, recall that we assumed that $\rho_2$ does not have a triplex $(\{A,C\},B)$. Recall that $V_1=A$, $V_2=B$, $V_3=C$. Moreover, if the path $V_1, \ldots, V_n$ has a triplex $(\{V_i,V_{i+2}\},V_{i+1})$ in $G_2$ with $2 \leq i \leq n-2$, then $V_i$ and $V_{i+2}$ must be adjacent in $G_1$ and $G_2$, because such a triplex does not exist in $G_1$, which is triplex equivalent to $G_2$. Specifically, $V_i - V_{i+2}$ must be in $G_1$ because, as seen above, $V_i - V_{i+1} - V_{i+2}$ is in $G_1$. Then, the path obtained from $\rho_1$ by replacing $V_i - V_{i+1} - V_{i+2}$ with $V_i - V_{i+2}$ is $Z$-open, which contradicts the condition (i). However, if the path $V_1, \ldots, V_n$ has no triplex in $G_2$, then every edge in the path must be directed as $\la$ in the case of the first and second subgraphs above, whereas every edge in the path must be undirected or directed as $\la$ in the third subgraph above. Either case contradicts the constraint C1 or C2.

\end{description}

\item[Case 2] Case 1 does not apply. Then, $\rho_2$ has a triplex $(\{A,C\},B)$ and $B \notin Z \cup san_{G_2}(Z)$. Then, $\rho_1$ cannot have a triplex $(\{A,C\},B)$. Then, $A$ and $C$ must be adjacent in $G_1$ and $G_2$ because these are triplex equivalent. Let $\varrho_1$ be the path obtained from $\rho_1$ by replacing the triplex $(\{A,C\},B)$ with the edge between $A$ and $C$ in $G_1$. Note that $\varrho_1$ cannot be $Z$-open because, otherwise, it would contradict the condition (i). Then, $\varrho_1$ is not $Z$-open because $A$ or $C$ do not meet the requirements. Assume without loss of generality that $C$ does not meet the requirements. Then, one of the following cases must occur.

\begin{description}
\item[Case 2.1] $\varrho_1$ has a triplex $(\{A,D\},C)$ and $C \notin Z \cup san_{G_1}(Z)$. Then, one of the following subgraphs must occur in $G_1$.\footnote{If $\varrho_1$ has a triplex $(\{A,D\},C)$, then $A \ra C \ao D$, $A \ra C - D$, $A \aa C \ao D$, $A \aa C - D$ or $A - C \ao D$ must be in $G_1$. Moreover, recall that $B$ is not a triplex node in $\rho_1$. Then, $A \la B \la C$, $A \la B \ra C$, $A \la B \aa C$, $A \la B - C$, $A \ra B \ra C$, $A \aa B \ra C$, $A - B \ra C$ or $A - B - C$ must be in $G_1$. However, if $A \ra C$ is in $G_1$ then the only legal options are those that contain the edge $B \ra C$. On the other hand, if $A \aa C$ is in $G_1$ then the only legal option is $A \la B \ra C$. Finally, if $A - C$ is in $G_1$ then the only legal options are $A \la B \ra C$ and $A - B - C$.}

\begin{table}[H]
\centering
\scalebox{1.0}{
\begin{tabular}{ccc}
\begin{tikzpicture}[inner sep=1mm]
\node at (0,0) (A) {$A$};
\node at (1,0) (B) {$B$};
\node at (2,0) (C) {$C$};
\node at (3,0) (D) {$D$};
\path[o-o] (A) edge (B);
\path[->] (B) edge (C);
\path[<-o] (C) edge (D);
\path[->] (A) edge [bend left] (C);
\end{tikzpicture}
&
\begin{tikzpicture}[inner sep=1mm]
\node at (0,0) (A) {$A$};
\node at (1,0) (B) {$B$};
\node at (2,0) (C) {$C$};
\node at (3,0) (D) {$D$};
\path[o-o] (A) edge (B);
\path[->] (B) edge (C);
\path[-] (C) edge (D);
\path[->] (A) edge [bend left] (C);
\end{tikzpicture}
&
\begin{tikzpicture}[inner sep=1mm]
\node at (0,0) (A) {$A$};
\node at (1,0) (B) {$B$};
\node at (2,0) (C) {$C$};
\node at (3,0) (D) {$D$};
\path[<-] (A) edge (B);
\path[->] (B) edge (C);
\path[<-o] (C) edge (D);
\path[<->] (A) edge [bend left] (C);
\end{tikzpicture}
\\
\begin{tikzpicture}[inner sep=1mm]
\node at (0,0) (A) {$A$};
\node at (1,0) (B) {$B$};
\node at (2,0) (C) {$C$};
\node at (3,0) (D) {$D$};
\path[<-] (A) edge (B);
\path[->] (B) edge (C);
\path[-] (C) edge (D);
\path[<->] (A) edge [bend left] (C);
\end{tikzpicture}
&
\begin{tikzpicture}[inner sep=1mm]
\node at (0,0) (A) {$A$};
\node at (1,0) (B) {$B$};
\node at (2,0) (C) {$C$};
\node at (3,0) (D) {$D$};
\path[<-] (A) edge (B);
\path[->] (B) edge (C);
\path[<-o] (C) edge (D);
\path[-] (A) edge [bend left] (C);
\end{tikzpicture}
&
\begin{tikzpicture}[inner sep=1mm]
\node at (0,0) (A) {$A$};
\node at (1,0) (B) {$B$};
\node at (2,0) (C) {$C$};
\node at (3,0) (D) {$D$};
\path[-] (A) edge (B);
\path[-] (B) edge (C);
\path[<-o] (C) edge (D);
\path[-] (A) edge [bend left] (C);
\end{tikzpicture}
\end{tabular}}
\end{table}

However, this implies that $C$ is a triplex node in $\rho_1$, which is a contradiction because $\rho_1$ is $Z$-open but $C \notin Z \cup san_{G_1}(Z)$.

\item[Case 2.2] $\varrho_1$ does not have a triplex $(\{A,D\},C)$ and $C \in Z$. Then, $A \la C$, $C \ra D$ or $A - C - D$.

\begin{description}
\item[Case 2.2.1] If $C \ra D$ or $A - C - D$, then one of the following subgraphs must occur in $G_1$.

\begin{table}[H]
\centering
\scalebox{1.0}{
\begin{tabular}{ccc}
\begin{tikzpicture}[inner sep=1mm]
\node at (0,0) (A) {$A$};
\node at (1,0) (B) {$B$};
\node at (2,0) (C) {$C$};
\node at (3,0) (D) {$D$};
\path[o-o] (A) edge (B);
\path[o-o] (B) edge (C);
\path[->] (C) edge (D);
\path[o-o] (A) edge [bend left] (C);
\end{tikzpicture}
&
\begin{tikzpicture}[inner sep=1mm]
\node at (0,0) (A) {$A$};
\node at (1,0) (B) {$B$};
\node at (2,0) (C) {$C$};
\node at (3,0) (D) {$D$};
\path[-] (A) edge (B);
\path[-] (B) edge (C);
\path[-] (C) edge (D);
\path[-] (A) edge [bend left] (C);
\end{tikzpicture}
&
\begin{tikzpicture}[inner sep=1mm]
\node at (0,0) (A) {$A$};
\node at (1,0) (B) {$B$};
\node at (2,0) (C) {$C$};
\node at (3,0) (D) {$D$};
\path[<-] (A) edge (B);
\path[->] (B) edge (C);
\path[-] (C) edge (D);
\path[-] (A) edge [bend left] (C);
\end{tikzpicture}
\end{tabular}}
\end{table}

However, the first and second subgraphs imply that $\rho_1$ is not $Z$-open, which is a contradiction. The third subgraph implies that $\varrho_1$ is $Z$-open, which is a contradiction.

\item[Case 2.2.2] If $A \la C$ then $(\{A,D\},C)$ is not a triplex in $\varrho_1$. However, note that $\rho_1$ must have a triplex $(\{B,D\},C)$, because $\rho_1$ is $Z$-open and $C \in Z$. Then, one of the following subgraphs must occur in $G_1$.

\begin{table}[H]
\centering
\scalebox{1.0}{
\begin{tabular}{ccc}
\begin{tikzpicture}[inner sep=1mm]
\node at (0,0) (A) {$A$};
\node at (1,0) (B) {$B$};
\node at (2,0) (C) {$C$};
\node at (3,0) (D) {$D$};
\path[<-] (A) edge (B);
\path[o->] (B) edge (C);
\path[<-o] (C) edge (D);
\path[<-] (A) edge [bend left] (C);
\end{tikzpicture}
&
\begin{tikzpicture}[inner sep=1mm]
\node at (0,0) (A) {$A$};
\node at (1,0) (B) {$B$};
\node at (2,0) (C) {$C$};
\node at (3,0) (D) {$D$};
\path[<-] (A) edge (B);
\path[o->] (B) edge (C);
\path[-] (C) edge (D);
\path[<-] (A) edge [bend left] (C);
\end{tikzpicture}
&
\begin{tikzpicture}[inner sep=1mm]
\node at (0,0) (A) {$A$};
\node at (1,0) (B) {$B$};
\node at (2,0) (C) {$C$};
\node at (3,0) (D) {$D$};
\path[<-] (A) edge (B);
\path[-] (B) edge (C);
\path[<-o] (C) edge (D);
\path[<-] (A) edge [bend left] (C);
\end{tikzpicture}
\end{tabular}}
\end{table}

Assume that $A$ and $D$ are adjacent in $G_1$. Then, $A \la D$ must be in $G_1$. Moreover, $D \in Z$ because, otherwise, we can remove $B$ and $C$ from $\rho_1$ and get a $Z$-open path between $A$ and $B$ in $G_1$ that is shorter than $\rho_1$, which contradicts the condition (i). Then, $D$ must be a triplex node in $\rho_1$. Then, one of the following subgraphs must occur in $G_1$.

\begin{table}[H]
\centering
\scalebox{1.0}{
\begin{tabular}{ccc}
\begin{tikzpicture}[inner sep=1mm]
\node at (0,0) (A) {$A$};
\node at (1,0) (B) {$B$};
\node at (2,0) (C) {$C$};
\node at (3,0) (D) {$D$};
\node at (4,0) (E) {$E$};
\path[<-] (A) edge (B);
\path[o->] (B) edge (C);
\path[<->] (C) edge (D);
\path[<-o] (D) edge (E);
\path[<-] (A) edge [bend left] (C);
\path[<-] (A) edge [bend left] (D);
\end{tikzpicture}
&
\begin{tikzpicture}[inner sep=1mm]
\node at (0,0) (A) {$A$};
\node at (1,0) (B) {$B$};
\node at (2,0) (C) {$C$};
\node at (3,0) (D) {$D$};
\node at (4,0) (E) {$E$};
\path[<-] (A) edge (B);
\path[o->] (B) edge (C);
\path[<->] (C) edge (D);
\path[-] (D) edge (E);
\path[<-] (A) edge [bend left] (C);
\path[<-] (A) edge [bend left] (D);
\end{tikzpicture}
&
\begin{tikzpicture}[inner sep=1mm]
\node at (0,0) (A) {$A$};
\node at (1,0) (B) {$B$};
\node at (2,0) (C) {$C$};
\node at (3,0) (D) {$D$};
\node at (4,0) (E) {$E$};
\path[<-] (A) edge (B);
\path[o->] (B) edge (C);
\path[-] (C) edge (D);
\path[<-o] (D) edge (E);
\path[<-] (A) edge [bend left] (C);
\path[<-] (A) edge [bend left] (D);
\end{tikzpicture}
\\
\begin{tikzpicture}[inner sep=1mm]
\node at (0,0) (A) {$A$};
\node at (1,0) (B) {$B$};
\node at (2,0) (C) {$C$};
\node at (3,0) (D) {$D$};
\node at (4,0) (E) {$E$};
\path[<-] (A) edge (B);
\path[-] (B) edge (C);
\path[<->] (C) edge (D);
\path[<-o] (D) edge (E);
\path[<-] (A) edge [bend left] (C);
\path[<-] (A) edge [bend left] (D);
\end{tikzpicture}
&
\begin{tikzpicture}[inner sep=1mm]
\node at (0,0) (A) {$A$};
\node at (1,0) (B) {$B$};
\node at (2,0) (C) {$C$};
\node at (3,0) (D) {$D$};
\node at (4,0) (E) {$E$};
\path[<-] (A) edge (B);
\path[-] (B) edge (C);
\path[<->] (C) edge (D);
\path[-] (D) edge (E);
\path[<-] (A) edge [bend left] (C);
\path[<-] (A) edge [bend left] (D);
\end{tikzpicture}
\end{tabular}}
\end{table}

Thus, by repeatedly applying the previous reasoning, we can conclude without loss of generality that the following subgraph must occur in $G_1$, with $n \geq 4$, $V_1=A$, $V_2=B$, $V_3=C$, $V_4=D$ and where $V_1$ and $V_n$ are not adjacent in $G_1$. Note that the subgraph below covers the case where $A$ and $D$ are not adjacent in the original subgraph by simply taking $n=4$. 

\begin{table}[H]
\centering
\scalebox{1.0}{
\begin{tabular}{c}
\begin{tikzpicture}[inner sep=1mm]
\node at (0,0) (A) {$V_1$};
\node at (1,0) (B) {$V_2$};
\node at (2,0) (D2) {$\ldots$};
\node at (3,0) (E) {$V_{n-1}$};
\node at (4.3,0) (F) {$V_n$};
\path[<-] (A) edge (B);
\path[o-o] (E) edge (F);
\path[<-] (A) edge [bend left] (E);
\end{tikzpicture}
\end{tabular}}
\end{table}

Note that $V_i$ is a triplex node in $\rho_1$ for all $3 \leq i \leq n-1$. Then, $V_i \in Z$ for all $3 \leq i \leq n-1$ by the condition (ii) because $\rho_1$ is $Z$-open. Then, $V_i$ must be a triplex node in $\rho_2$ for all $3 \leq i \leq n-1$ because, otherwise, Case 1 would apply instead of Case 2. Recall that $V_2=B$ is also a triplex node in $\rho_2$. Note that $G_1$ does not have a triplex $(\{V_1, V_n\},V_{n-1})$ and, thus, $G_2$ does not have it either because these are triplex equivalent. Then, one of the following subgraphs must occur in $G_2$.

\begin{table}[H]
\centering
\scalebox{1.0}{
\begin{tabular}{ccc}
\begin{tikzpicture}[inner sep=1mm]
\node at (0,0) (A) {$V_1$};
\node at (1,0) (D2) {$\ldots$};
\node at (2,0) (E) {$V_{n-1}$};
\node at (3,0) (F) {$V_n$};
\path[->] (E) edge (F);
\path[o-o] (A) edge [bend left] (E);
\end{tikzpicture}
&
\begin{tikzpicture}[inner sep=1mm]
\node at (0,0) (A) {$V_1$};
\node at (1,0) (D2) {$\ldots$};
\node at (2,0) (E) {$V_{n-1}$};
\node at (3.3,0) (F) {$V_n$};
\path[o-o] (E) edge (F);
\path[<-] (A) edge [bend left] (E);
\end{tikzpicture}
&
\begin{tikzpicture}[inner sep=1mm]
\node at (0,0) (A) {$V_1$};
\node at (1,0) (D2) {$\ldots$};
\node at (2,0) (E) {$V_{n-1}$};
\node at (3,0) (F) {$V_n$};
\path[-] (E) edge (F);
\path[-] (A) edge [bend left] (E);
\end{tikzpicture}
\end{tabular}}
\end{table}

However, the first subgraph implies that $V_{n-1}$ is not a triplex node in $\rho_2$, which is a contradiction. The second subgraph implies that $G_2$ has a cycle that violates the constraint C1. To see it, recall that $V_i$ is a triplex node in $\rho_2$ for all $2 \leq i \leq n-1$ and, thus, $V_i \la V_{i+1}$ is not in $G_2$ for all $1 \leq i \leq n-2$. The third subgraph implies that $V_{n-2} \aa V_{n-1}$ is not in $G_2$ because, otherwise, $V_1$ and $V_n$ would be adjacent by the constraint C3. Therefore, $V_{n-2} \ra V_{n-1}$ must be in $G_2$ because $V_{n-1}$ is a triplex node in $\rho_2$. However, this implies that $V_{n-2}$ is not a triplex node in $\rho_2$, which is a contradiction.

\end{description}

\end{description}

\end{description}

{\bf Part 2}

Let $\rho_1$ be any of the shortest $Z$-open paths between $\alpha$ and $\beta$ in $G_1$ st all its triplex nodes are in $Z$. Let $\rho_2$ be the path in $G_2$ that consists of the same nodes as $\rho_1$. We prove below that $\rho_2$ is $Z$-open. We prove this result by induction on the number of non-triplex nodes of $\rho_1$ that are in $Z$. If this number is zero, then Part 1 proves the result. Assume as induction hypothesis that the result holds when the number is smaller than $m$. We now prove it for $m$.

Let $\rho_1^{A:B}$ denote the subpath of $\rho_1$ between the nodes $A$ and $B$. Let $C$ be any of the non-triplex nodes of $\rho_1$ that are in $Z$. Note that there must exist some node $D \in pa_{G_1}(C) \setminus Z$ for $\rho_1$ to be $Z$-open. If $D$ is in $\rho_1$, then $\rho_1^{\alpha:D} \cup D \ra C \cup \rho_1^{C:\beta}$ or $\rho_1^{\alpha:C} \cup C \la D \cup \rho_1^{D:\beta}$ is a $Z$-open path between $\alpha$ and $\beta$ in $G_1$ that has fewer than $m$ non-triplex nodes in $Z$. Then, the result holds by the induction hypothesis. On the other hand, if $D$ is not in $\rho_1$, then $\rho_1^{\alpha:C} \cup C \la D$ and $D \ra C \cup \rho_1^{C:\beta}$ are two paths. Moreover, they are $Z$-open in $G_1$ and they have fewer than $m$ non-triplex nodes in $Z$. Then, by the induction hypothesis, there are two $Z$-open paths $\rho_2^{\alpha:D}$ and $\rho_2^{D:\beta}$ in $G_2$ st the former ends with the nodes $C$ and $D$ and the latter starts with these two nodes. Now, consider the following cases.

\begin{description}

\item[Case 1] $\rho_2^{\alpha:D}$ ends with $A - C \la D$. Then, $\rho_2^{D:\beta}$ starts with $D \ra C - B$ or $D \ra C \ao B$. Then, $\rho_2 = \rho_2^{\alpha:C} \cup \rho_2^{C:\beta}$ is $Z$-open a path in either case.

\item[Case 2] $\rho_2^{\alpha:D}$ ends with $A - C \aa D$. Then, $\rho_2^{D:\beta}$ starts with $D \aa C - B$ or $D \aa C \ao B$. Then, $\rho_2 = \rho_2^{\alpha:C} \cup \rho_2^{C:\beta}$ is $Z$-open a path in either case.

\item[Case 3] $\rho_2^{\alpha:D}$ ends with $A \oa C - D$. Then, $\rho_2^{D:\beta}$ starts with $D - C \ao B$, or $D - C - B$ st there is some node $E \in pa_{G_2}(C) \setminus Z$. Then, $\rho_2 = \rho_2^{\alpha:C} \cup \rho_2^{C:\beta}$ is $Z$-open a path in either case.

\item[Case 4] $\rho_2^{\alpha:D}$ ends with $A \oa C \ao D$. Then, $\rho_2^{D:\beta}$ starts with $D \oa C - B$ or $D \oa C \ao B$. Then, $\rho_2 = \rho_2^{\alpha:C} \cup \rho_2^{C:\beta}$ is $Z$-open a path in either case.

\item[Case 5] $\rho_2^{\alpha:D}$ ends with $A - C - D$ st there is some node $E \in pa_{G_2}(C) \setminus Z$. Then, $\rho_2^{D:\beta}$ starts with $D - C \ao B$, or $D - C - B$ st there is some node $F \in pa_{G_2}(C) \setminus Z$. Then, $\rho_2 = \rho_2^{\alpha:C} \cup \rho_2^{C:\beta}$ is a $Z$-open path in either case.

\end{description}

{\bf Part 3}

Assume that Part 2 does not apply. Then, every $Z$-open path between $\alpha$ and $\beta$ in $G_1$ has some triplex node $B_1$ that is outside $Z$ because, otherwise, Part 2 would apply. Note that for the path to be $Z$-open, $G_1$ must have a subgraph $B_1 \ra \ldots \ra B_n$ st $B_1, \ldots, B_{n-1} \notin Z$ but $B_n \in Z$. Let us convert every $Z$-open path between $\alpha$ and $\beta$ in $G_1$ into a route by replacing each of its triplex nodes $B_1$ that are outside $Z$ with the corresponding route $B_1 \ra \ldots \ra B_n \la \ldots \la B_1$. Let $\varrho_1$ be any of the shortest routes so-constructed. Let $\rho_1$ be the path from which $\varrho_1$ was constructed. Note that $\rho_1$ cannot be $Z$-open st all its triplex nodes are in $Z$ because, otherwise, Part 2 would apply. Let $W$ denote the set of all the triplex nodes in $\rho_1$ that are outside $Z$. Then, $\rho_1$ is one of the shortest $(Z \cup W)$-open paths between $\alpha$ and $\beta$ in $G_1$ st all its triplex nodes are in $Z \cup W$. To see it, assume to the contrary that $\rho_1'$ is a $(Z \cup W)$-open path between $\alpha$ and $\beta$ in $G_1$ that is shorter than $\rho_1$ and st all the triplex nodes in $\rho_1'$ are in $Z \cup W$. Let $\varrho_1'$ be the route resulting from replacing every node $B_1$ of $\rho_1'$ that is in $W$ with the route $B_1 \ra \ldots \ra B_n \la \ldots \la B_1$ that was added to $\rho_1$ to construct $\varrho_1$. Clearly, $\varrho_1'$ is shorter than $\varrho_1$, which is a contradiction. Let $\varrho_2$ and $\rho_2$ be the route and the path in $G_2$ that consist of the same nodes as $\varrho_1$ and $\rho_1$. Note that $\rho_2$ is $(Z \cup W)$-open by Part 2.

Consider any of the routes $B_1 \ra \ldots \ra B_n \la \ldots \la B_1$ that were added to $\rho_1$ to construct $\varrho_1$. This implies that $\rho_1$ has a triplex $(\{A,C\},B_1)$. Assume that $B_1 \ra B_2$ is in $G_1$ but $B_1 - B_2$ or $B_1 \ao B_2$ is in $G_2$. Note that $A \oa B_1$ or $B_1 \ao C$ is in $G_2$ because, as noted above, $\rho_2$ is $(Z \cup W)$-open. Assume without loss of generality that $A \oa B_1$ is in $G_2$. Then, $A - B_1 \ra B_2$ or $A \oa B_1 \ra B_2$ is in $G_1$ whereas $A \oa B_1 - B_2$ or $A \oa B_1 \ao B_2$ is in $G_2$. Therefore, $A$ and $B_2$ must be adjacent in $G_1$ and $G_2$ because these are triplex equivalent. This implies that $A \ra B_2$ is in $G_1$. Moreover, $A \in Z$ because, otherwise, we can construct a route that is shorter than $\varrho_1$ by simply removing $B_1$ from $\varrho_1$, which is a contradiction. This implies that $A \aa B_1$ is in $G_2$ because, otherwise, $\rho_2$ would not be $(Z \cup W)$-open. This implies that $A \aa B_1 - B_2$ or $A \aa B_1 \ao B_2$ is in $G_2$, which implies that $A - B_2$ or $A \ao B_2$ is in $G_2$. The situation is depicted in the following subgraphs.

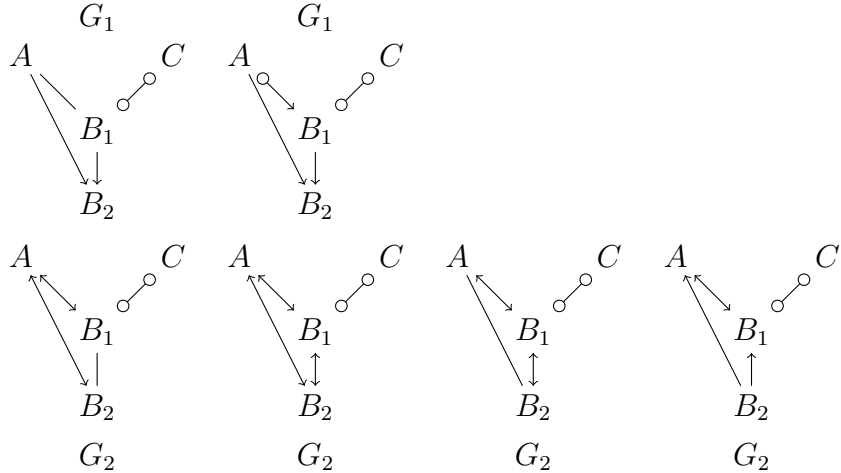
\begin{table}[H]
\centering
\scalebox{1.0}{
\begin{tabular}{cccc}
$G_1$&$G_1$
\\
\begin{tikzpicture}[inner sep=1mm]
\node at (0,0) (A) {$A$};
\node at (2,0) (C) {$C$};
\node at (1,-1) (B) {$B_1$};
\node at (1,-2) (D) {$B_2$};
\path[-] (A) edge (B);
\path[o-o] (C) edge (B);
\path[->] (B) edge (D);
\path[->] (A) edge (D);
\end{tikzpicture}
&
\begin{tikzpicture}[inner sep=1mm]
\node at (0,0) (A) {$A$};
\node at (2,0) (C) {$C$};
\node at (1,-1) (B) {$B_1$};
\node at (1,-2) (D) {$B_2$};
\path[o->] (A) edge (B);
\path[o-o] (C) edge (B);
\path[->] (B) edge (D);
\path[->] (A) edge (D);
\end{tikzpicture}
\\
\begin{tikzpicture}[inner sep=1mm]
\node at (0,0) (A) {$A$};
\node at (2,0) (C) {$C$};
\node at (1,-1) (B) {$B_1$};
\node at (1,-2) (D) {$B_2$};
\path[o-o] (C) edge (B);
\path[<->] (A) edge (B);
\path[-] (B) edge (D);
\path[<->] (A) edge (D);
\end{tikzpicture}
&
\begin{tikzpicture}[inner sep=1mm]
\node at (0,0) (A) {$A$};
\node at (2,0) (C) {$C$};
\node at (1,-1) (B) {$B_1$};
\node at (1,-2) (D) {$B_2$};
\path[o-o] (C) edge (B);
\path[<->] (A) edge (B);
\path[<->] (B) edge (D);
\path[<->] (A) edge (D);
\end{tikzpicture}
&
\begin{tikzpicture}[inner sep=1mm]
\node at (0,0) (A) {$A$};
\node at (2,0) (C) {$C$};
\node at (1,-1) (B) {$B_1$};
\node at (1,-2) (D) {$B_2$};
\path[o-o] (C) edge (B);
\path[<->] (A) edge (B);
\path[<->] (B) edge (D);
\path[-] (A) edge (D);
\end{tikzpicture}
&
\begin{tikzpicture}[inner sep=1mm]
\node at (0,0) (A) {$A$};
\node at (2,0) (C) {$C$};
\node at (1,-1) (B) {$B_1$};
\node at (1,-2) (D) {$B_2$};
\path[o-o] (C) edge (B);
\path[<->] (A) edge (B);
\path[<-] (B) edge (D);
\path[<-] (A) edge (D);
\end{tikzpicture}
\\
$G_2$&$G_2$&$G_2$&$G_2$
\end{tabular}}
\end{table}

Now, let $A'$ be the node that precedes $A$ in $\rho_1$. Note that $A' \la A$ cannot be in $\rho_1$ or $\rho_2$ because, otherwise, these would not be $(Z \cup W)$-open since $A \in Z$. Then, $A' - A$ or $A' \oa A$ is in $G_1$ and $G_2$. Then, $A' - A \ra B_2$ or $A' \oa A \ra B_2$ is in $G_1$ whereas $A' - A \ao B_2$, $A' - A - B_2$, $A' \oa A \ao B_2$ or $A' \oa A - B_2$ is in $G_2$. These four subgraphs of $G_2$ imply that $A'$ and $B_2$ must be adjacent in $G_1$ and $G_2$: The second subgraph due to the constraint C3 because $A \aa B_1$ is in $G_2$, and the other three subgraphs because $G_1$ and $G_2$ are triplex equivalent. By repeating the reasoning in the paragraph above, we can conclude that $A' \ra B_2$ is in $G_1$, which implies that $A' \in Z$, which implies that $A' - A$ or $A' \aa A$ is in $G_2$, which implies that $A' - B_2$ or $A' \ao B_2$ is in $G_2$.

By repeating the reasoning in the paragraph above,\footnote{Let $A''$ be the node that precedes $A'$ in $\rho_1$. For this repeated reasoning to be correct, it is important to realize that if $A' - A$ is in $G_2$, then $A'' \aa A'$ must be in $G_2$, because $A' \in Z$ and $\rho_2$ is $(Z \cup W)$-open.} we can conclude that $\alpha \ra B_2$ is in $G_1$ and, thus, we can construct a route that is shorter than $\varrho_1$ by simply removing some nodes from $\varrho_1$, which is a contradiction. Consequently, $B_1 \ra B_2$ must be in $G_2$.

Finally, assume that $B_1 \ra B_2 \ra B_3$ is in $G_1$ but $B_1 \ra B_2 - B_3$ or $B_1 \ra B_2 \ao B_3$ is in $G_2$. Then, $B_1$ and $B_3$ must be adjacent in $G_1$ and $G_2$ because these are triplex equivalent. This implies that $B_1 \ra B_3$ is in $G_1$, which implies that we can construct a route that is shorter than $\varrho_1$ by simply removing $B_2$ from $\varrho_1$, which is a contradiction. By repeating this reasoning, we can conclude that $B_1 \ra \ldots \ra B_n$ is in $G_2$ and, thus, that $\rho_2$ is $Z$-open.

\end{proof}

\begin{proof}[{\bf Proof of Lemma \ref{lem:directed}}]

Assume to the contrary that there are two such sets of directed node pairs. Let the MAMP CG $G$ contain exactly the directed node pairs in one of the sets, and let the MAMP CG $H$ contain exactly the directed node pairs in the other set. For every $A \ra B$ in $G$ st $A - B$ or $A \aa B$ is in $H$, replace the edge between $A$ and $B$ in $H$ with $A \ra B$ and call the resulting graph $F$. We prove below that $F$ is a MAMP CG that is triplex equivalent to $G$ and thus to $H$, which is a contradiction since $F$ has a proper superset of the directed node pairs in $H$.

First, note that $F$ cannot violate the constraints C2 and C3. Assume to the contrary that $F$ violates the constraint C1 due to a cycle $\rho$. Note that none of the directed edges in $\rho$ can be in $H$ because, otherwise, $H$ would violate the constraint C1, since $H$ has the same adjacencies as $F$ but a subset of the directed edges in $F$. Then, all the directed edges in $\rho$ must be in $G$. However, this implies the contradictory conclusion that $G$ violates the constraint C1, since $G$ has the same adjacencies as $F$ but a subset of the directed edges in $F$.

Second, assume to the contrary that $G$ (and, thus, $H$) has a triplex $(\{A,C\},B)$ that $F$ has not. Then, $\{A,B\}$ or $\{B,C\}$ must an directed node pair in $G$ because, otherwise, $F$ would have a triplex $(\{A,C\},B)$ since $F$ would have the same induced graph over $\{A, B, C\}$ as $H$. Specifically, $A \ra B$ or $B \la C$ must be in $G$ because, otherwise, $G$ would not have a triplex $(\{A,C\},B)$. Moreover, neither $A \la B$ nor $B \ra C$ can be $H$ because, otherwise, $H$ would not have a triplex $(\{A,C\},B)$. Therefore, if $A \ra B$ or $B \la C$ is in $G$ and neither $A \la B$ nor $B \ra C$ is in $H$, then $A \ra B$ or $B \la C$ must be in $F$. However, this implies that $B \ra C$ or $A \la B$ must be in $F$ because, otherwise, $F$ would have a triplex $(\{A,C\},B)$ which would be a contradiction. However, this is a contradiction since neither $B \ra C$ nor $A \la B$ can be in $G$ or $H$ because, otherwise, neither $G$ nor $H$ would have a triplex $(\{A,C\},B)$.

Finally, assume to the contrary that $F$ has a triplex $(\{A,C\},B)$ that $G$ has not (and, thus, nor does $H$). Then, $A - B - C$ must be in $H$ because, otherwise, $A \la B$ or $B \ra C$ would be in $H$ and, thus, $F$ would not have a triplex $(\{A,C\},B)$. However, this implies that $A \ra B$ or $B \la C$ is in $G$ because, otherwise, $F$ would not have a triplex $(\{A,C\},B)$. However, this implies that $B \ra C$ or $A \la B$ is in $G$ because, otherwise, $G$ would have a triplex $(\{A,C\},B)$. Therefore, $A \ra B \ra C$ or $A \la B \la C$ is in $G$ and, thus, $A \ra B \ra C$ or $A \la B \la C$ must be in $F$ since $A - B - C$ is in $H$. However, this contradicts the assumption that $F$ has a triplex $(\{A,C\},B)$.

\end{proof}

\begin{proof}[{\bf Proof of Lemma \ref{lem:bidirected}}]

Assume to the contrary that there are two such sets of bidirected edges. Let the MDCG $G$ contain exactly the bidirected edges in one of the sets, and let the MDCG $H$ contain exactly the bidirected edges in the other set. For every $A \aa B$ in $G$ st $A - B$ is in $H$, replace $A - B$ with $A \aa B$ in $H$ and call the resulting graph $F$. We prove below that $F$ is a MDCG that is triplex equivalent to $G$, which is a contradiction since $F$ has a proper superset of the bidirected edges in $G$.

First, note that $F$ cannot violate the constraint C1. Assume to the contrary that $F$ violates the constraint C2 due to a cycle $\rho$. Note that all the undirected edges in $\rho$ are in $H$. In fact, they must also be in $G$, because $G$ and $H$ have the same directed node pairs and bidirected edges. Moreover, the bidirected edge in $\rho$ must be in $G$ or $H$. However, this is a contradiction. Now, assume to the contrary that $F$ violates the constraint C3 because $A - B - C$ and $B \aa D$ are in $F$ but $A$ and $C$ are not adjacent in $F$ (note that if $A$ and $C$ were adjacent in $F$, then they would not violate the constraint C3 or they would violate the constraint C1 or C2, which is impossible as we have just shown). Note that $A - B - C$ must be in $H$. In fact, $A - B - C$ must also be in $G$, because $G$ and $H$ have the same directed node pairs and bidirected edges. Moreover, $B \aa D$ must be in $G$ or $H$. However, this implies that $A$ and $C$ are adjacent in $G$ or $H$ by the constraint C3, which implies that $A$ and $C$ are adjacent in $G$ and $H$ because they are triplex equivalent and thus also in $F$, which is a contradiction. Consequently, $F$ is a MAMP CG, which implies that $F$ is a MDCG because it has the same directed edges as $G$ and $H$.

Second, note that all the triplexes in $G$ are in $F$ too.

Finally, assume to the contrary that $F$ has a triplex $(\{A,C\},B)$ that $G$ has not (and, thus, nor does $H$). Then, $A - B - C$ must be in $H$ because, otherwise, $A \la B$ or $B \ra C$ would be in $H$ and thus $F$ would not have a triplex $(\{A,C\},B)$. However, this implies that $F$ has the same induced graph over $\{A, B, C\}$ as $G$, which contradicts the assumption that $F$ has a triplex $(\{A,C\},B)$.

\end{proof}

\begin{proof}[{\bf Proof of Theorem \ref{the:GG'2}}]

It suffices to show that every $Z$-open path between $\alpha$ and $\beta$ in $G$ can be transformed into a $Z$-open path between $\alpha$ and $\beta$ in $G'$ and vice versa, with $\alpha, \beta \in V$ and $Z \subseteq V \setminus \alpha \setminus \beta$.

Let $\rho$ denote a $Z$-open path between $\alpha$ and $\beta$ in $G$. We can easily transform $\rho$ into a path $\rho'$ between $\alpha$ and $\beta$ in $G'$: Simply, replace every maximal subpath of $\rho$ of the form $V_1 \bb V_2 \bb \ldots \bb V_{n-1} \bb V_n$ ($n \geq 2$) with $V_1 \la \epsilon^{V_1} \bb \epsilon^{V_2} \bb \ldots \bb \epsilon^{V_{n-1}} \bb \epsilon^{V_n} \ra V_n$. We now show that $\rho'$ is $Z$-open.

\begin{description}

\item[Case 1.1] If $B \in V$ is a triplex node in $\rho'$, then $\rho'$ must have one of the following subpaths:

\begin{table}[H]
\centering
\scalebox{1.0}{
\begin{tabular}{c}
\begin{tikzpicture}[inner sep=1mm]
\node at (0,0) (A) {$A$};
\node at (1,0) (B) {$B$};
\node at (2,0) (C) {$C$};
\path[->] (A) edge (B);
\path[<-] (B) edge (C);
\end{tikzpicture}
\begin{tikzpicture}[inner sep=1mm]
\node at (0,0) (A) {$A$};
\node at (1,0) (B) {$B$};
\node at (2,0) (C) {$\epsilon^B$};
\node at (3,0) (D) {$\epsilon^C$};
\path[->] (A) edge (B);
\path[<-] (B) edge (C);
\path[|-|] (D) edge (C);
\end{tikzpicture}
\begin{tikzpicture}[inner sep=1mm]
\node at (0,0) (A) {$\epsilon^B$};
\node at (1,0) (B) {$B$};
\node at (2,0) (C) {$C$};
\node at (-1,0) (D) {$\epsilon^A$};
\path[->] (A) edge (B);
\path[<-] (B) edge (C);
\path[|-|] (D) edge (A);
\end{tikzpicture}
\end{tabular}}
\end{table}

with $A, C \in V$. Therefore, $\rho$ must have one of the following subpaths (specifically, if $\rho'$ has the $i$-th subpath above, then $\rho$ has the $i$-th subpath below):

\begin{table}[H]
\centering
\scalebox{1.0}{
\begin{tabular}{c}
\begin{tikzpicture}[inner sep=1mm]
\node at (0,0) (A) {$A$};
\node at (1,0) (B) {$B$};
\node at (2,0) (C) {$C$};
\path[->] (A) edge (B);
\path[<-] (B) edge (C);
\end{tikzpicture}
\begin{tikzpicture}[inner sep=1mm]
\node at (0,0) (A) {$A$};
\node at (1,0) (B) {$B$};
\node at (2,0) (C) {$C$};
\path[->] (A) edge (B);
\path[|-|] (B) edge (C);
\end{tikzpicture}
\begin{tikzpicture}[inner sep=1mm]
\node at (0,0) (A) {$A$};
\node at (1,0) (B) {$B$};
\node at (2,0) (C) {$C$};
\path[|-|] (A) edge (B);
\path[<-] (B) edge (C);
\end{tikzpicture}
\end{tabular}}
\end{table}

In either case, $B$ is a triplex node in $\rho$ and, thus, $B \in Z \cup san_G(Z)$ for $\rho$ to be $Z$-open. Then, $B \in Z \cup san_{G'}(Z)$ by construction of $G'$ and, thus, $B \in D(Z) \cup san_{G'}(D(Z))$.

\item[Case 1.2] If $B \in V$ is a non-triplex node in $\rho'$, then $\rho'$ must have one of the following subpaths:

\begin{table}[H]
\centering
\scalebox{1.0}{
\begin{tabular}{c}
\begin{tikzpicture}[inner sep=1mm]
\node at (0,0) (A) {$A$};
\node at (1,0) (B) {$B$};
\node at (2,0) (C) {$C$};
\path[->] (A) edge (B);
\path[->] (B) edge (C);
\end{tikzpicture}
\begin{tikzpicture}[inner sep=1mm]
\node at (0,0) (A) {$A$};
\node at (1,0) (B) {$B$};
\node at (2,0) (C) {$C$};
\path[<-] (A) edge (B);
\path[->] (B) edge (C);
\end{tikzpicture}
\begin{tikzpicture}[inner sep=1mm]
\node at (0,0) (A) {$A$};
\node at (1,0) (B) {$B$};
\node at (2,0) (C) {$C$};
\path[<-] (A) edge (B);
\path[<-] (B) edge (C);
\end{tikzpicture}
\begin{tikzpicture}[inner sep=1mm]
\node at (0,0) (A) {$A$};
\node at (1,0) (B) {$B$};
\node at (2,0) (C) {$\epsilon^B$};
\node at (3,0) (D) {$\epsilon^C$};
\path[<-] (A) edge (B);
\path[<-] (B) edge (C);
\path[|-|] (D) edge (C);
\end{tikzpicture}
\begin{tikzpicture}[inner sep=1mm]
\node at (0,0) (A) {$\epsilon^B$};
\node at (1,0) (B) {$B$};
\node at (2,0) (C) {$C$};
\node at (-1,0) (D) {$\epsilon^A$};
\path[->] (A) edge (B);
\path[->] (B) edge (C);
\path[|-|] (D) edge (A);
\end{tikzpicture}
\end{tabular}}
\end{table}

with $A, C \in V$. Therefore, $\rho$ must have one of the following subpaths (specifically, if $\rho'$ has the $i$-th subpath above, then $\rho$ has the $i$-th subpath below):

\begin{table}[H]
\centering
\scalebox{1.0}{
\begin{tabular}{c}
\begin{tikzpicture}[inner sep=1mm]
\node at (0,0) (A) {$A$};
\node at (1,0) (B) {$B$};
\node at (2,0) (C) {$C$};
\path[->] (A) edge (B);
\path[->] (B) edge (C);
\end{tikzpicture}
\begin{tikzpicture}[inner sep=1mm]
\node at (0,0) (A) {$A$};
\node at (1,0) (B) {$B$};
\node at (2,0) (C) {$C$};
\path[<-] (A) edge (B);
\path[->] (B) edge (C);
\end{tikzpicture}
\begin{tikzpicture}[inner sep=1mm]
\node at (0,0) (A) {$A$};
\node at (1,0) (B) {$B$};
\node at (2,0) (C) {$C$};
\path[<-] (A) edge (B);
\path[<-] (B) edge (C);
\end{tikzpicture}
\begin{tikzpicture}[inner sep=1mm]
\node at (0,0) (A) {$A$};
\node at (1,0) (B) {$B$};
\node at (2,0) (C) {$C$};
\path[<-] (A) edge (B);
\path[|-|] (B) edge (C);
\end{tikzpicture}
\begin{tikzpicture}[inner sep=1mm]
\node at (0,0) (A) {$A$};
\node at (1,0) (B) {$B$};
\node at (2,0) (C) {$C$};
\path[|-|] (A) edge (B);
\path[->] (B) edge (C);
\end{tikzpicture}
\end{tabular}}
\end{table}

In either case, $B$ is a non-triplex node in $\rho$ and, thus, $B \notin Z$ for $\rho$ to be $Z$-open. Since $Z$ contains no error node, $Z$ cannot determine any node in $V$ that is not already in $Z$. Then, $B \notin D(Z)$.

\item[Case 1.3] If $\epsilon^B$ is a triplex node in $\rho'$, then $\rho'$ must have one of the following subpaths:

\begin{table}[H]
\centering
\scalebox{1.0}{
\begin{tabular}{c}
\begin{tikzpicture}[inner sep=1mm]
\node at (0,0) (A) {$\epsilon^A$};
\node at (1,0) (B) {$\epsilon^B$};
\node at (2,0) (C) {$\epsilon^C$};
\path[<->] (A) edge (B);
\path[|-|] (B) edge (C);
\end{tikzpicture}
\begin{tikzpicture}[inner sep=1mm]
\node at (0,0) (A) {$\epsilon^A$};
\node at (1,0) (B) {$\epsilon^B$};
\node at (2,0) (C) {$\epsilon^C$};
\path[-] (A) edge (B);
\path[<->] (B) edge (C);
\end{tikzpicture}
\end{tabular}}
\end{table}

Therefore, $\rho$ must have one of the following subpaths (specifically, if $\rho'$ has the $i$-th subpath above, then $\rho$ has the $i$-th subpath below):

\begin{table}[H]
\centering
\scalebox{1.0}{
\begin{tabular}{c}
\begin{tikzpicture}[inner sep=1mm]
\node at (0,0) (A) {$A$};
\node at (1,0) (B) {$B$};
\node at (2,0) (C) {$C$};
\path[<->] (A) edge (B);
\path[|-|] (B) edge (C);
\end{tikzpicture}
\begin{tikzpicture}[inner sep=1mm]
\node at (0,0) (A) {$A$};
\node at (1,0) (B) {$B$};
\node at (2,0) (C) {$C$};
\path[-] (A) edge (B);
\path[<->] (B) edge (C);
\end{tikzpicture}
\end{tabular}}
\end{table}

In either case, $B$ is a triplex node in $\rho$ and, thus, $B \in Z \cup san_G(Z)$ for $\rho$ to be $Z$-open. Then, $\epsilon^B \in Z \cup san_{G'}(Z)$ by construction of $G'$ and, thus, $\epsilon^B \in D(Z) \cup san_{G'}(D(Z))$.

\item[Case 1.4] If $\epsilon^B$ is a non-triplex node in $\rho'$, then $\rho'$ must have one of the following subpaths:

\begin{table}[H]
\centering
\scalebox{1.0}{
\begin{tabular}{c}
\begin{tikzpicture}[inner sep=1mm]
\node at (0,0) (A) {$A$};
\node at (1,0) (B) {$B$};
\node at (2,0) (C) {$\epsilon^B$};
\node at (3,0) (D) {$\epsilon^C$};
\path[->] (A) edge (B);
\path[<-] (B) edge (C);
\path[|-|] (D) edge (C);
\end{tikzpicture}
\begin{tikzpicture}[inner sep=1mm]
\node at (0,0) (A) {$\epsilon^B$};
\node at (1,0) (B) {$B$};
\node at (2,0) (C) {$C$};
\node at (-1,0) (D) {$\epsilon^A$};
\path[->] (A) edge (B);
\path[<-] (B) edge (C);
\path[|-|] (D) edge (A);
\end{tikzpicture}
\begin{tikzpicture}[inner sep=1mm]
\node at (0.65,0) (B) {$\alpha=B$};
\node at (2,0) (C) {$\epsilon^B$};
\node at (3,0) (D) {$\epsilon^C$};
\path[<-] (B) edge (C);
\path[|-|] (D) edge (C);
\end{tikzpicture}
\begin{tikzpicture}[inner sep=1mm]
\node at (0,0) (A) {$\epsilon^B$};
\node at (1.35,0) (B) {$B=\beta$};
\node at (-1,0) (D) {$\epsilon^A$};
\path[->] (A) edge (B);
\path[|-|] (D) edge (A);
\end{tikzpicture}\\
\begin{tikzpicture}[inner sep=1mm]
\node at (0,0) (A) {$A$};
\node at (1,0) (B) {$B$};
\node at (2,0) (C) {$\epsilon^B$};
\node at (3,0) (D) {$\epsilon^C$};
\path[<-] (A) edge (B);
\path[<-] (B) edge (C);
\path[|-|] (D) edge (C);
\end{tikzpicture}
\begin{tikzpicture}[inner sep=1mm]
\node at (0,0) (A) {$\epsilon^B$};
\node at (1,0) (B) {$B$};
\node at (2,0) (C) {$C$};
\node at (-1,0) (D) {$\epsilon^A$};
\path[->] (A) edge (B);
\path[->] (B) edge (C);
\path[|-|] (D) edge (A);
\end{tikzpicture}
\begin{tikzpicture}[inner sep=1mm]
\node at (0,0) (A) {$\epsilon^A$};
\node at (1,0) (B) {$\epsilon^B$};
\node at (2,0) (C) {$\epsilon^C$};
\path[-] (A) edge (B);
\path[-] (B) edge (C);
\end{tikzpicture}
\end{tabular}}
\end{table}

with $A, C \in V$. Recall that $\epsilon^B \notin Z$ because $Z \subseteq V \setminus \alpha \setminus \beta$. In the first case, if $\alpha=A$ then $A \notin Z$, else $A \notin Z$ for $\rho$ to be $Z$-open. Then, $\epsilon^B \notin D(Z)$. In the second case, if $\beta=C$ then $C \notin Z$, else $C \notin Z$ for $\rho$ to be $Z$-open. Then, $\epsilon^B \notin D(Z)$. In the third and fourth cases, $B \notin Z$ because $\alpha=B$ or $\beta=B$. Then, $\epsilon^B \notin D(Z)$. In the fifth and sixth cases, $B \notin Z$ for $\rho$ to be $Z$-open. Then, $\epsilon^B \notin D(Z)$. The last case implies that $\rho$ has the following subpath:

\begin{table}[H]
\centering
\scalebox{1.0}{
\begin{tabular}{c}
\begin{tikzpicture}[inner sep=1mm]
\node at (0,0) (A) {$A$};
\node at (1,0) (B) {$B$};
\node at (2,0) (C) {$C$};
\path[-] (A) edge (B);
\path[-] (B) edge (C);
\end{tikzpicture}
\end{tabular}}
\end{table}

Thus, $B$ is a non-triplex node in $\rho$, which implies that $B \notin Z$ or $pa_G(B) \setminus Z \neq \emptyset$ for $\rho$ to be $Z$-open. In either case, $\epsilon^B \notin D(Z)$ (recall that $pa_{G'}(B)=pa_G(B) \cup \epsilon^B$ by construction of $G'$).

\end{description}

Finally, let $\rho'$ denote a $Z$-open path between $\alpha$ and $\beta$ in $G'$. We can easily transform $\rho'$ into a path $\rho$ between $\alpha$ and $\beta$ in $G$: Simply, replace every maximal subpath of $\rho'$ of the form $V_1 \la \epsilon^{V_1} \bb \epsilon^{V_2} \bb \ldots \bb \epsilon^{V_{n-1}} \bb \epsilon^{V_n} \ra V_n$ ($n \geq 2$) with $V_1 \bb V_2 \bb \ldots \bb V_{n-1} \bb V_n$. We now show that $\rho$ is $Z$-open. Note that all the nodes in $\rho$ are in $V$. 

\begin{description}

\item[Case 2.1] If $B$ is a triplex node in $\rho$, then $\rho$ must have one of the following subpaths:

\begin{table}[H]
\centering
\scalebox{1.0}{
\begin{tabular}{c}
\begin{tikzpicture}[inner sep=1mm]
\node at (0,0) (A) {$A$};
\node at (1,0) (B) {$B$};
\node at (2,0) (C) {$C$};
\path[->] (A) edge (B);
\path[<-] (B) edge (C);
\end{tikzpicture}
\begin{tikzpicture}[inner sep=1mm]
\node at (0,0) (A) {$A$};
\node at (1,0) (B) {$B$};
\node at (2,0) (C) {$C$};
\path[->] (A) edge (B);
\path[|-|] (B) edge (C);
\end{tikzpicture}
\begin{tikzpicture}[inner sep=1mm]
\node at (0,0) (A) {$A$};
\node at (1,0) (B) {$B$};
\node at (2,0) (C) {$C$};
\path[|-|] (A) edge (B);
\path[<-] (B) edge (C);
\end{tikzpicture}
\begin{tikzpicture}[inner sep=1mm]
\node at (0,0) (A) {$A$};
\node at (1,0) (B) {$B$};
\node at (2,0) (C) {$C$};
\path[<->] (A) edge (B);
\path[|-|] (B) edge (C);
\end{tikzpicture}
\begin{tikzpicture}[inner sep=1mm]
\node at (0,0) (A) {$A$};
\node at (1,0) (B) {$B$};
\node at (2,0) (C) {$C$};
\path[-] (A) edge (B);
\path[<->] (B) edge (C);
\end{tikzpicture}
\end{tabular}}
\end{table}

with $A, C \in V$. Therefore, $\rho'$ must have one of the following subpaths (specifically, if $\rho$ has the $i$-th subpath above, then $\rho'$ has the $i$-th subpath below):

\begin{table}[H]
\centering
\scalebox{1.0}{
\begin{tabular}{c}
\begin{tikzpicture}[inner sep=1mm]
\node at (0,0) (A) {$A$};
\node at (1,0) (B) {$B$};
\node at (2,0) (C) {$C$};
\path[->] (A) edge (B);
\path[<-] (B) edge (C);
\end{tikzpicture}
\begin{tikzpicture}[inner sep=1mm]
\node at (0,0) (A) {$A$};
\node at (1,0) (B) {$B$};
\node at (2,0) (C) {$\epsilon^B$};
\node at (3,0) (D) {$\epsilon^C$};
\path[->] (A) edge (B);
\path[<-] (B) edge (C);
\path[|-|] (D) edge (C);
\end{tikzpicture}
\begin{tikzpicture}[inner sep=1mm]
\node at (0,0) (A) {$\epsilon^B$};
\node at (1,0) (B) {$B$};
\node at (2,0) (C) {$C$};
\node at (-1,0) (D) {$\epsilon^A$};
\path[->] (A) edge (B);
\path[<-] (B) edge (C);
\path[|-|] (D) edge (A);
\end{tikzpicture}
\begin{tikzpicture}[inner sep=1mm]
\node at (0,0) (A) {$\epsilon^B$};
\node at (1,0) (C) {$\epsilon^C$};
\node at (-1,0) (D) {$\epsilon^A$};
\path[|-|] (A) edge (C);
\path[<->] (D) edge (A);
\end{tikzpicture}
\begin{tikzpicture}[inner sep=1mm]
\node at (0,0) (A) {$\epsilon^B$};
\node at (1,0) (C) {$\epsilon^C$};
\node at (-1,0) (D) {$\epsilon^A$};
\path[<->] (A) edge (C);
\path[-] (D) edge (A);
\end{tikzpicture}
\end{tabular}}
\end{table}

In the first three cases, $B$ is a triplex node in $\rho'$ and, thus, $B \in D(Z) \cup san_{G'}(D(Z))$ for $\rho'$ to be $Z$-open. Since $Z$ contains no error node, $Z$ cannot determine any node in $V$ that is not already in $Z$. Then, $B \in D(Z)$ iff $B \in Z$. Since there is no strictly descending route from $B$ to any error node, then any strictly descending route from $B$ to a node $D \in D(Z)$ implies that $D \in V$ which, as seen, implies that $D \in Z$. Then, $B \in san_{G'}(D(Z))$ iff $B \in san_{G'}(Z)$. Moreover, $B \in san_{G'}(Z)$ iff $B \in san_{G}(Z)$ by construction of $G'$. These results together imply that $B \in Z \cup san_{G}(Z)$.

In the last two cases, $\epsilon^B$ is a triplex node in $\rho'$ and, thus, $B \in D(Z) \cup san_{G'}(D(Z))$ for $\rho'$ to be $Z$-open because $Z$ contains no error node. Therefore, as shown in the previous paragraph, $B \in Z \cup san_{G}(Z)$.

\item[Case 2.2] If $B$ is a non-triplex node in $\rho$, then $\rho$ must have one of the following subpaths:

\begin{table}[H]
\centering
\scalebox{1.0}{
\begin{tabular}{c}
\begin{tikzpicture}[inner sep=1mm]
\node at (0,0) (A) {$A$};
\node at (1,0) (B) {$B$};
\node at (2,0) (C) {$C$};
\path[->] (A) edge (B);
\path[->] (B) edge (C);
\end{tikzpicture}
\begin{tikzpicture}[inner sep=1mm]
\node at (0,0) (A) {$A$};
\node at (1,0) (B) {$B$};
\node at (2,0) (C) {$C$};
\path[<-] (A) edge (B);
\path[->] (B) edge (C);
\end{tikzpicture}
\begin{tikzpicture}[inner sep=1mm]
\node at (0,0) (A) {$A$};
\node at (1,0) (B) {$B$};
\node at (2,0) (C) {$C$};
\path[<-] (A) edge (B);
\path[<-] (B) edge (C);
\end{tikzpicture}
\begin{tikzpicture}[inner sep=1mm]
\node at (0,0) (A) {$A$};
\node at (1,0) (B) {$B$};
\node at (2,0) (C) {$C$};
\path[<-] (A) edge (B);
\path[|-|] (B) edge (C);
\end{tikzpicture}
\begin{tikzpicture}[inner sep=1mm]
\node at (0,0) (A) {$A$};
\node at (1,0) (B) {$B$};
\node at (2,0) (C) {$C$};
\path[|-|] (A) edge (B);
\path[->] (B) edge (C);
\end{tikzpicture}
\begin{tikzpicture}[inner sep=1mm]
\node at (0,0) (A) {$A$};
\node at (1,0) (B) {$B$};
\node at (2,0) (C) {$C$};
\path[-] (A) edge (B);
\path[-] (B) edge (C);
\end{tikzpicture}
\end{tabular}}
\end{table}

with $A, C \in V$. Therefore, $\rho'$ must have one of the following subpaths (specifically, if $\rho$ has the $i$-th subpath above, then $\rho'$ has the $i$-th subpath below):

\begin{table}[H]
\centering
\scalebox{1.0}{
\begin{tabular}{c}
\begin{tikzpicture}[inner sep=1mm]
\node at (0,0) (A) {$A$};
\node at (1,0) (B) {$B$};
\node at (2,0) (C) {$C$};
\path[->] (A) edge (B);
\path[->] (B) edge (C);
\end{tikzpicture}
\begin{tikzpicture}[inner sep=1mm]
\node at (0,0) (A) {$A$};
\node at (1,0) (B) {$B$};
\node at (2,0) (C) {$C$};
\path[<-] (A) edge (B);
\path[->] (B) edge (C);
\end{tikzpicture}
\begin{tikzpicture}[inner sep=1mm]
\node at (0,0) (A) {$A$};
\node at (1,0) (B) {$B$};
\node at (2,0) (C) {$C$};
\path[<-] (A) edge (B);
\path[<-] (B) edge (C);
\end{tikzpicture}
\begin{tikzpicture}[inner sep=1mm]
\node at (0,0) (A) {$A$};
\node at (1,0) (B) {$B$};
\node at (2,0) (C) {$\epsilon^B$};
\node at (3,0) (D) {$\epsilon^C$};
\path[<-] (A) edge (B);
\path[<-] (B) edge (C);
\path[|-|] (D) edge (C);
\end{tikzpicture}
\begin{tikzpicture}[inner sep=1mm]
\node at (0,0) (A) {$\epsilon^B$};
\node at (1,0) (B) {$B$};
\node at (2,0) (C) {$C$};
\node at (-1,0) (D) {$\epsilon^A$};
\path[->] (A) edge (B);
\path[->] (B) edge (C);
\path[|-|] (D) edge (A);
\end{tikzpicture}\\
\begin{tikzpicture}[inner sep=1mm]
\node at (0,0) (A) {$\epsilon^A$};
\node at (1,0) (B) {$\epsilon^B$};
\node at (2,0) (C) {$\epsilon^C$};
\path[-] (A) edge (B);
\path[-] (B) edge (C);
\end{tikzpicture}
\end{tabular}}
\end{table}

In the first five cases, $B$ is a non-triplex node in $\rho'$ and, thus, $B \notin D(Z)$ for $\rho'$ to be $Z$-open. Since $Z$ contains no error node, $Z$ cannot determine any node in $V$ that is not already in $Z$. Then, $B \notin Z$. In the last case, $\epsilon^B$ is a non-triplex node in $\rho'$ and, thus, $\epsilon^B \notin D(Z)$ for $\rho'$ to be $Z$-open. Then, $B \notin Z$ or $pa_{G'}(B) \setminus \epsilon^B \setminus Z \ \neq \emptyset$. Then, $B \notin Z$ or $pa_{G}(B) \setminus Z \ \neq \emptyset$ (recall that $pa_{G'}(B)=pa_G(B) \cup \epsilon^B$ by construction of $G'$).

\end{description}

\end{proof}

\begin{proof}[{\bf Proof of Theorem \ref{the:closed}}]

We find it easier to prove the theorem by defining separation in MAMP CGs in terms of routes rather than paths. A node $B$ in a route $\rho$ in a MAMP CG $G$ is called a triplex node in $\rho$ if $A \oa B \ao C$, $A \oa B - C$, or $A - B \ao C$ is a subroute of $\rho$ (note that maybe $A=C$ in the first case). Note that $B$ may be both a triplex and a non-triplex node in $\rho$. Moreover, $\rho$ is said to be $Z$-open with $Z \subseteq V$ when 

\begin{itemize}
\item every triplex node in $\rho$ is in $D(Z)$, and

\item no non-triplex node in $\rho$ is in $D(Z)$.
\end{itemize}

When there is no $Z$-open route in $G$ between a node in $X$ and a node in $Y$, we say that $X$ is separated from $Y$ given $Z$ in $G$ and denote it as $X \ci_G Y | Z$. It is straightforward to see that this and the original definition of separation in MAMP CGs introduced in Section \ref{sec:mampcgs} are equivalent, in the sense that they identify the same separations in $G$.

We prove the theorem for the case where $L$ contains a single node $B$. The general case follows by induction. Specifically, given $\alpha, \beta \in V \setminus L$ and $Z \subseteq V \setminus L \setminus \alpha \setminus \beta$, we show below that every $Z$-open route between $\alpha$ and $\beta$ in $[G']_L$ can be transformed into a $Z$-open route between $\alpha$ and $\beta$ in $G'$ and vice versa.

First, let $\rho$ denote a $Z$-open route between $\alpha$ and $\beta$ in $[G']_L$. We can easily transform $\rho$ into a $Z$-open route between $\alpha$ and $\beta$ in $G'$: For each edge $A \ra C$ or $A \la C$ with $A, C \in V \cup \epsilon$ that is in $[G']_L$ but not in $G'$, replace each of its occurrence in $\rho$ with $A \ra B \ra C$ or $A \la B \la C$, respectively. Note that $B \notin D(Z)$ because $B, \epsilon^B \notin Z$.

Second, let $\rho$ denote a $Z$-open route between $\alpha$ and $\beta$ in $G'$. Note that $B$ cannot participate in any undirected or bidirected edge in $G'$, because $B \in V$. Note also that $B$ cannot be a triplex node in $\rho$, because $B \notin D(Z)$ since $B, \epsilon^B \notin Z$. Note also that $B \neq \alpha, \beta$. Then, $B$ can only appear in $\rho$ in the following configurations: $A \ra B \ra C$, $A \la B \la C$, or $A \la B \ra C$ with $A, C \in V \cup \epsilon$. Then, we can easily transform $\rho$ into a $Z$-open route between $\alpha$ and $\beta$ in $[G']_L$: Replace each occurrence of $A \ra B \ra C$ in $\rho$ with $A \ra C$, each occurrence of $A \la B \la C$ in $\rho$ with $A \la C$, and each occurrence of $A \la B \ra C$ in $\rho$ with $A \la \epsilon^B \ra C$. In the last case, note that $\epsilon^B \notin D(Z)$ because $B, \epsilon^B \notin Z$.

\end{proof}

\end{document}